\documentclass[10pt]{article}
\usepackage{amsmath,amsfonts,amsthm,mathrsfs,url}
\usepackage{color}
\usepackage{fancyhdr}
\usepackage{amscd}
\usepackage[top=2.3cm, bottom=2.5cm, left=2.1cm,
right=2.1cm]{geometry}
\usepackage{bm}
\usepackage{enumitem,amsthm,amsmath,amsfonts,amssymb}
\usepackage{float}
\usepackage{mathtools}
\usepackage{algorithmic}
\usepackage{algorithm}
\usepackage{xcolor}
\usepackage{graphicx}
\usepackage{subcaption}
\usepackage{tikz}
\usepackage{multirow}
\usepackage{multicol}
\usepackage{hyperref}
\usepackage{mathrsfs, fancyhdr}
\usepackage{multirow}
\usepackage{booktabs}

\newcommand{\Z}{\mathbb{Z}}

\newcommand{\E}{\mathbb{E}}

\newcommand{\R}{\mathbb{R}}

\newtheorem{ass}{Assumption}

\newtheorem{cor}{Corollary}

\newtheorem{prop}{Proposition}
\newtheorem{rem}{Remark}

\DeclareMathOperator*{\argmin}{argmin}

\def\bW{\mathbf{W}}
\def\bw{\mathbf{w}}
\def\O{\mathcal{O}}

\def\X{\mathcal{X}}
\def\Y{\mathcal{Y}}
\def\W{\mathcal{W}}

\def\Z{\mathcal{Z}}

\def\bx{\mathbf{x}}

\def\bw{\mathbf{w}}

\def\bu{\mathbf{u}}
\def\A{\mathcal{A}}

\def\0{\mathbf{0}}

\newtheorem{theorem}{Theorem}
\newtheorem{lemma}[theorem]{Lemma}
\newtheorem{proposition}[theorem]{Proposition}
\newtheorem{corollary}[theorem]{Corollary}
\theoremstyle{definition}
\newtheorem{definition}{Definition}

\newtheorem{example}{Example}
\theoremstyle{definition}

\newtheorem{remark}{Remark}

\def\begeqn{\begin{equation}}
\def\endeqn{\end{equation}}
\def\begth{\begin{theorem}}
\def\endth{\end{theorem}}
\def\begprop{\begin{proposition}}
\def\endprop{\end{proposition}}
\def\begcor{\begin{corollary}}
\def\endcor{\end{corollary}}
\def\begdef{\begin{definition}}
\def\enddef{\end{definition}}
\def\beglemm{\begin{lemma}}
\def\endlemm{\end{lemma}}
\def\begexm{\begin{example}}
\def\endexm{\end{example}}
\def\begrem{\begin{remark}}
\def\endrem{\end{remark}}

\def\begdef{\begin{definition}}
\def\enddef{\end{definition}}

\begin{document}

\title{Generalization Guarantees of Gradient Descent for Multi-Layer \\ Neural Networks}

\author{Puyu Wang$^1 $\quad  Yunwen Lei$^2$\quad Di Wang$^3$  \quad Yiming Ying$^{4*}$\quad Ding-Xuan Zhou$^5$ \\ 
$^{1}$ Hong Kong Baptist University, Hong Kong \\
$^{2}$ The University of Hong Kong, Hong Kong \\
$^{3}$ King Abdullah University of Science and Technology, Saudi Arabia\\
$^{4}$  State University of New York at Albany, USA \\
$^{5}$  University of Sydney, Australia \\
}

 \date{}

\maketitle

\begin{abstract}
Recently, significant progress has been made in understanding the generalization of neural networks (NNs) trained by gradient descent (GD) using the algorithmic stability approach. However, most of the existing research has focused on one-hidden-layer NNs and has not   addressed the impact of different network scaling parameters. In this paper, we greatly extend the previous work \cite{lei2022stability,richards2021stability} by conducting a comprehensive stability and generalization analysis of GD for multi-layer NNs. For two-layer NNs, our results are established under general network scaling parameters, relaxing previous conditions. In the case of three-layer NNs, our technical contribution lies in demonstrating its nearly co-coercive property by utilizing a novel induction strategy that thoroughly explores the effects of over-parameterization. As a direct application of our general findings, we derive the excess risk rate of $\O(1/\sqrt{n})$ for GD algorithms in both two-layer and three-layer NNs. This sheds light on sufficient or  necessary conditions for under-parameterized and over-parameterized NNs trained by GD to attain the desired risk rate of $\O(1/\sqrt{n})$. Moreover, we demonstrate that as the scaling parameter increases or the network complexity decreases, less over-parameterization is required for GD to achieve the desired error rates. Additionally, under a low-noise condition, we obtain a fast risk rate of $\O(1/n)$ for GD in both two-layer and three-layer NNs. \let\thefootnote\relax\footnotetext{
$^{*}$The corresponding author is Yiming Ying. }  

\end{abstract}

\bigskip

\parindent=0cm

\section{Introduction}
Deep neural networks (DNNs) trained by (stochastic) gradient descent (GD) have achieved great  success in a wide spectrum of applications such as image recognition \cite{krizhevsky2017imagenet}, speech recognition \cite{hinton2012deep}, machine translation  \cite{bahdanau2014neural}, and reinforcement learning \cite{silver2016mastering}. In practical applications, most of the deployed DNNs are over-parameterized, i.e., the number of parameters is far larger than the size of the training data. 
In \cite{zhang2016understanding}, it was empirically demonstrated that over-parameterized NNs trained with SGD can generalize well to the test data while achieving a small training error.  This has triggered a surge of theoretical studies on unveiling this generalization mastery of DNNs.

In particular, norm-based generalization bounds are established using the uniform convergence approach \cite{bartlett2017spectrally,bartlett2021deep,golowich2018size,long2019generalization,neyshabur2015norm,neyshabur2018pac}.  However, this approach does not take the optimization algorithm and the data distribution into account.   Another line of work is to take into account the structure of the data distribution and provide algorithm-dependent generalization bounds. In \cite{brutzkus2017sgd,li2018learning}, it is shown that SGD for over-parameterized two-layer NNs can achieve small generalization error under certain assumptions on the structure of the data. \cite{allen2019learning} studies the generalization of SGD for two-layer and three-layer NNs if there exists a true (unknown) NN with low error on the data distribution. The other important line of work is the neural tangent kernel (NTK)-type approach \cite{arora2019fine,cao2019generalization,chizat2018global,nitanda2019gradient} which shows that the model trained by GD is well approximated by the tangent space near the initialization and the generalization analysis can be reduced to those of the convex case or kernel methods.  However, most of them either require a very  high over-parameterization or focus on special function classes.

Recently, the appealing work \cite{richards2021stability} provides an alternative approach in a kernel-free regime, using the concept of algorithmic stability \cite{bousquet2002stability,hardt2016train,kuzborskij2018data}.  Specifically, it uses the model-average stability \cite{lei2020fine} to derive generalization bounds of GD for two-layer over-parameterized NNs.  \cite{lei2022stability} improves this result by deriving generalization bounds for both GD and SGD, and relaxing the  over-parameterization requirement. \cite{taheri2023generalization} derives fast generalization bounds for GD under a separable distribution.  However, the above studies only focus on two-layer NNs and a very specific network scaling.  

\vspace*{-1mm}
\noindent {\bf Contributions.}  We study the stability and generalization of GD for both two-layer and three-layer NNs with generic network scaling factors. Our contributions are summarized as follows.

\vspace*{-1mm}
\hspace*{2mm}$\bullet$ We establish excess risk bounds for GD on both two-layer and three-layer NNs with general network scaling parameters under relaxed over-parameterization conditions. 
As a direct application of our generalization results, we show that GD can achieve the excess risk rate $\O(1/\sqrt{n})$ when the network width $m$ satisfies certain qualitative conditions related to the scaling parameter $c$, the size of training data $n$,  and the network complexity measured by the norm of the minimizer of the population risk. Further, under a low-noise condition, our excess risk rate can be improved to $\O(1/n)$ for both two-layer and three-layer NNs.

\vspace*{-1mm}
\hspace*{2mm}$\bullet$ A crucial technical element in the stability analysis for NNs trained by GD is establishing the almost co-coercivity of the gradient operator. This property naturally holds true for two-layer NNs due to the empirical risks' monotonically decreasing nature which no longer remains valid for three-layer NNs. Our technical contribution lies in demonstrating that the nearly co-coercive property still holds valid throughout the trajectory of GD. To achieve this, we employ a novel induction strategy that fully explores the effects of over-parameterization (refer to Section \ref{sec:idea} for further discussions). Furthermore, we are able to eliminate a critical assumption made in \cite{lei2022stability} regarding an inequality associated with the population risk of GD's iterates (see Remark \ref{remark:2layer} for additional details).

\vspace*{-1mm}
\hspace*{2mm}$\bullet$  Our results characterize a quantitative condition in terms of the network complexity and scaling factor under which  GD for two-layer and three-layer NNs can achieve the excess risk rate $\O(1/\sqrt{n})$  in under-parameterization
and over-parameterization regimes.  Our results shed light on sufficient or necessary conditions for under-parameterized and over-parameterized NNs trained by GD  to achieve the risk rate $\O(1/\sqrt{n}).$   In addition,   our results show that the larger the scaling parameter or the simpler the network complexity is, the less over-parameterization is needed for GD to achieve the desired error rates for multi-layer NNs.

\section{Problem Formulation}\label{sec:problem-formulation}

 Let $P$ be a probability measure defined on a sample space $\Z=\X\times \Y$, where $\X\subseteq \R^d$ and $\Y\subseteq \R $. Let $S=\{z_i=(\bx_i,y_i)\}_{i=1}^n$ be a training dataset drawn from $P$. One aims to build a prediction model $f_{\bW}:\X \mapsto \R$ parameterized by $\bW$ in some parameter space $\W$ based on $S$. The performance of $f_{\bW}$ can be measured by the population risk defined as  
 $L(f_{\bW})=\frac{1}{2} \iint_{\X\times\Y} \big( f_{\bW}(\bx)-y \big)^2dP(\bx,y).$
 The corresponding empirical risk is defined as
\begin{equation}\label{eq:erm}
    L_S(f_{\bW})= \frac{1}{2 n}\sum_{i=1}^n  \big( f_{\bW}(\bx_i)-y_i \big)^2. 
\end{equation}
 We denote by
$\ell( \bW;z)=\frac{1}{2}(f_\bW(\bx )-y)^2$ the loss function of $
\bW$ on a  data point $z=(\bx,y)$. 
 The best possible model   is the regression function $f^{*}$ defined as $f^{*}(\bx)=\E[y|\bx]$, where $\E[\cdot|\bx]$ is the conditional expectation given $\bx$. 
In this paper, we consider the prediction model $f_{\bW}$ with a neural network structure. In particular,  we are interested in two-layer and three-layer fully-connected NNs.

\noindent\textbf{Two-layer NNs.}  A two-layer NN of width $m>0$ and scaling parameter $c\in[1/2,1]$  takes the form 
\[f_\bW(\bx)=\frac{1}{m^c}\sum_{k=1}^m a_k \sigma( \bw_k \bx ),\] 
 where $\sigma:\R\mapsto \R$ is an activation function, $\bW=[\bw_1^{\top},\ldots,\bw_m^{\top}]^{\top} \in \W $ with $\W =\R^{m\times d}$ is the weight matrix of the first layer, and a fixed $\mathbf{a}=[a_1,\ldots,a_m]$ with $a_k\in\{-1,+1\}$ is the weight of the output layer. In the above formulation, $\bw_k\in\R^{1\times d}$ denotes the weight of the edge connecting the input to the $k$-th hidden node, and $a_k$ denotes the weight of the edge connecting the $k$-th hidden node to the output node.  The output of the network is scaled by a factor $m^{-c}$ that is decreasing with the network width $m$. Two popular choices of scaling are Neural Tangent Kernel (NTK) \cite{allen2019convergence,arora2019fine,arora2019exact,jacot2018neural,du2018gradient,du2019gradient} with $c=1/2$ and mean field \cite{chizat2018global,chizat2019lazy,mei2019mean,mei2018mean} with $c=1$. \cite{richards2021learning} studied two-layer NNs with $c\in[1/2,1]$ and discussed the influence of the scaling by trading off it with the network complexity. 
 We also focus on the scaling $c\in[1/2,1]$ to ensure that meaningful generalization bounds can be obtained. We fix the output layer weight $\mathbf{a}$ and only optimize the first layer weight $\bW$ in this setting. 
 
\noindent\textbf{Three-layer NNs.} 
For a matrix $\bW$, let $\bW_{s:}$ and $\bW_{is}$ denote the  $s$-th row  and the $(i,s)$-th entry of $\bW$. For an input $\bx\in\R^d$, a three-layer fully-connected NN with width $m>0$ and scaling $c\in(1/2,1]$  is 
 \[ f_{\bW}(\bx)=\frac{1}{m^c} \sum_{i=1}^m a_i \sigma
\big( \frac{1}{m^c} \sum_{s=1}^m \bW_{is}^{(2)} \sigma( \bW_{s:}^{(1)} \bx ) \big), \]
where $\sigma:\R\mapsto \R$ is an activation function,  $\bW=[\bW^{(1)},\bW^{(2)}]\in\W=\R^{m\times(d+m)}$  is the weight matrix of the neural network,    and  $\mathbf{a}=[a_1,\ldots,a_m]$ with $a_i\in\{-1,+1\}$ is the fixed output layer weight. Here,    $\bW^{(1)} \in \R^{m\times d}$  and $\bW^{(2)}   \in \R^{m\times m}$   are the weights at the first and the second layer, respectively.  We consider the setting of optimizing the weights in both the first and the second layers.  

We consider the gradient descent to solve the minimization problem: $\min_{\bW} L_S(f_\bW).$ For simplicity, let $L(\bW)=L(f_{\bW})$ and $L_S(\bW)=L_S(f_{\bW})$.   
\begin{definition}[Gradient Descent]
Let $\bW_0 \in \W $ be an initialization point, and $\{\eta_t: t\in \mathbb{N}\}$ be a sequence of step sizes. Let $\nabla$ denote the gradient operator.   At iteration $t$, the update rule of  GD is
\begin{equation}\label{eq:GD}
  \bW_{t+1}=\bW_{t} -\eta_t \nabla L_S(\bW_t).
\end{equation} 
\end{definition}
\noindent\textbf{Target of Analysis.} 
Let $\bW^{*}=\arg\min_{\bW\in \W} L(\bW)$ where the minimizer is chosen to be the one enjoying the smallest norm.  For a randomized algorithm $\A$ to solve \eqref{eq:erm}, let $\A(S)$ be the output of $\A$ applied to the  dataset $S$. The  generalization performance of $\A(S)$ is measured by its \textit{excess population risk}, i.e., $L(\A(S))-L(\bW^{*})$. In this paper, we are interested in studying the excess population risk of models trained by GD for both two-layer and three-layer NNs. 

We denote by $\|\cdot\|_2$  the standard Euclidean norm of a vector or a matrix. 
For any $\lambda>0$, let  $\bW^{*}_{\lambda} = \argmin_\bW \big\{L(\bW) + \frac{\lambda}{2} \| \bW- \bW_0 \|_2^2\big\}$.  
We have the following error decomposition
\begin{align}\label{eq:excess-decom}
      \E[ L(\A(S)) ] -  L(\bW^{*})   =  &   \E \big[ L(\A(S))   -  L_S(\A(S)) \big]  +  \E\big[ L_S(\A(S)) -   L_S(\bW^{*}_{\lambda} )  -  \frac{\lambda}{2 } \| \bW^{*}_{\lambda}  -  \bW_0 \|_2^2   \big]\nonumber\\
     &+ \big[ L(\bW^{*}_{\lambda} ) + \frac{\lambda}{2 } \| \bW^{*}_{\lambda} - \bW_0 \|_2^2  -  L(\bW^{*}) \big],
\end{align}
where we use $\E[L_S(\bW^{*}_{\lambda})]=L(\bW^{*}_{\lambda})$.
The first term in \eqref{eq:excess-decom} is called the {\em generalization error}, which can be controlled by  stability analysis. The second term is the {\em optimization error}, which can be estimated by tools from optimization theory. The third term, denoted by $\Lambda_{\lambda}=L(\bW^{*}_{\lambda} ) + \frac{\lambda}{2 } \| \bW^{*}_{\lambda} - \bW_0 \|_2^2  -  L(\bW^{*})$,  is the {\em approximation error} which will be estimated by introducing an assumption on the complexity of the NN (Assumption~\ref{ass:minimizer} in Section~\ref{sec:main-snn}). 

We will use the on-average argument stability \cite{lei2020fine} to study the generalization error defined as follows. 
\begin{definition}[On-average argument stability]\label{def:stability}
Let $S=\{z_i\}_{i=1}^n$ and $S'=\{z'_i \}_{i=1}^n$ be drawn i.i.d. from an unknown distribution $P$. For any $i\in[n]$, define $S^{(i)}=\{z_1,\ldots,z_{i-1},z'_i,z_{i+1},\ldots,z_n\}$ as the set formed from $S$ by replacing the $i$-th element with $z_i'$. We say $\A$ is on-average argument $\epsilon$-stable if \[\E_{S,S',\A}\big[\frac{1}{n}\sum_{i=1}^n\|\A(S)-\A(S^{(i)})\|_2^2\big] \le \epsilon^2.\]
\end{definition}

We say a function $\ell:\W\times \Z \rightarrow \R$ is $\rho$-smooth if for any $\bW,\bW' \in\W$ and $z\in\Z$, there holds $\ell(\bW;z)-\ell(\bW';z)\le \langle \nabla \ell(\bW';z), \bW-\bW' \rangle + \frac{\rho}{2}\|\bW-\bW'\|_2^2$.
The following lemma establishes the connection \cite{lei2020fine} between the on-average argument stability and the generalization error.
\begin{lemma}\label{lem:connection}
Let $\A$ be an algorithm. If for any $z$, the map $\bW \mapsto \ell(\bW;z)$ is $\rho$-smooth, then
\begin{align*}
       &\E[ L( \A(S)) -   L_S ( \A(S)) ]  \le   \frac{\rho}{2n} \sum_{i=1}^n   \E[ \|\A(S)  -  \A(S^{( i)} ) \|_2^2 ]   +  \bigl(  \frac{2\rho \E[L_S( \A(S)) ]}{n} \sum_{i=1}^n  \E[ \|\A(S) - \A(S^{( i)} )\|_2^2 ] \bigr)^{\frac{1}{2}}.
\end{align*}  
\end{lemma}
 
 Our analysis requires the following standard assumptions on activation and loss functions \cite{lei2022stability,richards2021stability}.  
\begin{ass}[Activation function]\label{ass:activation}
The activation function $a\mapsto\sigma(a)$ is continuous and twice differentiable with  $|\sigma(a)|\le B_{\sigma}$, $|\sigma'(a)|\le B_{\sigma'}$ and $|\sigma''(a)|\le B_{\sigma''}$, where $B_{\sigma}, B_{\sigma'},B_{\sigma''}>0 .$
\end{ass}
Both sigmoid and hyperbolic tangent activation functions satisfy Assumption \ref{ass:activation}~\cite{richards2021stability}. 
\begin{ass}[Inputs, labels and loss]\label{ass:loss}
 There exist constants $c_\bx, c_y, c_0 >0$ such that $\|\bx\|_2\le c_\bx$ , $|y|\le c_y$  and $\max\{\ell(\mathbf{0};z),\ell(\bW_0;z)\}\le c_0 $ for any $\bx\in \X, y\in\Y$ and $z\in\Z$. 
\end{ass}

\section{Main Results}\label{sec:main-resluts}
\vspace*{-2mm}
In this section, we present our main results on the excess population risk of NNs trained by GD. 

 \vspace*{-2mm}
\subsection{Two-layer Neural Networks with Scaling Parameters}\label{sec:main-snn}
 \vspace*{-1mm}

We denote $B\asymp B'$ if there exist some universal constants $c_1, c_2>0$ such that $c_1B\ge B'\ge c_2B$. We denote $B\gtrsim B'$ ($B\lesssim B'$) if there exists a constant  $c >0$ such that $B\ge c B'$ ($c B\le B'$). Let $\rho = c_\bx^2\big( \frac{  B_{\sigma'}^2 + B_\sigma B_{\sigma''}  }{m^{2c-1}} + \frac{ B_{\sigma''} c_y }{ m^c } \big).$ 
The following theorem presents generalization error bounds of GD for two-layer NNs. The proof can be found in Appendix~\ref{appendix:snn-generalization}. 
\begin{theorem}[Generalization error]\label{thm:generalization}
Suppose Assumptions~\ref{ass:activation} and \ref{ass:loss} hold. Let $\{\bW_t\}$  be produced by \eqref{eq:GD} with $\eta_t\equiv \eta\le 1/(2\rho)$ based on $S$.  Assume
\begin{align}\label{m-1}
     m \gtrsim \big( { (\eta T)^2 (1+\eta \rho)   \sqrt{  \rho (\rho \eta T + 2)} }\big/{n} \big)^{ \frac{2}{4c-1}} + (\eta T)^{\frac{3}{4c-1}}  +      \big( \eta T   \big)^{\frac{1}{c}}.
\end{align}
Then, for any $t\in[T]$,  there holds
\[ 
      \E[ L(\bW_t) - L_S(\bW_t) ] \le \big( \frac{4  e^2 \eta^2 \rho^2 t }{n^2} + \frac{4 e  \eta  \rho       }{n } \big)
      \sum_{j=0}^{t-1}\E\big[  L_S( \bW_j ) \big]. 
\]
\end{theorem}
\begin{remark}
Theorem~\ref{thm:generalization} establishes the first generalization results for general $c\in[1/2,1]$, which shows that the requirement on $m$
becomes weaker as $c$ becomes larger. Indeed, with a typical choice $\eta T\asymp n^{\frac{c}{6\mu+2c-3}}$ (as shown in Corollary~\ref{cor:excess-rate} below), the assumption \eqref{m-1} becomes $m \gtrsim  n^{\frac{1}{ 6\mu+2c-3 }(1+\frac{51-96\mu}{8(8c-3) })}$. This requirement becomes milder as $c$ increases for any $\mu\in[0,1]$, which implies that large scaling reduces the requirement on the network width. 
In particular, our assumption only  requires $m\gtrsim  \eta T$ when $c=1$.      
\cite{lei2022stability} provided a similar generalization bound with an assumption  $m\gtrsim (\eta T)^5/n^2 + (\eta T)^{2}$, and our result with $c=1/2$ is consistent with their result. 
\end{remark}

The following theorem to be proved in Appendix~\ref{appendix:snn-optimization} develops the optimization error bounds of GD.
\begin{theorem}[Optimization error]\label{thm:opt}
Suppose Assumptions~\ref{ass:activation} and \ref{ass:loss} hold.  Let $\{\bW_t\}$  be produced by \eqref{eq:GD} with  $\eta_t\equiv \eta\le 1/(2\rho)$ based on $S$. Let   
$\tilde{b}= c_\bx^2 B_{\sigma''} (  \frac{2B_{\sigma'} c_\bx}{m^{c-1/2}}   + \sqrt{2c_0}  )$  and  assume  \eqref{m-1} holds. Assume 
\begin{equation}\label{m-5}
  m \gtrsim \big(\tilde{b}T (  \sqrt{ \eta T } + \| \bW^{*}_{\frac{1}{\eta T}} -\bW_0 \|_2 ) \big( \frac{ e^2 \eta^3 \rho^2 T^2 }{ n^2} +     \frac{ e  \eta^2 T  \rho       }{n } + 1  \big) \big)^{ \frac{1}{c}}. 
\end{equation}
 Then we have
 \begin{align*}
     &\E[L_S(\bW_{T})] - L_S(\bW^{*}_{\frac{1}{\eta T}})  - \frac{   1 } {2\eta T}\|\bW_{\frac{1}{\eta T}}^{*} -\bW_0  \|_2^2 ]  \\
    &\!\le\!  \frac{2\tilde{b}  }{m^c } ( \!\sqrt{2 \eta T c_0} \!+\!  \|\bW_{\frac{1}{\eta T}}^{*} \!\!-\! \bW_0  \|_2  )   \big( \big( \frac{8  e^2 \eta^3 \rho^2 T^2 }{ n^2}\! +\!     \frac{8 e  \eta^2 T  \rho       }{n } \big)\!  \sum_{s=0}^{T-1} \!\E\big[ L_S(\bW_s ) \big]
     \! +\!   \|\bW_{\frac{1}{\eta T}}^{*}\! \!-\! \bW_0 \|_2^2  + \eta T \big[  L(\bW_{\frac{1}{\eta T}}^{*} ) \!\!-\! L(\bW^{*})\big]  \big) . 
\end{align*}
\end{theorem}
Recall that $\Lambda_{\frac{1}{\eta T}}=L(\bW^{*}_{\frac{1}{\eta T}} ) + \frac{1}{2\eta T } \| \bW^{*}_{\frac{1}{\eta T}} - \bW_0 \|_2^2  -  L(\bW^{*})$ is the approximation error. We can combine generalization and optimization error bounds together to derive our main result of GD for two-layer NNs.
It is worth mentioning that our excess population bound is dimension-independent, which is mainly due to  that stability analysis focuses on the optimization process trajectory, as opposed to the uniform convergence approach which involves the complexity of function space and thus is often dimension-dependent.  Without loss of generality, we assume $\eta T\ge 1$. 
\begin{theorem}[Excess population risk]\label{thm:excess}
Suppose Assumptions~\ref{ass:activation} and  \ref{ass:loss} hold. 
Let $\{\bW_t\}$ be produced by \eqref{eq:GD} with  $\eta_t\equiv \eta \le 1/(2\rho)$. Assume \eqref{m-1} and \eqref{m-5} hold. For any $c\in [1/2,1]$, if $\eta T m^{1-2c}=\O(n)$ and $m \gtrsim  (\eta T( \sqrt{\eta T} + \|\bW^{*}_{\frac{1}{\eta T}} - \bW_0 \|_2 ) )^{1/c} $, then there holds
\vspace*{-2mm}
\begin{equation*}
    \E[L(\bW_T)  -  L(\bW^{*})]
    = \O\big( \frac{  \eta  T  m^{1 - 2c}      }{n }    L(\bW^{*} )    + \Lambda_{\frac{1}{\eta T}} \big).
\end{equation*}    
\end{theorem}
\begin{remark}[Comparison with \cite{richards2021stability,lei2022stability}]\label{remark:2layer}
Theorem~\ref{thm:excess} provides the first excess risk bounds of GD for two-layer NNs with general scaling parameter $c\in[1/2,1]$ which recovers the previous work \cite{richards2021stability,lei2022stability} with $c=1/2$. Specifically, \cite{richards2021stability} derived the excess risk bound $\O\big(    \frac{  \eta  T   }{n }    L(\bW^{*} )    + \Lambda_{\frac{1}{\eta T}} \big)$ with $m\gtrsim (\eta T)^5$, we relax this condition to $m\gtrsim (\eta T)^3$ by providing a better estimation of the smallest eigenvalue of a Hessian matrix of the empirical risk (see Lemma~\ref{thm:stability}). This bound was obtained in \cite{lei2022stability}   under a crucial condition $\E[L(\bW_s)]\ge L(\bW^{*}_{\frac{1}{\eta T}})$ for any $s\in[T]$, which is difficult to verify in practice. Here, we are able to remove this condition by using  $L(\bW^{*}_{\frac{1}{\eta T}}) - L(\bW_s)\le L(\bW^{*}_{\frac{1}{\eta T}}) - L(\bW^{*} ) \le \Lambda_{\frac{1}{\eta T}}$ since $L(\bW_s)\ge L(\bW^{*})$ when controlling the bound of GD iterates (see Lemma~\ref{lem:distance} for details).  
Furthermore,  our result in Theorem \ref{thm:excess} implies that the larger the scaling parameter $c$ is, the better the excess risk bound is. The reason could be that the smoothness parameter related to the objective function along the GD trajectory becomes smaller for a larger $c$ (see Lemma~\ref{lem:smoothness} for details).  
\end{remark}
As a direct application of Theorem \ref{thm:excess}, we can derive the excess risk rates of GD for two-layer NNs by properly trade-offing the generalization, optimization and approximation errors. 
To this end, we further introduce the following assumption on network complexity to control the approximation error by noting that $\Lambda_{\frac{1}{\eta T}}\le \|\bW^{*}\|_2^2/2\eta T$.  
\begin{ass}[\cite{richards2021learning}]\label{ass:minimizer}
There is $\mu\!\in\![0,1]$  and population risk minimizer $\bW^{*}$ with  $\|\bW^{*}\|_2\!\le\! m^{\frac{1}{2}-\mu}$. 
\end{ass}


\begin{corollary} \label{cor:excess-rate} 
Let assumptions in Theorem~\ref{thm:excess} hold and Assumption  \ref{ass:minimizer} hold. 
Suppose ${c\over 3}+\mu>{1\over 2}$.
 \begin{enumerate}[label=(\alph*), leftmargin=*]
  \setlength\itemsep{-2mm}
     \item If $c\in[1/2,3/4) $ and $c+\mu\ge 1$,    
we can choose $m\asymp (\eta T)^{\frac{3}{2c}}$ and $\eta T$ such that $ n^{ \frac{c}{6\mu+2c-3} } \lesssim \eta T \lesssim n^{\frac{c}{3-4c}}$. If $c\in[3/4,1]$, we can choose  $m\asymp (\eta T)^{\frac{3}{2c}}$ and $\eta T$ such that $ \eta T \gtrsim n^{ \frac{c}{6\mu+2c-3} }$. Then there holds 
\[\E[L(\bW_T) - L(\bW^{*})] = \O\big(\frac{1}{\sqrt{n}} \big).\] 

\item Assume   $L(\bW^{*})=0$. 
We can choose   $m\asymp (\eta T)^{\frac{3}{2c}}$ and $\eta T$ such that $ \eta T \gtrsim n^{\frac{2c}{6\mu+2c-3}}$, and get 
\[\E[L(\bW_T) - L(\bW^{*})] 
     = \O\big( \frac{1}{n}  \big).\]
 \end{enumerate}
\end{corollary}

\begin{figure}[t!]
\begin{subfigure}{.5\textwidth}
\centering
\includegraphics[width=.94\linewidth]{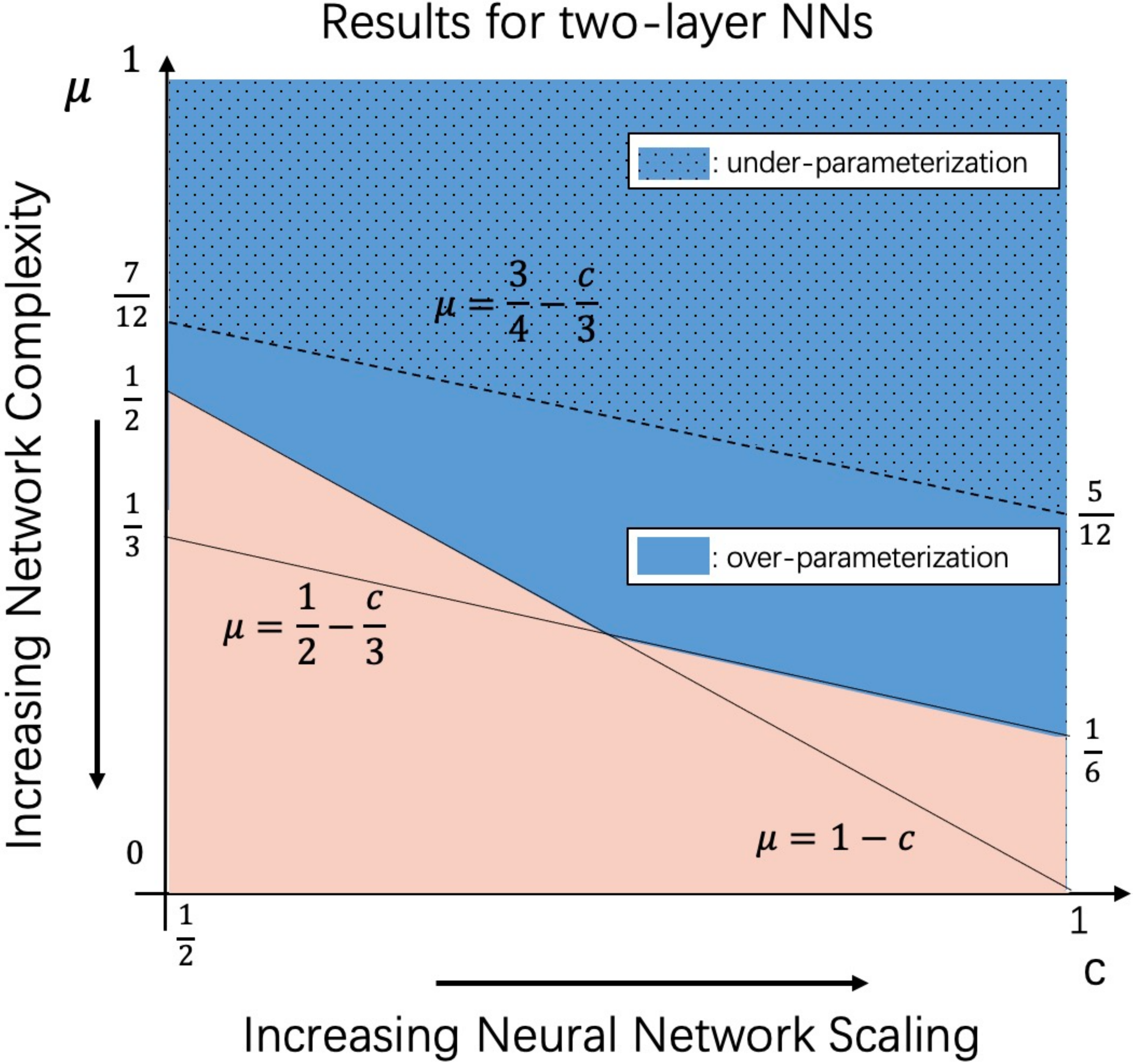}  
\end{subfigure}
\begin{subfigure}{.5\textwidth}
\centering
\includegraphics[width=.95\linewidth]{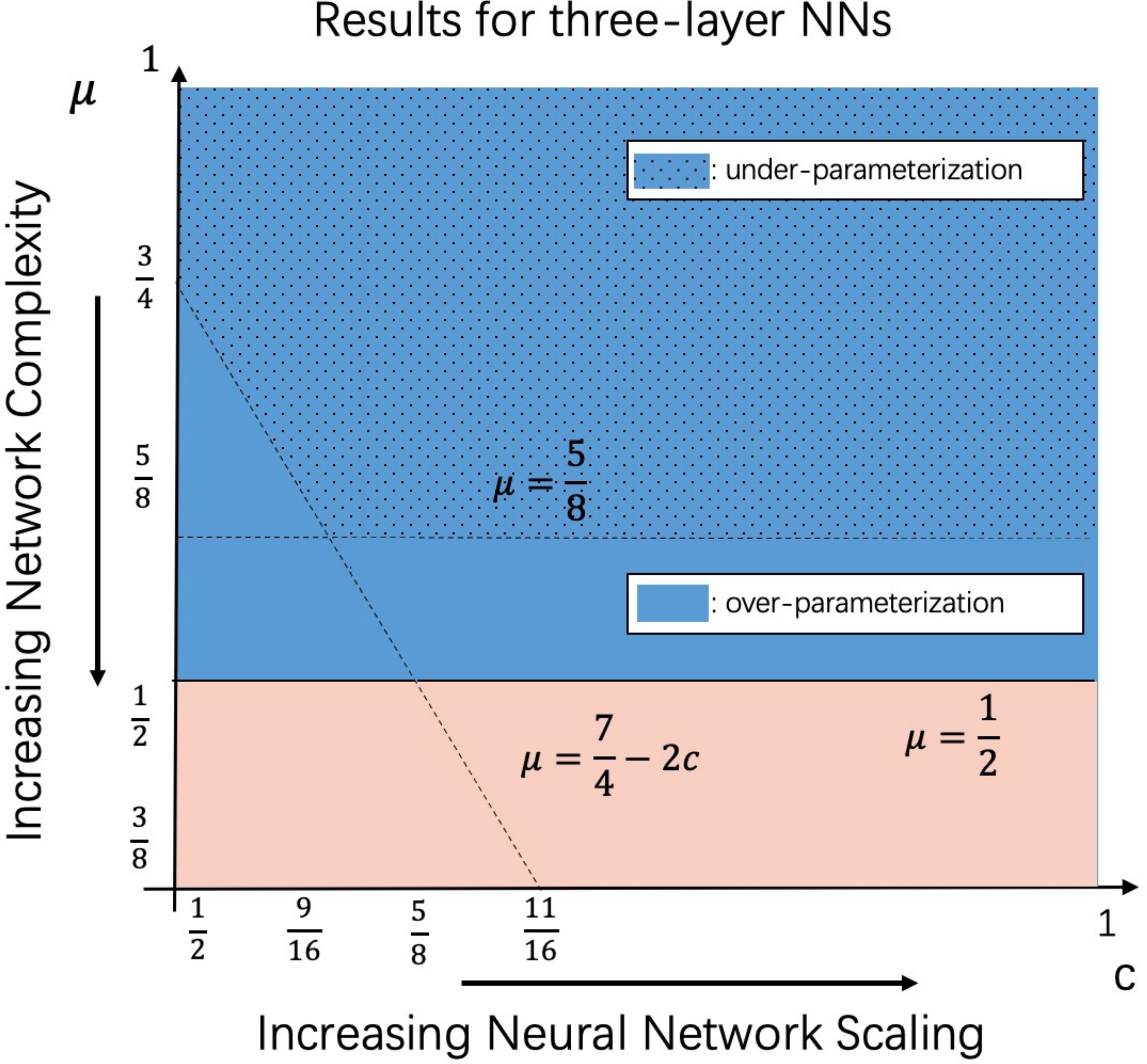}  
\end{subfigure}
\caption{\small Scaling parameter $c$ versus Network complexity parameter $\mu$ for Part (a) in Corollary~\ref{cor:excess-rate} (left) and Part (a) in Corollary~\ref{cor:3-excess-rate-2} (right).  \textit{Blue Region without dots}:  values of $c$ and $\mu$ where  over-parameterization is necessary to achieve error bound $\O(1/\sqrt{n})$. \textit{Blue  Region with dots}:  values of $c$ and $\mu$ where under-parameterization is sufficient to achieve error bound $\O(1/\sqrt{n})$. \textit{Pink Region}: the desired bound cannot be 
guaranteed.\label{fig-1}}
\vspace*{-3mm}
\end{figure}
 
 \noindent{\bf Interpretation of the Results:}  Corollary \ref{cor:excess-rate}  indicates a quantitative condition via the network complexity and scaling factor where GD for two-layer NNs can achieve the excess risk rate $\O(1/\sqrt{n})$ in  under-parameterization and over-parameterization regimes. We use the left panel of Figure \ref{fig-1} to interpret the results. 
 
Let us first explain the meanings of the regions and lines in the left panel of Figure \ref{fig-1}. Specifically, the \textit{blue regions} (with or without dots)  correspond to the conditions ${c\over 3}+\mu>{1\over 2}$ and $c+\mu\ge 1$ in part (a) of Corollary \ref{cor:excess-rate}, while our results do not hold in the \textit{pink region} which violates the conditions in part (a).     Furthermore, under conditions ${c\over 3}+\mu>{1\over 2}$ and $c+\mu\ge 1$, that   the desired bound in part (a) can be achieved by choosing $m\asymp (\eta T)^{\frac{3}{2c}}$ for any $\eta T$ satisfying $\eta T\gtrsim n^{\frac{c}{6\mu+2c-3}}$ if $c\ge 3/4$ and $n^{\frac{c}{3-4c}} \gtrsim \eta T \gtrsim  n^{ \frac{c}{6\mu+2c-3} }   $ if $c\in[1/2,3/4)$, which further implies that  GD with any width $m\gtrsim n^{\frac{3}{2(6\mu+2c-3)}}$ if $c\in[3/4,1]$ and $n^{\frac{3}{2(3-4c)}}\gtrsim m \gtrsim   n^{\frac{3}{2(6\mu+2c-3)}} $ if $c\in[1/2,3/4)$ with suitable iterations can achieve the error rate  $\O(1/\sqrt{n})$. This observation   tells us  that \textit{\bf the smallest width} for guaranteeing our results in part (a) is $m\asymp n^{\frac{3}{2(6\mu+2c-3)}}$ for any $c\in[1/2,1]$.   The dotted line $\mu=3/4-c/3$ in the figure corresponds to the setting $m\asymp n$, i.e.,  $   \frac{3}{2(6\mu+2c-3)}= 1$. Correspondingly, we use the {\em dotted blue region} above the dotted line to indicate the {under-parameterization region} and {\em the blue region without dots} below the dotted line for the {over-parameterization region}. With the above explanations, we can interpret the left panel of  Figure \ref{fig-1} as follows. 

\vspace*{-1mm}
 $\bullet$  Firstly, from the figure, we know that, if values of $c$ and $\mu$ are located above the dotted line $\mu\ge 3/4-c/3$, i.e., the blue region with dots, under-parameterization is {\em sufficient} for GD to achieve the error rate  $\O(1/\sqrt{n})$. It implies that the sufficient condition for under-parameterized NNs trained by GD achieving the desired rate is $\mu\ge 3/4-c/3$. The potential reason  is that the population risk minimizer $\bW^{*}$ is well-behaved in terms of its norm being relatively small with $\mu$ being relatively large there. In particular,  when $\mu > 1/2$, $\|\bW^{*}\|_2 \le m^{1/2-\mu}$ tends to $0$ as $m$ tends to infinity. Hence,  it is expected that  under-parameterized NNs can learn this relatively simple $\bW^*$ well.   However, it is worthy of mentioning that over-parameterization can also achieve the rate   $\O(1/\sqrt{n})$ since the dotted line only indicates the smallest width required for achieving such an error rate.  

\vspace*{-1mm}
 $\bullet$ Secondly, from the figure, we see that, if $c$ and $\mu$ belong to the blue region without dots which is between the solid lines and the dotted line, then over-parameterization is {\em necessary} for achieving the error rate $\O(1/\sqrt{n}).$  This is because, in the blue region without dots, that the conditions of choosing $m\gtrsim n^{\frac{3}{2(6\mu+2c-3)}}$ in part (a) of Corollary \ref{cor:excess-rate} which will always indicate the over-parameterization region, i.e., $m \gtrsim n.$ Furthermore, from the above discussions, our theoretical results indicate that the over-parameterization does bring benefit for GD to achieve good generalization in the sense that GD can achieve  excess risk  rate $\O(1/\sqrt{n})$ when $c$ and $\mu$ is in the whole blue region (with or without dots) while under-parameterization can only do so for the blue region with dots where the network complexity is relatively simple, i.e., $\mu$ is relatively large.

 \vspace*{-1mm}

 $\bullet$ Thirdly, our results do not hold for GD when values of $c$ and $\mu$ are in the pink region in the figure. In particular, when $\mu<1/6$, our bounds do not hold for any $c\in[1/2,1]$.  We suspect that this is due to the artifacts of our analysis tools and it remains an open question to us whether we can get a generalization error bound $\O(1/\sqrt{n})$ when $\mu<1/6.$ In addition,  our results in Corollary~\ref{cor:excess-rate} also indicate that the requirement on $m$ becomes weaker as $c$ and $\mu$ become larger. It implies that networks with larger scaling and simpler network complexity are biased to weaken the over-parameterization   for GD to achieve the desired error rates for two-layer NNs. 

\vspace*{-2mm}
\begin{remark}
In Lemma~\ref{lem:smoothness}, we show  that $f_{\bW}$ is $B_{\sigma'}c_{\bx} m^{\frac{1}{2}-c}$-Lipschitz. Combining this result with Assumption~\ref{ass:minimizer}  we know $|f_{\bW^{*}}(\bx)-f_{\mathbf{0}}(\bx) |^2=\O(m^{2(1-c-\mu)})$. In order for  $|f_{\bW{*}}(\bx)-f_{\mathbf{0}}(\bx) |^2$ to  not  vanish as $m$ tends to infinity, one needs $c+\mu\le 1$. In Corollary~\ref{cor:excess-rate},  we also need $\frac{c}{3}+\mu>\frac{1}{2}$ to ensure the excess risk bounds vanish. Combining these two conditions together implies that $c$ can not be larger than $3/4$. That is, for the range $c\in(3/4,1]$,  the conditions in Corollary~\ref{cor:excess-rate} restrict the class of functions the networks can represent as $m$ tends to infinity. However, we want to emphasize that even for the simplest case that $|f_{\bW^{*}}(\bx)-f_{\mathbf{0}}(\bx)|^2$ tends to $0$ as $m$ tends to infinity, our results still imply that over-parameterization does
bring benefit for GD to achieve optimal excess risk rate $\O(1/\sqrt{n})$.
Besides, our corollary mainly discusses the conditions for achieving the excess risk rate $\O(1/\sqrt{n})$  and $\O(1/n)$.  The above-mentioned conditions will be milder if we consider the slower excess risk rates. Then the restriction on $c$ will be weaker.       
Furthermore, our main result (i.e., Theorem~\ref{thm:excess})  does not rely on  Assumption~\ref{ass:minimizer}, and it holds for any setting.
\end{remark} 
\vspace*{-2mm}
 \noindent {\bf Comparison with the Existing Work:} 
 Part (b) in Corollary~\ref{cor:excess-rate} shows fast rate $\O(1/n)$ can be derived under a low-noise condition $L(\bW^{*})=0$ which is equivalent to the fact that there is a true network such that $ y = f_{\bW^{*}}(\bx)$ almost surely. 
   Similar to part (a),  large $c$ and large $\mu$  also help weaken the requirement on the width $m$ in this case. 
  For a special case   $\mu=1/2$  and $c= 1/2$,  \cite{lei2022stability}  proved that GD for two-layer NNs achieves the excess risk rate $\O(1/\sqrt{n})$ with $m\asymp n^{3/2}$ and $\eta T\asymp \sqrt{n}$ in the general case, which is further improved to $\O(1/n)$ with $m\asymp n^3$ and $\eta T\asymp n$ in a low-noise case. Corollary~\ref{cor:excess-rate} recovers their results with the same conditions on $m$ and $\eta T$ for this setting. 

 \cite{richards2021learning} studied GD with weakly convex losses and showed that the excess population risk is controlled by
 $\O\big( \frac{\eta T L^2}{n} 
+  \frac{\|\bW_{\epsilon}-\bW_0\|_2^2}{\eta T}+ \epsilon(\eta T + \|\bW^{*}-\bW_0\|_2) \big)$ if $2\eta \epsilon<1/T$ when the empirical risk is $\epsilon$-weakly convex and $L$-Lipschitz continuous, where $\bW_{\epsilon}=\arg\min_{\bW}L_S(\bW)+\epsilon\|\bW-\bW_0\|_2^2$. If the approximation error is small enough, then the $\O(1/\sqrt{n})$ bound can be achieved by choosing $\eta T=\sqrt{n}$ if $\|\bW_{\epsilon}\|_2=\O(1)$. Indeed, their excess risk bound will not converge for the general case. Specifically, note that $L_S(\bW_\epsilon)+\epsilon\|\bW_\epsilon-\bW_0\|_2^2\le L_S(0)+\epsilon\|\bW_0\|_2^2$, then there holds $\|\bW_{\epsilon}-\bW_0\|_2^2=\O(1/\epsilon)$. 
The simultaneous appearance of $ \frac{1}{\eta T\epsilon} $ and $\eta T \epsilon  $ causes the non-vanishing error bound.   \cite{richards2021learning} also investigated the weak convexity of two-layer NNs with a smooth activation function. Under the assumption that the derivative of the loss function is uniformly bounded by a constant,  
they proved that the weak convexity parameter is controlled by  $\O(d/m^c)$. We provide a dimension-independent weak convexity parameter which further yields a dimension-independent excess risk rate $\O(1/\sqrt{n})$. More discussion can be found in Appendix~\ref{appendix:snn-excess}.

\subsection{Three-layer Neural Networks with Scaling Parameters }\label{sec:main-three}
 
Now, we present our results for three-layer NNs.   Let $\hat{\rho}=4B_2(1+2B_1)$, where $B_1,B_2>0$ are  constants depending on $c_{\bx}, B_{\sigma'}, B_{\sigma''}$ and $c_0$, whose specific forms are given in Appendix~\ref{appendix:three-generalization}. Let $\mathcal{B}_{T}=\sqrt{\eta T}+\|\bW_0\|_2$. 
We first present the generalization bounds for three-layer NNs.  
\begin{theorem}[Generalization error]\label{thm:3-unbounded-generalization}
Suppose Assumptions~\ref{ass:activation} and \ref{ass:loss} hold. Let $\{\bW_t\}$  be produced by \eqref{eq:GD} with $\eta_t\equiv \eta\le 1/(8\hat{\rho})$ based on $S$. Assume
\begin{align}\label{eq:3-m-condition-1}
        &m \gtrsim    (\eta T)^{4} +  (\eta T)^{\frac{1}{4c-2}} +  \|\bW_0\|_2^{\frac{4}{8c-3}} + \|\bW_0\|_2^{\frac{1}{6c-3}} + \big( (\eta T  \mathcal{B}_{T} )^2   +   \frac{(\eta T)^{\frac{7}{2}} \mathcal{B}_{T} }{n} \big)^{\frac{1}{5c - \frac{1}{2}}} \nonumber\\
        & +  \big( (\eta T  )^\frac{3}{2}\mathcal{B}_{T}^2 +  \frac{(\eta T)^3 \mathcal{B}_{T} }{n} \big)^{\frac{1}{4c - 1}}  +   \big( (\eta T)^2  \mathcal{B}_{T}  + \frac{(\eta T)^{\frac{7}{2}}}{n} \big)^{\frac{1}{5c - 1}} + \big( (\eta T)^{\frac{3}{2}}  \mathcal{B}_{T}   + \frac{(\eta T)^3}{n} \big)^{\frac{1}{4c - \frac{3}{2}}}.
     \end{align}
     Then, for any    $t\in[T]$,  
 \begin{equation*}
     \E[ L(\bW_t) - L_S(\bW_t) ] \le \Big( \frac{4  e^2 \eta^2 \hat{\rho}^2 t }{n^2} + \frac{4 e  \eta  \hat{\rho}       }{n } \Big)
      \sum_{j=0}^{t-1} \E\big[  L_S( \bW_j ) \big]. 
\end{equation*}
\end{theorem}
Similar to Theorem~\ref{thm:generalization}, Theorem~\ref{thm:3-unbounded-generalization} also implies that a larger scaling $c$ relaxes the requirement on   $m$.
\begin{remark}\label{rmk:difference}
As compared to two-layer NNs,  the analysis of three-layer NNs is more challenging since we can only show that  $\lambda_{\max}\big(\nabla^2\ell(\bW;z)\big)\le \rho_{\bW}$ where $\rho_\bW$ depends on $\|\bW\|_2$, i.e., the smoothness parameter of $\ell$ relies on the upper bound of $\bW$, while that of two-layer NNs is uniformly bounded. In this way, three-layer NNs do not enjoy the almost co-coercivity, which is the key step to control the stability of GD. To handle this problem,   we first establish a crude estimate $\|\bW_t-\bW_0\|_2\le \eta t m^{2c-1}$ for any $c>1/2$ by induction strategy. By using this estimate, we can further show that $\rho_{\bW}\le \hat{\rho}$ with $\hat{\rho}=\O(1)$ for any $\bW$ produced by GD iterates if $m$ satisfies \eqref{eq:3-m-condition-1}.  Finally, by assuming $\eta \le 1/(2\hat{\rho})$ we build the upper bound  $\|\bW_t-\bW_0\|_2\le \sqrt{2c_0\eta t}$. However, for the case $c=1/2$, we cannot get a similar bound due to the condition $m^{2-4c}\|\bW_t-\bW_0\|_2^2=\O(m^{2c-1})$. Specifically, the upper bound of $\|\bW_t-\bW_0\|_2$ in this case contains a worse term  $2^t$, which is not easy to control. Therefore, we only consider $c\in(1/2,1]$ for three-layer NNs. 
The estimate of $\|\bW_t-\bW_0\|_2$  when $c=1/2$ remains an open problem.   
The detailed proof of the theorem is given in Appendix~\ref{appendix:three-generalization}.  
\end{remark}
Let  $\hat{C}_\bW=4B_1\big(  m^{-3c}(\|\bW\|_2^2 + 2c_0\eta T) +  m^{\frac{1}{2}-2c}(\|\bW\|_2 +  \sqrt{2c_0\eta T}) \big)$ and $\hat{B}_{\bW }= \big(\frac{B_{\sigma'}^2 c_\bx}{m^{2c-1/2}} ( \sqrt{2 c_0\eta T } +  \| \bW   \|_2) + \frac{B_{\sigma'}B_{\sigma}}{m^{2c-1}}\big)( 2\sqrt{2\eta T c_0} + \|\bW  \|_2 )  + \sqrt{2c_0}$ .
The following theorem gives optimization error bounds for three-layer NNs. The proof is given in Appendix~\ref{appendix:three-optimization}. 
 \begin{theorem}[Optimization error]\label{thm:opt-unbounded}
Let Assumptions~\ref{ass:activation}, \ref{ass:loss} hold. Let  $\{\bW_t\}$  be produced by \eqref{eq:GD} with   $\eta_t\equiv \eta\le 1/(8\hat{\rho})$.  Let $\mathcal{C}_{T,n}=\eta T + \eta^3T^2/n^2$. Assume \eqref{eq:3-m-condition-1} and 
\begin{align}\label{m-condition-unbounded2}
      m \gtrsim & (\mathcal{C}_{T,n}(  \mathcal{B}_T + \|\bW^{*}_{\frac{1}{\eta T}}\|_2 )^4)^{\frac{1}{5c-\frac{1}{2}}} \!+\! (\mathcal{C}_{T,n}(  \mathcal{B}_T +\|\bW^{*}_{\frac{1}{\eta T}}\|_2 )^3)^{\frac{1}{4c-1}} \!+\! (\mathcal{C}_{T,n}(  \mathcal{B}_T + \|\bW^{*}_{\frac{1}{\eta T}}\|_2 )^2)^{\frac{1}{4c-\frac{3}{2}}} \nonumber\\
      &+ (\mathcal{C}_{T,n}(  \mathcal{B}_T + \|\bW^{*}_{\frac{1}{\eta T}}\|_2 ))^{\frac{1}{ 2c-\frac{1}{2}}}   .
\end{align}   
 Then we have
 \begin{align*}
      & \E\big[ L_S(\bW_{T})]-  L_S(\bW^{*}_{\frac{1}{\eta T}})  -   \frac{1}{2\eta T}\|\bW^{*}_{\frac{1}{\eta T}}  - \bW_0 \|_2^2\big] \\
    &\le          \hat{C}_{\bW^{*}_{\frac{1}{\eta T}}}   \hat{B}_{\bW^{*}_{\frac{1}{\eta T}}}  \big(  \big( \frac{4  e^2 \eta^3  T^2 \hat{\rho}^2}{ n^2} + \frac{4 e  \eta^2 T  \hat{\rho}  }{n }  \big)  \sum_{s=0}^{T-1 } \E \big[  L_S(\bW_s) \big]   +    \|\bW_{\frac{1}{\eta T}}^{*} -\bW_0 \|_2^2    +   \eta T \big[  L(\bW_{\frac{1}{\eta T}}^{*}) -  L(\bW^{*})\big]\big).  
 \end{align*}
\end{theorem}
Now, we develop excess risk bounds of GD for three-layer NNs by combining Theorem~\ref{thm:3-unbounded-generalization} and Theorem~\ref{thm:opt-unbounded} together. The proof is given in Appendix~\ref{appendix:three-excess}. 
\begin{theorem}[Excess population risk]\label{thm:excess-unbounded}
Suppose Assumptions~\ref{ass:activation} and  \ref{ass:loss}   hold.   
Let $\{\bW_t\}$ be produced by \eqref{eq:GD} with  $\eta\le 1/(8\hat{\rho})$. Assume \eqref{eq:3-m-condition-1}  and \eqref{m-condition-unbounded2} hold. For any $c\in (1/2,1]$, if $n\gtrsim \eta  T  $, then there holds
\[ \E[L(\bW_T)\! -\! L(\bW^{*})]
    =\O\Big(   \frac{  \eta  T   }{n }  L(\bW^{*} ) + \Lambda_{\frac{1}{\eta T}}  \Big).\]
\end{theorem}
Finally, we establish  excess risk bounds of GD for three-layer NNs by assuming Assumption~\ref{ass:minimizer} holds. 
 \begin{corollary}\label{cor:3-excess-rate-2}
Let assumptions in Theorem~\ref{thm:excess-unbounded} and Assumption~\ref{ass:minimizer} hold. The following statements hold.
 \begin{enumerate}[label=(\alph*), leftmargin=*]
  \setlength\itemsep{-2mm}
     \item Assume $\mu\ge 1/2$.  If $c\in [9/16,1]$,  we  can choose  $m\asymp (\eta T)^{4}$ and $\eta T$ such that  $n^{\frac{ 1}{2(8\mu-3)}}\lesssim \eta T\lesssim \sqrt{n}$. If $c\in(1/2,9/16)$, we can choose $m\asymp (\eta T)^{\frac{1}{4c-2}}$ and $\eta T$ such that $n^{\frac{2c-1}{2\mu+4c-3}}\lesssim \eta T\lesssim \sqrt{n}$.   
     Then  \[\E[ L(\bW_T) - L( \bW^{*}  ) ] =  \O\big(\frac{1}{\sqrt{n}}\big). \]
\item Assume   $L(\bW^{*})=0$ and $\mu\ge 1/2$. If $c\in [9/16,1]$,  we  can choose  $m\asymp (\eta T)^{4}$ and $\eta T\gtrsim  n^{\frac{ 1}{8\mu-3}}$. If $c\in(1/2,9/16)$, we can choose $m\asymp (\eta T)^{\frac{1}{4c-2}}$ and $\eta T$ such that $\eta T\gtrsim  n^{\frac{4c-2}{4c+2\mu-3}}$.   Then \[     \E[ L(\bW_T) - L( \bW^{*}  ) ] =  \O\big(\frac{1}{n}\big). \]
 \end{enumerate}
\end{corollary}
\noindent {\bf Discussion of the Results: }
Part (a) in Corollary~\ref{cor:3-excess-rate-2} shows GD for three-layer NNs can achieve the excess risk rate $\O(1/\sqrt{n})$ with $m\asymp (\eta T)^{4}$ and $n^{\frac{ 1}{2(8\mu-3)}}\lesssim \eta T\lesssim \sqrt{n}$ for the case  $c\in[9/16,1]$  and $m\asymp (\eta T)^{\frac{1}{4c-2}}$ and   $n^{\frac{2c-1}{2\mu+4c-3}}\lesssim \eta T\lesssim \sqrt{n}$ for the case $c\in(1/2,9/16)$, respectively. 
Note that there is an additional assumption $\mu\ge 1/2$ in part (a). Combining this assumption with  $\|\bW^{*}\|_2\le m^{\frac{1}{2}-\mu}$ together, we know that the population risk minimizer $\bW^{*}$ cannot be too large to reach the power of the exponent of $m$. The potential reason is that we use a constant to bound the smoothness parameter in the analysis for three-layer NNs. 
 Part (a) also indicates a quantitative condition in terms of $m$  and $c$ where GD for three-layer NNs can achieve the excess risk rate $\O(1/\sqrt{n})$ in under-parameterization and over-parameterization regimes, which is interpreted in the right panel of Figure \ref{fig-1}. 
 The results in part (a) tell us that the smallest width for guaranteeing the desired bounds are $ m\asymp n^{\frac{2}{ 8\mu-3 }} $ for $c\in[9/16,1]$ and $ m\asymp n^{\frac{1}{2(2\mu+4c-3)}} $ for $c\in(1/2,9/16)$. Similar to the left panel of Figure~\ref{fig-1}, the dotted lines $\mu=\frac{5}{8}  $ and $\mu=\frac{7}{4}-2c$ in the right panel of Figure \ref{fig-1} correspond to the setting $m\asymp n$, i.e., $\frac{2}{ 8\mu-3 }=1$ and $\frac{ 1}{2(2\mu+4c-3)}=1$. 
 Hence, when $c$ and $\mu$ belong to the blue region with dots, under-parameterization is \textit{sufficient} to achieve the desired rate. When $c$ and $\mu$ are located in the blue region without dots  which is between the solid lines and the dotted line,  over-parameterization  is \textit{necessary} for GD to achieve the rate $\O(1/\sqrt{n}
 )$. 
  Our results for three-layer NNs also imply that the over-parameterization does bring benefit for GD to achieve good generalization in the sense that GD can achieve  excess risk  rate $\O(1/\sqrt{n})$.   
Under a low-noise condition, part (b) implies that the excess risk rate can be improved to $\O(1/n)$ with suitable choices of $m$ and $\eta T$. 
These results imply that the larger the scaling parameter is, the less over-parameterization is needed for GD to achieve the desired error rate for both the  general case and  low-noise case.  

\noindent{\bf Comparison with the Existing Work:}  
\cite{richards2021learning} studied the minimal eigenvalue of the empirical risk Hessian for a  three-layer NN with a linear activation in the first layer for   Lipschitz and convex losses (e.g.,  logistic loss),  while we focus on NNs with more general activation functions for least square loss. See more detailed discussion in Appendix~\ref{appendix:disc}.
\cite{ju2022generalization} studied the generalization performance of overparameterized three-layer NTK models with the absolute loss and ReLU activation. They showed that  the generalization error is in the order of $\O(1/\sqrt{n})$ when there are infinitely many neurons. They only trained the middle-layer weights of the networks. To the best of our knowledge, our work is the first study on the stability and generalization of GD to train both the first and the second layers of the network in the kernel-free regime.

\section{Main Idea of the Proof}\label{sec:idea}

In this section, we present  the main ideas for proving the main results in Section \ref{sec:main-resluts}.

\textbf{Two-layer Neural Networks.}
 \eqref{eq:excess-decom} decomposes the excess population risk into three terms: generalization error, optimization error and approximation error. We estimate these three terms separately.

\noindent\textit{Generalization error.} 
From Lemma~\ref{lem:connection} we know the generalization error can be upper bounded by the on-average argument stability of GD. Hence, it remains to figure out the on-average argument stability of GD. A key step in our stability analysis is to show that the loss is strongly smooth and weakly convex, which can be obtained by  the following results given in Lemma~\ref{lem:smoothness}: 
\vspace*{-1mm}
\[ \lambda_{\max} \!\big(\nabla^2  \ell(\bW;z)\big)\!\le\!\rho  \text{ and }   \lambda_{\min}\! \big(\nabla^2 \ell( \bW;z) \big) \!\ge\! - \big(    c_\bx^3 B_{\sigma'}\! B_{\sigma''}  m^{\frac{1}{2}-2c} \| \bW\!-\!\bW_{\!0}  \|_2 +  c_\bx^2 B_{\sigma}  \!\sqrt{2c_0}m^{-c} \!\big).  \]
Here $ \rho = \O( m^{1-2c} )$, $\lambda_{\max}(A)$ and $\lambda_{\min}(A)$ denote the largest and the smallest eigenvalue of $A$, respectively.
The lower bound of $  \lambda_{\min}\! (\nabla^2 \ell( \bW;z) ) $ is related to $\|\bW - \bW_0 \|_2$, we further  show it is uniformly bounded (see the proof of Theorem~\ref{thm:stability}).
Then the smoothness and weak convexity of the loss scale with $m^{1-2c}$ and $m^c$, 
which implies that the loss becomes more smooth and more convex for wider networks  with a larger~scaling.   

Based on the above results, we derive the following uniform stability bounds
$ \| \bW_t - \bW_t^{(i)} \|_2   \le  \frac{2  \eta   e T \sqrt{ 2c_0 \rho (\rho \eta T+2) }     }{n }.$ 
 Here,  the weak convexity of the loss plays an important role in presenting the almost co-coercivity of the gradient operator, which  helps us establish the recursive relationship of $\|\bW_t-\bW_t^{(i)}\|_2$. 
Note $\rho=\O(m^{1-2c})=\O(1)$. 
Then we know $\| \bW_t - \bW_t^{(i)} \|_2=\O(\frac{m^{1-2c}}{n}(\eta T)^{3/2})$, which implies that GD is more stable for a wider neural network with a large scaling.  For a specific case $c=1/2$, our  stability bound is in the order of $\O((\eta T)^{3/2}/n)$,  and matches the result in~\cite{lei2022stability}. 
 
\noindent\textit{Optimization error.} A key step in optimization error analysis is to use the smoothness and weak convexity to control 
\[\E[\|\bW_{\!\frac{1}{\eta T}}^{*}\! -\! \bW_t  \|_2^2 ]\! \le\!  \big(\frac{8  e^2 \eta^3 \rho^2 t^2 }{ n^2}\! + \!\frac{8 e  \eta^2 t  \rho       }{n } \big)\! \sum_{s=0}^{t-1}  \E[  L_S(\bW_{ s} ) ] \! +\!   2\E\big[  \|\bW_{\frac{1}{\eta T}}^{*} \! -\! \bW_{ 0}   \|_2^2  \big] \! +\! 2\eta T  \big[ L(\bW_{ \frac{1}{\eta T}}^{*})\!-\!L(\bW^{*})\big] \]
in Lemma~\ref{lem:smoothness}.  
Here, to remove the condition $\E[L(\bW_s)]\ge L(\bW^{*}_{\frac{1}{\eta T}})$ for any $s\in[T]$ in \cite{lei2022stability}, we  use 
$
L(\bW^{*}_{\frac{1}{\eta T}}) - L(\bW_s) \le L(\bW^{*}_{\frac{1}{\eta T}}) - L(\bW^{*} ) \le  \Lambda_{\frac{1}{\eta T}},
$
where the first inequality is due to $L(\bW_s)\ge L(\bW^{*})$ for any $s\in[T]$. 
Then the optimization error can be controlled by  the monotonically decreasing property of $\{L_S(\bW_t)\}$. The proofs are given in Appendix~\ref{appendix:snn-optimization}.  

\noindent\textit{Excess population risk.}
Combining stability bounds, optimization bounds and approximation error together, and noting that $L(\bW^{*}_{\frac{1}{\eta T}})+\frac{1}{2\eta T}\|\bW^{*}_{\frac{1}{\eta T}}-\bW_0\|_2^2\le L(\bW^{*} )+\frac{1}{2\eta T}\|\bW^{*}-\bW_0 \|_2^2$,  one can get the final error bound. 
The detailed proof can be found in Appendix~\ref{appendix:snn-excess}.  

\textbf{Three-layer Neural Networks.} 
The basic idea to study the generalization behavior of GD for three-layer NNs is similar to that of two-layer NNs.  We develop stability bounds  and control the generalization error by figuring out the  smoothness and curvature of the loss function:  
 $\lambda_{\max} (\nabla^2 \ell(\bW;z) ) \le \rho_\bW $ and
$\lambda_{\min} (\nabla^2 \ell( \bW;z)  )\!\ge\! -  C_\bW  (   \frac{2 B_{\sigma'}B_\sigma  }{m^{2c-1}} \|\bW^{(2)}\|_2  +  \sqrt{2c_0} ), $
where $\rho_\bW$ and $C_\bW$ depend on $\|\bW\|_2$. The specific forms of $\rho_\bW$ and $C_\bW$ are given in Appendix~\ref{appendix:three-generalization}. 
As mentioned in Remark~\ref{rmk:difference}, it is not easy to estimate $\|\bW_t\|_2$ since the smoothness of the loss relies on the norm of $\bW$ in this case. We address this difficulty by first giving a rough bound of $\|\bW_t-\bW_0\|_2$ by induction, i.e., $\|\bW_t-\bW_0\|_2\le \eta t m^{2c-1}$. Then, for any $\bW$ produced by GD iterates, we can control $\rho_\bW$ by a constant $\hat{\rho}$ if $m$ is large enough. Finally, by assuming $\eta\hat{\rho}\le 1/2$, we prove that $\|\bW_t-\bW_0\|_2\le \sqrt{2c_0\eta t}$ and further get $\|\bW_t\|_2\le \|\bW_0\|_2+\sqrt{2c_0\eta t}$.  

After estimating $\|\bW_t\|_2$,  
we can develop the following  almost co-coercivity of the gradient operator:  
\begin{align*}
     \langle \bW_t - \bW_t^{(i)}, \nabla L_{S^{\setminus i}}(\bW_t)  - \nabla L_{S^{\setminus i}}(\bW^{(i)}_t)  \rangle  
     \ge & 2\eta \big(1- 4\eta \hat{\rho} \big)\big\| \nabla L_{S^{\setminus i}}(\bW_t) - \nabla L_{S^{\setminus i}}(\bW^{(i)}_t)   \big\|_2^2 \\& - \tilde{\epsilon}_t \big\| \bW_t -\bW_t^{(i)}  -\eta \big( \nabla L_{S^{\setminus i}}(\bW_t) - \nabla L_{S^{\setminus i}}(\bW_t^{(i)}) \big) \big\|_2^2.
\end{align*}
It helps to establish the uniform stability bound  $
     \big\| \bW_{t } -\bW_{t }^{(i)} \big\|_2 \le   2 e   \eta T  \sqrt{2c_0\hat{\rho}   (  \hat{\rho}\eta T   +   2 )} /n.$ 
  Based on  stability bounds, we get generalization bounds by Lemma~\ref{lem:connection}. The proofs are given in Appendix~\ref{appendix:three-generalization}.

Similar to two-layer NNs, to estimate the optimization error for three-layer NNs, we first control $\E[\|\bW^{*}_{\frac{1}{\eta T}}-\bW_t\|_2^2]$ by using the smoothness and weak convexity of the loss. Then the desired bound can be obtained by the monotonically decreasing property of $\{L_S(\bW_t)\}$. The final error bounds (Theorem~\ref{thm:excess-unbounded} and Corollary~\ref{cor:3-excess-rate-2}) can be derived by plugging generalization and optimization bounds  
back into  \eqref{eq:excess-decom}. The detailed proof can  be found in Appendix~\ref{appendix:three-excess}.

\section{Conclusion}\label{sec:conclu}
We present stability and generalization analysis of GD for multi-layer NNs with generic network scaling factors. Under some qualitative conditions on the network scaling and the network complexity,  we establish excess risk bounds of the order $\O(1/\sqrt{n})$ for GD on both two-layer and three-layer NNs, which are improved to $\O(1/n)$ with an additional low-noise condition. Our results describe  a quantitative condition related to the scaling factor and the network complexity under which GD on two-layer and three-layer NNs can achieve the desired excess risk rate.

There remain several questions for further study. The first question is whether our analysis of GD for multi-layer NNs can be extended to SGD with less computation cost. The key challenge here is that the analysis for GD relies critically on the monotonicity of the objective functions along the optimization process, which does not hold for SGD.     Second, our analysis for three-layer NNs does not hold for $c=1/2$. It would be interesting to develop a result for this setting. 
Finally, 
the results in Corollaries  \ref{cor:excess-rate} and \ref{cor:3-excess-rate-2} hold true for $\mu \ge 1/6$ or $\mu\ge 1/2$. It remains an open question to us whether we can get a generalization error bound $\O(1/\sqrt{n})$ when $\mu$ is small.

\medskip 

\noindent{\bf Acknowledgement.}  
The work of Ding-Xuan Zhou was partially supported by the Laboratory for AI-Powered Financial Technologies under the InnoHK scheme. The corresponding author is Yiming Ying whose work was supported by  NSF research grants (DMS-2110836, IIS-2103450, and IIS-2110546). Di Wang was supported in part by  BAS/1/1689-01-01, URF/1/4663-01-01, FCC/1/1976-49-01 of KAUST, and a funding of the SDAIA-KAUST AI center.

\onecolumn
\appendix
\numberwithin{equation}{section}
\numberwithin{theorem}{section}
\numberwithin{remark}{section}
\numberwithin{figure}{section}
\numberwithin{table}{section}
\renewcommand{\thesection}{{\Alph{section}}}
\renewcommand{\thesubsection}{\Alph{section}.\arabic{subsection}}
\renewcommand{\thesubsubsection}{\Roman{section}.\arabic{subsection}.\arabic{subsubsection}}
\setcounter{secnumdepth}{-1}
\setcounter{secnumdepth}{3}

\vspace*{0cm}
\begin{center}
  \Large \textbf{Appendix for ``Generalization Guarantees of Gradient Descent for Multi-Layer Neural Networks''}
\end{center}

\section{Proofs of Two-layer Neural Networks}\label{appendix:snn}
\subsection{Proofs of Generalization Bounds}\label{appendix:snn-generalization}
We first introduce the  self-bounding property of smooth  functions  \cite{srebro2010smoothness}.  
\begin{lemma}[Self-bounding property]\label{lem:self-bounding}
Suppose for all $z\in\Z$, the function $\bW\mapsto \ell(\bW;z)$ is nonnegative and $\rho$-smooth. Then $\|\nabla \ell(\bW;z)\|_2^2\le 2\rho \ell(\bW;z)$. 
\end{lemma} 

We work with vectorized quantities so $\bW\in\R^{md}$. Then $\nabla f_\bW(\bx)\in\R^{md}$ and $\nabla^2f_\bW(\bx) \in\R^{md\times md}.$  Denote by $\|\bW\|_{op}$ the spectral norm of a matrix $\bW$.
 We first introduce the following lemma, which shows that the loss function is  smooth and weakly convex. 
 \begin{lemma}[Smoothness and Curvature]\label{lem:smoothness}
Suppose Assumptions~\ref{ass:activation} and \ref{ass:loss} hold. Let $\bW_0$ be the initial point of GD.  
For any fixed $\bW  \in \R^{m\times d}$ and any $z\in\Z$,  there holds
\[ \lambda_{\max}\!\big(\nabla^2\! \ell(\bW;z)\big)\!\le\!\rho \text{ with } \rho\!=\!c_\bx^2\Big( \frac{\! B_{\sigma'}^2\!+\!B_\sigma\!B_{\sigma''}\! }{m^{2c-1}}\!+\!\frac{\!B_{\sigma''} c_y }{ m^c }\!\Big). \]
\[ \lambda_{\min}\!\big(\nabla^2\!\ell( \bW;z) \big)\!\ge\!-\!\Big(\!\frac{  c_\bx^3 B_{\sigma'}\!B_{\sigma''}\!}{ m^{2c-\frac{1}{2}} }\!\| \bW\!-\!\bW_0  \|_2\!+\!\frac{\!c_\bx^2 B_{\sigma}\! \sqrt{2c_0}\!}{ m^c }\Big). \]
\end{lemma}
\begin{proof}
Recall that $f_\bW(\bx)=\frac{1}{m^c}\sum_{k=1}^m a_k \sigma( \bw_k \bx ) $. 
Let $\mathbf{v}=(\mathbf{v}_1,\ldots,\mathbf{v}_m) \in\R^{dm}$ with $\mathbf{v}_k\in\R^d$ where $k\in[m]$. According to Assumption~\ref{ass:activation},\ref{ass:loss} and noting that $|a_k|=1$, we can give the following estimations:
\begin{align}\label{eq:loss-1}
    \|\nabla^2 f_\bW(\bx) \|_{op} &= \max_{\mathbf{v}:\|\mathbf{v}\|_2
    \le 1} \sum_{k=1}^m \frac{a_k}{m^c} \langle \mathbf{v}_k, \bx \rangle^2 \sigma''(\bw_k\bx)\nonumber\\
    &\le \frac{\|\bx\|_2^2 B_{\sigma''}}{m^c}  \max_{\mathbf{v}:\|\mathbf{v}\|_2
    \le 1} \sum_{k=1}^m \|\mathbf{v}_k\|_2^2\nonumber\\
    &\le \frac{ c_\bx^2 B_{\sigma''} }{m^c},
\end{align}
\begin{equation}\label{eq:nabla-f}
     \|\nabla f_\bW(\bx)\|_2^2 = \sum_{k=1}^m \big\| \frac{a_k}{m^c} \bx \sigma'(\bw_k \bx) \big\|_2^2 \le \frac{  B_{\sigma'}^2 c_\bx^2}{m^{2c-1}} 
\end{equation}  
and
\begin{equation}\label{eq:f-y}
    \big| f_\bW(\bx) - y \big|\le \big| f_\bW(\bx) \big|+c_y\le m^{1-c} B_\sigma   + c_y .
\end{equation}  
Note 
\begin{equation}\label{eq:hessian}
    \nabla^2 \ell(\bW;z)= \nabla f_\bW(\bx) \nabla f_\bW(\bx)^{\top} + \nabla^2 f_\bW(\bx)\big( f_\bW(\bx) -y \big)  .
\end{equation}
Then for any $\bW \in\R^{md}$, we can upper bound the maximum eigenvalue of the Hessian  as
\begin{align*}
   \lambda_{\max}(\nabla^2 \ell( \bW;z ))&\le \| \nabla f_{\bW}(\bx) \|_2^2 + \|\nabla^2 f_\bW(\bx)\|_{op} |f_\bW(\bx) -y| \\
    &\le \frac{  B_{\sigma'}^2 c_\bx^2}{m^{2c-1}}  + \frac{ c_\bx^2 B_{\sigma''} }{m^c}\big( m^{1-c} B_\sigma   + c_y \big)\\
    &=  c_\bx^2 \Big(\frac{  B_{\sigma'}^2  +  B_\sigma   B_{\sigma''} }{m^{2c-1}}  + \frac{   B_{\sigma''} c_y }{m^c} \Big), 
\end{align*}
which implies that the loss function is $\rho$-smooth with $\rho= c_\bx^2 \big(\frac{  B_{\sigma'}^2  +  B_\sigma   B_{\sigma''} }{m^{2c-1}}  + \frac{   B_{\sigma''} c_y }{m^c} \big)$.

For any $\bW, \bW' \in \R^{md}$, from Assumption~\ref{ass:activation} we know
\begin{equation}\label{eq:loss-2}
     \big| f_\bW(\bx) -   f_{\bW'}(\bx) \big| \le \frac{1}{m^c}\sum_{k=1}^m \big| \sigma(\bw_k \bx) - \sigma(\bw_k'\bx)  \big|\le \frac{ B_{\sigma'}}{m^c}\sum_{k=1}^m\big|(\bw_k-\bw_k')\bx\big|\le \frac{c_\bx B_{\sigma'}}{m^{c-1/2}}\|\bW-\bW'\|_2 . 
\end{equation} 
Combining \eqref{eq:hessian} with the fact that $\nabla f_\bW(\bx) \nabla f_\bW(\bx)^{\top}$ is positive semi-definite together, we obtain
\begin{align}\label{eq:loss-3}
    \lambda_{\min}( \nabla^2 \ell(\bW;z) )&\ge  -  \| \nabla^2 f_\bW(\bx)  \|_{op}\big|  f_\bW(\bx) -y \big| \nonumber\\
    & \ge -  \frac{ c_\bx^2 B_{\sigma''} }{m^c}\Big( \big|  f_\bW(\bx) -  f_{\bW_0}(\bx) \big| + \big|   f_{\bW_0}(\bx)-y  \big|  \Big)\nonumber\\
    &\ge -  \frac{ c_\bx^2 B_{\sigma^{''}} }{m^c}\Big(  \frac{c_\bx B_{\sigma'}}{m^{c-1/2}}\|\bW-\bW_0\|_2 +  \sqrt{2\ell(\bW_0;z)}\Big)\nonumber\\
    &\ge -  \frac{ c_\bx^2 B_{\sigma^{''}} }{m^c}\Big(  \frac{c_\bx B_{\sigma'}}{m^{c-1/2}}\|\bW-\bW_0\|_2 +  \sqrt{2c_0}\Big),
\end{align}
where in the third inequality we used \eqref{eq:loss-2} with $\bW'=\bW_0$ and in the last inequality we used Assumption~\ref{ass:loss}.  
The proof is completed.
\end{proof}

To give an upper bound of the uniform stability, we need the following lemma which shows how the GD iterate will deviate from the initial point.  
\begin{lemma}[\cite{richards2021stability}]\label{lem:deviation} Suppose the loss is $\rho$-smooth and $\eta\le 1/(2\rho)$. Then for any $t\ge 0$, $i\in[n]$,  
\[ \|\bW_t -\bW_0\|_2 \le  \sqrt{2\eta t L_S(\bW_0)} ,\]
\[ \|\bW_t^{(i)} -\bW_0\|_2 \le  \sqrt{2\eta t L_{S^{(i)}}(\bW_0)}. \]
\end{lemma}

The following lemma shows an almost co-coercivity of the gradient operator associated with shallow neural networks. For any $i\in[n]$, define $S^{(i)}=\{z_1,\ldots,z_{i-1},z'_i,z_{i+1},\ldots,z_n\}$ as the set formed from $S$ by replacing the $i$-th element with $z_i'$. For any $\bW\in\W$, 
\[ L_{S^{\setminus i}}(\bW)=L_S(\bW) - \frac{1}{ n}\ell(\bW;z_i) = L_{S^{(i)}}(\bW)-\frac{1}{ n}\ell(\bW;z_i') .\]
Let $\{\bW_t\}$ and $\{\bW_t^{(i)}\}$ be the sequence produced by GD based on $S$ and $S^{(i)}$, respectively. 
\begin{lemma}[Almost Co-coercivity of the Gradient Operator]\label{lem:coercivity}
Suppose the loss is $\rho$-smooth and $\eta\le1/(2\rho)$.  Then
\begin{align*}
    \langle \bW_t - \bW_t^{(i)}, \nabla L_{S^{\setminus i}}(\bW_t) - \nabla L_{S^{\setminus i}}&(\bW^{(i)}_t)  \rangle \ge   2\eta \Big(1-\frac{\eta \rho}{2}\Big)\big\| \nabla L_{S^{\setminus i}}(\bW_t) - \nabla L_{S^{\setminus i}}  (\bW^{(i)}_t)   \big\|_2^2\\
     & - \epsilon_t \big\| \bW_t -\bW_t^{(i)} -\eta \big( \nabla L_{S^{\setminus i}}(\bW_t) - \nabla L_{S^{\setminus i}}(\bW_t^{(i)}) \big) \big\|_2^2,
\end{align*}
where $\epsilon_t = \frac{ c_\bx^2 B_{\sigma''} }{m^c}\Big( \frac{ c_\bx B_{\sigma'} }{m^{c-1/2}}(1+\eta \rho)\big\| \bW_t - \bW_t^{(i)} \big\|_2 + \frac{c_\bx B_{\sigma'}\sqrt{2\eta Tc_0}  }{m^{c-1/2}}+  \sqrt{2c_0} \Big)$. 
\end{lemma}
\begin{proof}
This lemma can be proved in a similar way as Lemma 5 in \cite{richards2021stability} except the estimation of the eigenvalue of Hessian matrix. Specifically,  for $\alpha\in[0,1]$, let $ \bW(\alpha) = \alpha \bW_t  + (1-\alpha)\bW_t^{(i)} -\alpha \eta \big( \nabla \ell(\bW_t;z) -\nabla \ell(\bW_t^{(i)};z) \big)$.  
According to  \eqref{eq:loss-1} and \eqref{eq:hessian} , for any $\bW\in\W$, we know
\begin{align*}
    \lambda_{\min}\big( 
\nabla^2L_{S^{\setminus  i}}( \bW )\big)\ge - \frac{ c_\bx^2 B_{\sigma''} }{ m^c} \Big(\frac{1}{ n} \sum_{j\in[n], j\neq i} |f_\bW(\bx_j) - y_j|\Big).
\end{align*}
Let $\bW=\bW(\alpha)$. Note  Lemma~\ref{lem:smoothness} shows that the loss is $\rho$-smooth   with $\rho=c_\bx^2 \big(\frac{  B_{\sigma'}^2  +  B_\sigma   B_{\sigma''} }{m^{2c-1}}  + \frac{   B_{\sigma''} c_y }{m^c} \big)$. Then from   \eqref{eq:loss-2} and the smoothness of $\ell$ we can get 
\begin{align*}
    &\frac{1}{n} \sum_{j\in[n], j\neq i} |f_{\bW(\alpha)}(\bx_j) - y_j|\\&\le  \frac{1}{ n} \sum_{j\in[n], j\neq i} |(f_{\bW(\alpha)}(\bx_j) - f_{\bW_t^{(i)}}(\bx_j)) + (f_{\bW_t^{(i)}}(\bx_j)- f_{\bW_0}(\bx_j)) +(f_{\bW_0}(\bx_j) - y_j)|\\
    &\le \frac{c_\bx B_{\sigma'}}{m^{c-1/2}}\|\bW(\alpha)-\bW_t^{(i)}\|_2+ \frac{c_\bx B_{\sigma'}}{m^{c-1/2}}\| \bW_t^{(i)}-\bW_0\|_2 +\sqrt{2L_{S^{\setminus i}}( \bW_0 )}\\
    &\le  \frac{c_\bx B_{\sigma'} \alpha}{m^{c-1/2}}\big(\|\bW_t -\bW_t^{(i)}\|_2+  \eta \| \nabla \ell( \bW_t;z ) - \nabla \ell( \bW_t^{(i)};z ) \|_2\big)+  \frac{c_\bx B_{\sigma'}\sqrt{2\eta Tc_0}  }{m^{c-1/2}} +\sqrt{2c_0}\\
    &\le  \frac{c_\bx B_{\sigma'} (1+\eta \rho) }{m^{c-1/2}} \|\bW_t -\bW_t^{(i)}\|_2  + \frac{c_\bx B_{\sigma'}\sqrt{2\eta Tc_0}  }{m^{c-1/2}} +\sqrt{2c_0}, 
\end{align*}
where in the third inequality we used Lemma~\ref{lem:deviation} and $\ell(\bW_0;z)\le c_0$. 

Combining the above two inequalities together, we get
\begin{align*}
    \lambda_{\min}\big( 
\nabla^2L_{S^{\setminus i}}( \bW(\alpha) )\big) 
&\ge - \frac{ c_\bx^2 B_{\sigma''} }{m^c} \Big( \frac{c_\bx B_{\sigma'} (1+\eta \rho) }{ m^{c-1/2}} \|\bW_t -\bW_t^{(i)}\|_2 + \frac{c_\bx B_{\sigma'}\sqrt{2\eta Tc_0}  }{m^{c-1/2}} +\sqrt{2c_0}\Big).
\end{align*}

Similarly, let $ \widetilde{\bW}(\alpha) = \alpha \bW_t^{(i)}  + (1-\alpha)\bW_t  -\alpha \eta \big(\nabla \ell(\bW_t^{(i)};z))- \nabla \ell(\bW_t;z)  \big)$,  we can prove that
\begin{align*}
    \lambda_{\min}\big( 
\nabla^2L_{S^{\setminus  i}}( \widetilde{\bW}(\alpha) )\big)& \ge - \frac{ c_\bx^2 B_{\sigma''} }{m^c} \Big( \frac{c_\bx B_{\sigma'} (1+\eta \rho) }{ m^{c-1/2}} \|\bW_t -\bW_t^{(i)}\|_2 + \frac{c_\bx B_{\sigma'}\sqrt{2\eta Tc_0}  }{m^{c-1/2}} +\sqrt{2c_0}\Big).
\end{align*}
The remaining arguments in proving the lemma are the same as Lemma 5 in \cite{richards2021stability}. We omit the proof for simplicity. 
\end{proof}

 Based on the almost co-coercivity property of the gradient operator, we give the following uniform stability theorem. 
 \begin{theorem}[Uniform Stability]\label{thm:stability}
Suppose Assumptions~\ref{ass:activation} and \ref{ass:loss} hold. Let $S, S^{(i)}$ be constructed in Definition~\ref{def:stability}. Let $\{\bW_t\}$ and $\{\bW_t^{(i)}\}$ be produced by \eqref{eq:GD} with $\eta\le 1/(2\rho)$ based on $S$ and $S^{(i)}$, respectively. Assume \eqref{m-1} holds.  For any $t\in[T]$, there holds
\[\| \bW_t - \bW_t^{(i)} \|_2   \le  \frac{2  \eta   e T \sqrt{ 2c_0 \rho (\rho \eta T+2) }     }{n }   . \]
\end{theorem}
 \begin{proof}
Recall that 
\[ L_{S^{\setminus i}}(\bW)=L_S(\bW) - \frac{1}{n}\ell(\bW;z_i) = L_{S^{(i)}}(\bW)-\frac{1}{n}\ell(\bW;z_i') .\]
Note $\W=\R^{m\times d}$. Then by the update rule $\bW_{t+1}=\bW_t - \eta \nabla L_S(\bW_t)$, there holds
\begin{align}\label{eq:stab-1}
    &\big\| \bW_{t+1} -\bW_{t+1}^{(i)} \big\|_2^2\nonumber\\
    & = \big\| \bW_{t} -\bW_{t}^{(i)} -\eta\big(  \nabla L_{S^{\setminus i}}(\bW_t) -  \nabla L_{S^{\setminus i}}(\bW_{t}^{(i)} ) \big) - \frac{\eta}{ n}\big( \nabla \ell(\bW_t;z_i) - \nabla \ell(\bW_t^{(i)};z_i') \big) \big\|_2^2 \nonumber\\
    &\le (1+p)\big\| \bW_{t} -\bW_{t}^{(i)} -\eta\big(\nabla  L_{S^{\setminus i}}(\bW_t) -\nabla  L_{S^{\setminus i}}(\bW_{t}^{(i)} ) \big) \big\|_2^2 \nonumber\\&\quad + \frac{\eta^2 (1+1/p)}{ n^2}\big\| \nabla \ell(\bW_t;z_i) - \nabla \ell(\bW_t^{(i)};z_i')   \big\|_2^2\nonumber\\
    &\le (1+p)\big\| \bW_{t} -\bW_{t}^{(i)} -\eta\big(\nabla  L_{S^{\setminus i}}(\bW_t) -\nabla  L_{S^{\setminus i}}(\bW_{t}^{(i)} ) \big) \big\|_2^2 \nonumber\\&\quad+ \frac{ 2 \eta^2 (1+1/p)}{ n^2}\Big(\big\| \nabla \ell(\bW_t;z_i)\big\|_2^2+ \big\| \nabla \ell(\bW_t^{(i)};z_i')   \big\|_2^2\Big),
\end{align}
where in the first inequality we used $(a+b)^2\le (1+p)a^2 + (1+1/p)b^2$. 

According to Lemma~\ref{lem:coercivity} we can get 
\begin{align*}
     &\big\| \bW_{t} -\bW_{t}^{(i)} -\eta\big( \nabla L_{S^{\setminus i}}(\bW_t) -\nabla  L_{S^{\setminus i}}(\bW_{t}^{(i)} ) \big) \big\|_2^2 \\
    &= \big\| \bW_{t} -\bW_{t}^{(i)}   \big\|_2^2 +\eta^2 \big\| \nabla  L_{S^{\setminus i}}(\bW_t) - \nabla L_{S^{\setminus i}}(\bW_{t}^{(i)} )   \big\|_2^2  \\&\quad- 2\eta \Big\langle \bW_{t} -\bW_{t}^{(i)},\nabla L_{S^{\setminus i}}(\bW_t) - \nabla L_{S^{\setminus i}}(\bW_{t}^{(i)} )   \Big\rangle\\
    &\le \big\| \bW_{t} -\bW_{t}^{(i)}   \big\|_2^2 +\eta^2 \big\|  \nabla L_{S^{\setminus i}}(\bW_t) - \nabla L_{S^{\setminus i}}(\bW_{t}^{(i)} )   \big\|_2^2  \\&\quad - 4\eta^2 \Big(1-\frac{\eta \rho}{2}\Big)\big\| \nabla L_{S^{\setminus i}}(\bW_t) - \nabla L_{S^{\setminus i}}(\bW^{(i)}_t)   \big\|_2^2\\
    & \quad + 2\eta \epsilon_t \big\| \bW_t -\bW_t^{(i)} -\eta \big( \nabla L_{S^{\setminus i}}(\bW_t) - \nabla L_{S^{\setminus i}}(\bW_t^{(i)}) \big) \big\|_2^2,
\end{align*}
where  $\epsilon_t = \frac{ c_\bx B_{\sigma''} }{m^c}\big( \frac{ c_\bx B_{\sigma'} }{ m^{c-1/2}}(1+\eta \rho)\big\| \bW_t - \bW_t^{(i)} \big\|_2+ \frac{c_\bx B_{\sigma'}\sqrt{2\eta Tc_0}  }{m^{c-1/2}} +  \sqrt{2c_0} \big)$. 

Rearranging the above inequality and noting that $\eta \rho \le 1/2$, we obtain
\begin{align*}
    &(1-2\eta \epsilon_t)\big\| \bW_t -\bW_t^{(i)} -\eta \big( \nabla L_{S^{\setminus i}}(\bW_t) - \nabla L_{S^{\setminus i}}(\bW_t^{(i)}) \big) \big\|_2^2\\
    & \le  \big\| \bW_{t} -\bW_{t}^{(i)}   \big\|_2^2 +\eta^2 ( 2\eta \rho -3 ) \big\|  \nabla L_{S^{\setminus i}}(\bW_t) - \nabla L_{S^{\setminus i}}(\bW_{t}^{(i)} )   \big\|_2^2  \le  \big\| \bW_{t} -\bW_{t}^{(i)}   \big\|_2^2 . 
\end{align*}
We can choose $m$ large enough to ensure $2\eta \epsilon_t < 1$ holds for any $t\in[T]$. Indeed, $2\eta \epsilon_t < 1$ holds as long as condition~\eqref{m-1} holds. We will discuss it at the end of the proof. Now,  plugging the above inequality back into \eqref{eq:stab-1} yields
\begin{align}\label{eq:stab-2}
     &\big\| \bW_{t+1} -\bW_{t+1}^{(i)} \big\|_2^2 \nonumber\\
     &\le  \frac{1+p}{ 1-2\eta \epsilon_t }\big\| \bW_{t} -\bW_{t}^{(i)}   \big\|_2^2  + \frac{  2\eta^2 (1+1/p)}{ n^2}\Big(\big\| \nabla \ell(\bW_t;z_i)\big\|_2^2+ \big\| \nabla \ell(\bW_t^{(i)};z_i')   \big\|_2^2\Big).
\end{align}

We can apply  \eqref{eq:stab-2} recursively and derive
\begin{align}\label{eq:stab-3}
     \big\| \bW_{t+1} -\bW_{t+1}^{(i)} \big\|_2^2\le &  \frac{ 2 \eta^2 (1+1/p)}{ n^2} \sum_{j=0}^t  \Big(\big\| \nabla \ell(\bW_j;z_i)\big\|_2^2+ \big\| \nabla \ell(\bW_j^{(i)};z_i')   \big\|_2^2\Big)\prod_{\tilde{j}=j+1}^{t} \frac{1+p}{ 1-2\eta \epsilon_{\tilde{j}} },
\end{align}
where we used $\bW_0=\bW_0^{(i)}$.

According to Lemma~\ref{lem:self-bounding} and Lemma~\ref{lem:deviation}, we know
\begin{align*}
     \| \nabla \ell(\bW_j ;z) \|_2^2 &\le 2\| \nabla \ell(\bW_j;z) - \nabla \ell(\bW_0;z) \|_2^2 + 2 \|  \nabla \ell(\bW_0;z) \|_2^2\\
     & \le 2\rho^2 \|\bW_j -\bW_0\|_2^2 + 4\rho \ell(\bW_0;z)\le 4\rho^2  \eta j L_{S} (\bW_0) + 4\rho \ell(\bW_0;z). 
\end{align*}
Similarly, we can show that
\begin{align*}
     \| \nabla \ell(\bW_j^{(i)} ;z) \|_2^2  \le 4 \rho^2  \eta j L_{S^{(i)}}(\bW_0) + 4\rho \ell(\bW_0;z). 
\end{align*}
Combining the above three inequalities together, we get 
\begin{align*}
     &\big\| \bW_{t+1} -\bW_{t+1}^{(i)} \big\|_2^2\\
     &\le    \frac{8\rho \eta^2 (1+1/p)}{n^2} \sum_{j=0}^t  \Big(  \rho   \eta j L_{S} (\bW_0) +   \ell(\bW_0;z_i) +   \rho   \eta j L_{S^{(i)}}(\bW_0) +   \ell(\bW_0;z_i')\Big)\prod_{\tilde{j}=j+1}^{t} \frac{1+p}{ 1-2\eta \epsilon_{\tilde{j}} }\\
     &\le \frac{8\rho \eta^2 (1+1/p)}{n^2} \prod_{\tilde{j}=1}^{t}  \frac{1+p}{ 1-2\eta \epsilon_{\tilde{j}} } \sum_{j=0}^t  \Big(  \rho  \eta j L_{S} (\bW_0) +    \ell(\bW_0;z_i) +   \rho \eta j L_{S^{(i)}}(\bW_0) +    \ell(\bW_0;z_i')\Big)\\
     &=  \frac{4 \rho \eta^2 (1+1/p)}{n^2} \prod_{\tilde{j}=1}^{t}  \frac{1+p}{ 1-2\eta \epsilon_{\tilde{j}} }   \Big(  \rho  \eta t(t+1) \big( L_{S} (\bW_0) +     L_{S^{(i)}}(\bW_0) \big) +    2(t+1)\big(\ell(\bW_0;z_i)  \\&\quad  +    \ell(\bW_0;z_i')\big)\Big)\\
     &\le  \frac{8\rho \eta^2 c_0 (1+1/p)(1+t) \big(   \rho  \eta t   +    2  \big)}{n^2} \prod_{\tilde{j}=1}^{t}  \frac{1+p}{ 1-2\eta \epsilon_{\tilde{j}} }  ,
\end{align*}
where we used $\ell(\bW_0;z)\le c_0$ for any $z\in\Z$. 
If we further choose $p=1/t$, then there holds
\begin{align}\label{eq:stab-6}
     \big\| \bW_{t+1} -\bW_{t+1}^{(i)} \big\|_2^2
     &\le  \frac{8 \rho \eta^2c_0 e(1+t)^2 (  \rho  \eta t  +    2 )   }{n^2}    \prod_{\tilde{j}=1}^{t}  \frac{1 }{ 1-2\eta \epsilon_{\tilde{j}} } ,
\end{align}
where we used $(1+1/t)^t \le e$. 

Now, we prove by induction to show
\begin{equation}\label{eq:stab-4}
    \big\| \bW_{t+1} -\bW_{t+1}^{(i)} \big\|_2 
     \le  \frac{2  \eta   e T \sqrt{ 2c_0 \rho (\rho \eta T+2) }     }{n }     . 
\end{equation}
\eqref{eq:stab-4} with $k=0$ holds trivially. Assume \eqref{eq:stab-4} holds with all $k\le t$, i.e., for all $k\le t$ 
\begin{equation}\label{eq:stab-5}
      \big\| \bW_{k} -\bW_{k}^{(i)} \big\|_2 
     \le  \frac{2  \eta   e T \sqrt{  2c_0 \rho (\rho \eta T+2) }     }{n }     .  
\end{equation}  
and we want to show it holds with $k =t+1 \le T$. 
Recall that  $\epsilon_k = \frac{ c_\bx B_{\sigma''} }{m^c}\big( \frac{ c_\bx B_{\sigma'} }{m^{c-1/2}}(1+\eta \rho)\big\| \bW_k - \bW_k^{(i)} \big\|_2+ \frac{c_\bx B_{\sigma'}\sqrt{2\eta Tc_0}  }{m^{c-1/2}} +  \sqrt{2c_0} \big)$. From  \eqref{eq:stab-5}, for any $\tilde{j}\le t$, we know
\[ \epsilon_{\tilde{j}} \le \epsilon':= \frac{ c_\bx B_{\sigma''} }{m^c}\Big(   \frac{2   \sqrt{ 2c_0 \rho (\rho \eta T+2) }  \eta   e T (1+\eta \rho)   c_\bx B_{\sigma'} }{n m^{c-1/2}} + \frac{c_\bx B_{\sigma'}\sqrt{2\eta Tc_0}  }{m^{c-1/2}}+  \sqrt{2c_0} \Big) . \]
Putting the above inequality back into \eqref{eq:stab-6}, we get
\begin{align*}
     \big\| \bW_{t+1} -\bW_{t+1}^{(i)} \big\|_2^2
     &\le  \frac{8 \rho \eta^2c_0 e(1+t)^2 (  \rho  \eta t  +    2 )   }{n^2}     \Big( \frac{1 }{ 1-2\eta \epsilon' } \Big)^t. 
\end{align*}
If $m$ is large enough such that $2\eta\epsilon' \le 1/(t+1)$, then 
we can show 
\begin{equation}\label{eq:stab-8}
     \Big( \frac{1 }{ 1-2\eta \epsilon' } \Big)^t \le  \Big( \frac{1 }{ 1-1/(t+1) } \Big)^t \le e .
\end{equation} 
Then there holds
\begin{align}\label{eq:stab-7}
     \big\| \bW_{t+1} -\bW_{t+1}^{(i)} \big\|_2 
     &\le  \frac{2 \eta e (1+t) \sqrt{2 \rho  c_0  (  \rho  \eta t  +    2 ) } }{n }  \le    \frac{2 \eta e T \sqrt{2 \rho  c_0  (  \rho  \eta T  +    2 ) } }{n } .
\end{align}
Now, we discuss the conditions on $m$. Suppose $m$ satisfies the following conditions 
\[ m  \ge  C_1 \Big(\frac{ (\eta T)^2 (1+\eta \rho) \sqrt{  \rho (\rho \eta T + 2)} }{n} \Big)^{\frac{2}{4c-1}}, m\ge C_2(\eta T)^{\frac{3}{4c-1}}  \text{ and } m\ge  C_3 \Big( \eta T  \Big)^{\frac{1}{c}} , \]
where $C_1= (8e c_\bx^2 B_{\sigma'}B_{\sigma''}\sqrt{2c_0})^{\frac{2}{4c-1}}$, $C_2=\big( 4\sqrt{2c_0}c_{\bx}^2 B_{\sigma'}B_{\sigma''}\big)^{\frac{3}{4c-1}}$ and $ C_3= (8\sqrt{2c_0}c_\bx B_{\sigma''})^{1/c} $. 
Then it is easy to verify that 
\[ \frac{2\eta c_\bx B_{\sigma''} }{m^c}\Big(   \frac{2   \sqrt{ 2c_0 \rho (\rho \eta T+2) }  \eta   e T (1+\eta \rho)   c_\bx B_{\sigma'} }{n m^{c-1/2}}+ \frac{c_\bx B_{\sigma'}\sqrt{2\eta Tc_0}  }{m^{c-1/2}} +  \sqrt{2c_0} \Big) \le   \frac{1}{T} \le \frac{1}{1+t} , \]
which ensures that $2\eta\epsilon' \le 1/(t+1)$, and then \eqref{eq:stab-7} holds. The proof is completed. 
\end{proof}

We can combine Theorem~\ref{thm:stability} and Lemma~\ref{lem:connection} together to get the upper bound of the generalization error.
\begin{proof}[Proof of Theorem~\ref{thm:generalization}]
Eq.\eqref{eq:stab-3} with $p=1/t$ and  Eq.\eqref{eq:stab-8} implies 
\begin{align*}
     \big\| \bW_{t+1} -\bW_{t+1}^{(i)} \big\|_2^2 &\le    \frac{2 e^2 \eta^2 (1+t) }{n^2} \sum_{j=0}^t  \Big(\big\| \nabla \ell(\bW_j;z_i)\big\|_2^2+ \big\| \nabla \ell(\bW_j^{(i)};z_i')   \big\|_2^2\Big) \\
     &\le  \frac{4  e^2 \eta^2 \rho (1+t) }{n^2} \sum_{j=0}^t  \Big(  \ell(\bW_j;z_i) +   \ell(\bW_j^{(i)};z_i')     \Big),
\end{align*}
where in the last inequality we used self-bounding property of the smooth loss (Lemma~\ref{lem:self-bounding}).  Now, taking an average over $i\in[n]$ and using $\E\big[\ell(\bW_j;z_i)\big]=\E\big[\ell(\bW_j^{(i)};z_i') \big]$, we have
\begin{align*}
    \frac{1}{n}\sum_{i=1}^n \E \big\| \bW_{t+1} -\bW_{t+1}^{(i)} \big\|_2^2 &\le  \frac{4  e^2 \eta^2 \rho(1+t) }{n^3} \sum_{j=0}^t  \Big( \sum_{i=1}^n \E\big[ \ell(\bW_j;z_i)\big] +   \E[\ell(\bW_j^{(i)};z_i')  \big] \Big)\\
    &=  \frac{8  e^2 \eta^2 \rho (1+t) }{n^3} \sum_{j=0}^t   \sum_{i=1}^n \E\big[ \ell(\bW_j;z_i)\big]   \\
    &=  \frac{8  e^2 \eta^2 \rho (1+t) }{n^2} \sum_{j=0}^t     \E\big[ L_S(\bW_j )\big].  
\end{align*}
Combining the above stability bounds with  Lemma~\ref{lem:connection} together, we get 
\begin{align*}
     \E[ L(\bW_t)\! - \! L_S(\bW_t) ]&\le  \frac{4  e^2 \eta^2 \rho^2 t}{ n^2} \sum_{j=0}^{t-1}     \E\big[ L_S(\bW_j )\big] \!+\! \Big(  \frac{16 e^2 \eta^2 \rho^2 t  \E[L_S(\bW_t) ]  }{n^2} \sum_{j=0}^{t-1}     \E\big[ L_S(\bW_j )\big] \Big)^{\frac{1}{2}}\\
     &\le  \frac{4  e^2 \eta^2 \rho^2 t }{ n^2} \sum_{j=0}^{t-1}     \E\big[ L_S(\bW_j )\big] +   \frac{4 e  \eta  \rho       }{n } \sum_{j=0}^{t-1}     \E\big[ L_S(\bW_j )\big],  
\end{align*}
where in the last inequality we used $ L_S(\bW_t)  \le \frac{1}{t}\sum_{j=1}^{t-1} L_S( \bW_j ) $ \cite{richards2021stability}. 
The proof is completed.
\end{proof}
\subsection{Proofs of Optimization  Bounds}\label{appendix:snn-optimization}
Before giving the proofs of optimization error bound, we first introduce the following lemma on the bound of GD iterates. 
\begin{lemma}\label{lem:distance}
Suppose Assumptions~\ref{ass:activation} and \ref{ass:loss} hold, and $\eta \le 1/(2\rho)$. Assume  \eqref{m-1} and \eqref{m-5}  hold.  
 Then for any $t\in[T]$, there holds
\begin{align*}
  & 1\vee \E[ \|\bW_{\frac{1}{\eta T}}^{*} - \bW_t  \|_2^2 ] \le \Big( \frac{8  e^2 \eta^3 \rho^2 t^2 }{ n^2} + \frac{8 e  \eta^2 t  \rho       }{n } \!\Big)\! \sum_{s=0}^{t-1}\!\E \big[  L_S(\bW_s) \big]  +  2{\E\big[  \|\bW_{\frac{1}{\eta T}}^{*} -\bW_0  \|_2^2  \big]} + 2\eta T \big[  L(\bW_{\frac{1}{\eta T}}^{*})-L(\bW^{*})\big] .
\end{align*}
\end{lemma}
\begin{proof}
For any $\bW, \widetilde{\bW} \in \R^{md}$ and $\alpha\in[0,1]$, define $\bW(\alpha):= \widetilde{\bW} + \alpha(\bW-\widetilde{\bW})$. 
Note that
\begin{align*}
    &\lambda_{\min}\big( \nabla^2 L_S(\bW(\alpha))\big) \\
    &\ge -\max_i\{ \| \nabla^2 f(\bx_i) \|_2\} \Big( \frac{1}{n}\sum_{i=1}^n| f_{\bW(\alpha) }(\bx_i)- y_i  | \Big)\\
    &\ge -\frac{ c_x^2 B_{\sigma''}   }{m^c} \frac{1}{n} \Big(\sum_{i=1}^n \big( |f_{\bW(\alpha) }(\bx_i) - f_{\widetilde{\bW }}(\bx_i) |+ | f_{\widetilde{\bW }}(\bx_i) - f_{\bW_0}(\bx_i) |+ | f_{\bW_0}(\bx_i) - y_i |\big) \Big)\\
    &\ge  -\frac{ c_x^2 B_{\sigma''}   }{m^c}  \Big( \frac{B_{\sigma'} c_\bx}{m^{c-1/2}} \|\bW(\alpha) -\widetilde{\bW }  \|_2+ \frac{B_{\sigma'} c_\bx}{m^{c-1/2}}\|  \widetilde{\bW }   -  \bW_0   \|_2 + \sqrt{2L_S(\bW_0)} \Big)\\
    &\ge  -\frac{ c_x^2 B_{\sigma''}   }{m^c}  \Big( \frac{B_{\sigma'} c_\bx}{m^{c-1/2}}  \|\bW -\widetilde{\bW }  \|_2+ \frac{B_{\sigma'} c_\bx}{m^{c-1/2}}\|  \widetilde{\bW }   -  \bW_0   \|_2 + \sqrt{2c_0} \Big). 
\end{align*}
Then for any $t\in [T]$, let $\widetilde{\bW}=\bW_t$, and define
\begin{align*}
g(\alpha):= & L_S( \bW(\alpha) ) + \frac{c_\bx^2 B_{\sigma''}}{m^c}\frac{\alpha^2}{2} \Big( \frac{B_{\sigma'} c_\bx}{m^{c-1/2}}  \|\bW  -\bW_t  \|_2 + \frac{B_{\sigma'} c_\bx}{m^{c-1/2}}\|  \bW_t   -  \bW_0   \|_2+ \sqrt{2c_0} \Big)\\
&\times (1 \vee \E[\|\bW  -\bW_t  \|_2^2]). 
\end{align*}
It is obvious that
$ g''(\alpha) \ge 0$. Then $g(\alpha)$ is convex in $\alpha\in[0,1]$. Now, by convexity we know 
\begin{align*}
    g(1)-g(0)&=L_S(\bW) + \frac{c_\bx^2 B_{\sigma''}}{2m^c}  \Big( \frac{B_{\sigma'} c_\bx}{m^{c-1/2}}  \|\bW -\bW_t  \|_2  +  \frac{B_{\sigma'} c_\bx}{m^{c-1/2}}\|  \bW_t   -  \bW_0   \|_2+ \sqrt{2c_0} \Big)\\
    &\quad \times (1 \vee \E[\|\bW  -\bW_t  \|_2^2]) - L_S(\bW_t)\\
    &\ge \langle \bW-\bW_t , \nabla L_S(\bW_t)  \rangle= g'(0). 
\end{align*}
Rearranging the above inequality we get 
\begin{align}\label{eq:opt-7}
    L_S(\bW_t)& \le L_S(\bW) + \frac{c_\bx^2 B_{\sigma''}}{2m^c}  \Big(\frac{B_{\sigma'} c_\bx}{m^{c-1/2}}  \|\bW -\bW_t  \|_2  +  \frac{B_{\sigma'} c_\bx}{m^{c-1/2}}\|  \bW_t   -  \bW_0   \|_2+ \sqrt{2c_0} \Big)\nonumber\\
    &\quad \times (1 \vee \E[\|\bW  -\bW_t  \|_2^2])- \langle \bW-\bW_t , \nabla L_S(\bW_t)\rangle.
\end{align}
Combining \eqref{eq:opt-7} with the smoothness of the loss we can get 
\begin{align*}
    L_S(\bW_{t+1}) &\le L_S(\bW_t) + \langle \nabla L_S(\bW_t), \bW_{t+1}- \bW_t \rangle + \frac{\rho}{2}\| \bW_{t+1}- \bW_t  \|_2^2\\
    &\le L_S(\bW_t) - \eta  \langle \nabla L_S(\bW_t), \nabla L_S(\bW_t) \rangle  + \frac{\rho}{2}\| \bW_{t+1}- \bW_t  \|_2^2  \\
    &\le L_S(\bW_t) - \eta( 1-\frac{\eta \rho}{2} )\|  \nabla L_S(\bW_t)  \|_2^2\\
    &\le  L_S(\bW_t) - \frac{\eta}{2} \|  \nabla L_S(\bW_t)  \|_2^2\\
    &\le  L_S(\bW) + \frac{c_\bx^2 B_{\sigma''}}{2m^c}  \Big( \frac{B_{\sigma'} c_\bx}{m^{c-1/2}}  \|\bW -\bW_t  \|_2  +  \frac{B_{\sigma'} c_\bx}{m^{c-1/2}}\|  \bW_t   -  \bW_0   \|_2+ \sqrt{2c_0}  \Big)\\&\quad \times (1 \vee \E[\|\bW  -\bW_t  \|_2^2]) - \langle \bW-\bW_t , \nabla L_S(\bW_t) \rangle  - \frac{\eta}{2} \|  \nabla L_S(\bW_t)  \|_2^2,
\end{align*}
where in the third inequality we used the update rule \eqref{eq:GD} and $\eta \rho \le 1$. 

According to the equality $2\langle x-y,x-z \rangle=\|x-y\|_2^2+\|x-z\|_2^2-\|y-z\|_2^2$, we know
\begin{align*}
     \! - \langle \bW-\bW_t , \nabla L_S(\bW_t) \rangle - \frac{\eta}{2} \|  \nabla L_S(\bW_t)  \|_2^2&= \frac{1}{\eta}\langle \bW-\bW_t, \bW_{t+1} -\bW_t \rangle - \frac{1}{2\eta} \|\bW_{t+1} -\bW_t\|_2^2 \\& = \frac{1}{2\eta} \big( \|\bW -\bW_t  \|_2^2 \!-\! \|\bW_{t+1} -\bW   \|_2^2 \big).
\end{align*}  
Then there holds 
\begin{align}\label{eq:opt-2}
    L_S(\bW_{t+1}) 
     & \le    L_S(\bW) + \frac{c_\bx^2 B_{\sigma''}}{2m^c}  \Big( \frac{B_{\sigma'} c_\bx}{m^{c-1/2}} \|\bW -\bW_t  \|_2 +  \frac{B_{\sigma'} c_\bx \sqrt{2\eta T c_0}}{m^{c-1/2}}+ \sqrt{2c_0} \Big)\nonumber\\
     &\quad \times (1 \vee \E[\|\bW  -\bW_t  \|_2^2]) + \frac{1}{2\eta} \Big( \|\bW -\bW_t  \|_2^2 - \|\bW_{t+1} -\bW  \|_2^2 \Big). 
\end{align}
The above inequality with  $\bW=\bW^{*}_{\frac{1}{\eta T}}$  implies
\begin{align*} 
       &\frac{1}{t} \sum_{s=0}^{t-1}  \E[  L_S(\bW_s)]   +    \frac{\E\big[   \|\bW_{t}  - \bW_{\frac{1}{\eta T}}^{*}   \|_2^2\big] }{2\eta t}  \\
       &\!\le    \!\E[L(\bW_{\frac{1}{\eta T}}^{*} )]  \!+\!  \frac{\E\big[  \|\bW_{\frac{1}{\eta T}}^{*} \!\!\! -\!\! \bW_0  \|_2^2  \big]}{2\eta t} \! +\!  \frac{c_\bx^2 B_{\sigma''}}{2m^c t }  \!\! \sum_{s=0}^{t-1}\!  \Big(\!\frac{B_{\sigma'} c_\bx}{m^{c-1/2}}  \E[\|\bW_{\frac{1}{\eta T}}^{*}\!\!\! -\!\! \bW_s  \|_2  ]  \!+\!   \frac{B_{\sigma'} c_\bx \sqrt{2\eta T c_0}}{m^{c-1/2}} \!+\!  \sqrt{2c_0}  \Big)   (1\!\vee \!\E[\|\bW_{\frac{1}{\eta T}}^{*} \!\!\! -\!\!\bW_s  \|_2^2]).
\end{align*}
Combined the above inequality with  Theorem~\ref{thm:generalization} implies 
\begin{align}\label{eq:opt-4}
      &\frac{\E\big[   \|\bW_{t} -\bW_{\frac{1}{\eta T}}^{*}   \|_2^2 \big]}{2\eta t} \nonumber\\
      &\le   \frac{1}{t}\sum_{s=0}^{t-1} \big[  L(\bW_{\frac{1}{\eta T}}^{*})\!-\!\E[  L(\bW_s)] \big] \! + \!\frac{\E\big[  \|\bW_{\frac{1}{\eta T}}^{*} \!-\!\bW_0  \|_2^2  \big]}{2\eta t} \! + \! \Big( \frac{4 e^2 \eta^2 \rho^2 t }{ n^2}\! + \!   \frac{4 e  \eta  \rho       }{n } \Big) \sum_{s=0}^{t-1}     \E\big[ L_S(\bW_s)\big]\nonumber\\
      &\quad  +\! \frac{c_\bx^2 B_{\sigma''}}{2m^c t }  \!\sum_{s=0}^{t-1} \! \Big(\frac{B_{\sigma'} c_\bx}{m^{c-1/2}}  \E[\|\bW_{\frac{1}{\eta T}}^{*}\!-\!\bW_s  \|_2  ] \!+\!   \frac{B_{\sigma'} c_\bx \sqrt{2\eta T c_0}}{m^{c-1/2}}\!+\! \sqrt{2c_0} \Big)(1\vee \E[\|\bW_{\frac{1}{\eta T}}^{*} \! -\!\bW_s  \|_2^2])\nonumber\\
      &\le  L(\bW_{\frac{1}{\eta T}}^{*})\!-\!L(\bW^{*}) \!+\!  \frac{\E\big[  \|\bW_{\frac{1}{\eta T}}^{*} \!-\!\bW_0  \|_2^2  \big]}{2\eta t}  \!+\!  \Big( \frac{4 e^2 \eta^2 \rho^2 t }{ n^2} \!+\!    \frac{4 e  \eta  \rho       }{n } \Big) \sum_{s=0}^{t-1}     \E\big[ L_S(\bW_s)\big]\nonumber\\
      &\quad  + \!\frac{c_\bx^2 B_{\sigma''}}{2m^c t }  \!\sum_{s=0}^t \! \Big(\frac{B_{\sigma'} c_\bx}{m^{c-1/2}}  \E[\|\bW_{\frac{1}{\eta T}}^{*}\!-\!\bW_s  \|_2 ] \!+\!   \frac{B_{\sigma'} c_\bx \sqrt{2\eta T c_0}}{m^{c-1/2}}\!+\! \sqrt{2c_0}   \Big)(1\vee \E[\|\bW_{\frac{1}{\eta T}}^{*} \! -\!\bW_s  \|_2^2]), 
\end{align}
where in the second inequality we used $  L(\bW_{\frac{1}{\eta T}}^{*})- L(\bW_s)\le L(\bW_{\frac{1}{\eta T}}^{*}) - L(\bW^{*}) $ since $L(\bW_s)\ge L(\bW^{*}) $ for any $s\in[t-1]$. 
 

On the other hand, using Lemma~\ref{lem:deviation}  we can obtain 
\begin{align}\label{eq:opt-5}
     \|  \bW_{\frac{1}{\eta T}}^{*} -\bW_s  \|_2 \le  \|\bW_{\frac{1}{\eta T}}^{*} -\bW_0  \|_2 + \|\bW_s -\bW_0   \|_2 \le  \sqrt{2\eta T c_0} + \|\bW_{\frac{1}{\eta T}}^{*} -\bW_0  \|_2.
\end{align}
Then we know
\begin{align*}
&\Big(\frac{B_{\sigma'} c_\bx}{m^{c-1/2}}  \E[\|\bW_{\frac{1}{\eta T}}^{*}\!-\!\bW_s  \|_2 ] \!+\!  \frac{B_{\sigma'} c_\bx \sqrt{2\eta T c_0}}{m^{c-1/2}}\!+\! \sqrt{2c_0}   \Big)(1\vee \E[\|\bW_{\frac{1}{\eta T}}^{*}  -\bW_s  \|_2^2])\\
    &\le \Big(\frac{ B_{\sigma'} c_\bx}{m^{c-1/2}}(2\sqrt{2\eta T c_0} + \|\bW_{\frac{1}{\eta T}}^{*} -\bW_0  \|_2)  +  \sqrt{2c_0}  \Big)(1\vee\E[\|\bW_{\frac{1}{\eta T}}^{*}  -\bW_s  \|_2^2]).
\end{align*} 
Plugging the above inequality    back into \eqref{eq:opt-4} yields
\begin{align*}
     & \frac{\E\big[   \|\bW_{t} -\bW_{\frac{1}{\eta T}}^{*}   \|_2^2\big] }{2\eta t} \\ 
         &\le \frac{\E\big[  \|\bW_{\frac{1}{\eta T}}^{*} \!\!-\!\bW_0  \|_2^2  \big]}{2\eta t}\! +\! \frac{ \tilde{b}( \sqrt{2 \eta T c_0}+\|\bW_{\frac{1}{\eta T}}^{*}\!\! -\!\bW_0  \|_2 )}{m^c t}  \sum_{s=0}^t   (1\vee\E[\|\bW_{\frac{1}{\eta T}}^{*} \!\! -\!\bW_s  \|_2^2])  \! +\!  \Big( \frac{4  e^2 \eta^2 \rho^2 t }{ n^2}\! +
         \!\frac{4 e  \eta  \rho       }{n } \Big) \!\sum_{s=0}^{t-1}\!     \E\big[ L_S(\bW_s)\big]   \\
      &\quad  +  L(\bW_{\frac{1}{\eta T}}^{*})-L(\bW^{*}) , 
\end{align*}
where $\tilde{b}= \frac{c_\bx^2 B_{\sigma''}}{2}\big( \frac{ 2B_{\sigma'} c_\bx}{m^{c-1/2}}   + \sqrt{2c_0} \big)$.

Multiplying both sides by $2\eta t$ yields
\begin{align*}
      & \E\big[   \|\bW_{t} -\bW_{\frac{1}{\eta T}}^{*}   \|_2^2\big]    \\
      &\le  {\E\big[  \|\bW_{\frac{1}{\eta T}}^{*} -\bW_0  \|_2^2  \big]}  +  \frac{2  \tilde{b}\eta  ( \sqrt{2 \eta T c_0}+\|\bW_{\frac{1}{\eta T}}^{*} -\bW_0  \|_2 )}{m^c  }  \sum_{s=0}^t     (1\vee\E[\|\bW_{\frac{1}{\eta T}}^{*}  -\bW_s  \|_2^2])  \\
      &\quad  +  \Big( \frac{8 e^2 \eta^3 \rho^2 t^2 }{ n^2} +    \frac{8e  \eta^2  \rho  t     }{n } \Big) \sum_{s=0}^{t-1}     \E\big[ L_S(\bW_s)\big]  +2\eta T \big[  L(\bW_{\frac{1}{\eta T}}^{*})-L(\bW^{*})\big]   .
\end{align*}
Let $x=\max_{s\in[T]}\E[ \|\bW_{\frac{1}{\eta T}}^{*} -\bW_s  \|_2^2 ]\vee 1 $. Then the above inequality implies
\begin{align*}
      x  
    &  \le   {\E\big[  \|\bW_{\frac{1}{\eta T}}^{*} -\bW_0  \|_2^2  \big]}  +  \frac{2  \tilde{b} \eta T ( \sqrt{2 \eta T c_0}+\|\bW_{\frac{1}{\eta T}}^{*} -\bW_0  \|_2 )}{m^c  } x   \\&\quad   +  \Big( \frac{8 e^2 \eta^3 \rho^2 t^2 }{ n^2} +    \frac{8 e  \eta^2  \rho  t     }{n } \Big) \sum_{s=0}^{t-1}     \E\big[ L_S(\bW_s)\big]   +2\eta T \big[  L(\bW_{\frac{1}{\eta T}}^{*})-L(\bW^{*})\big]  .
\end{align*}
Without loss of generality, we assume $\eta \le 1$. Condition~\eqref{m-5} implies $m\ge \big( 4\tilde{b} \eta T ( \sqrt{2 \eta T c_0}+\|\bW_{\frac{1}{\eta T}}^{*} -\bW_0  \|_2 ) \big)^{\frac{1}{c}}$, then there holds $\frac{2  \tilde{b} \eta T ( \sqrt{2 \eta T c_0}+\|\bW_{\frac{1}{\eta T}}^{*} -\bW_0  \|_2 )}{m^c  }\le \frac{1}{2}$. Hence
\[ x \le  \Big( \frac{16  e^2 \eta^3 \rho^2 t^2 }{ n^2} +    \frac{16 e  \eta^2 t  \rho       }{n } \Big) \sum_{s=0}^{t-1}     \E\big[ L_S(\bW_s)\big] +  2{\E\big[  \|\bW_{\frac{1}{\eta T}}^{*} -\bW_0  \|_2^2  \big]}  +2\eta T \big[  L(\bW_{\frac{1}{\eta T}}^{*})-L(\bW^{*})\big] .  \]
It then follows that
\begin{align*}
    &1\vee \E[ \|\bW_{\frac{1}{\eta T}}^{*} -\bW_t  \|_2^2 ] \le \Big( \frac{16  e^2 \eta^3 \rho^2 t^2 }{ n^2} \!+\!    \frac{16 e  \eta^2 t  \rho       }{n } \Big) \sum_{s=0}^{t-1}     \E\big[ L_S(\bW_s)\big] \!+\!  2{\E\big[  \|\bW_{\frac{1}{\eta T}}^{*} \!-\!\bW_0  \|_2^2  \big]}\!  +\!2\eta T \big[  L(\bW_{\frac{1}{\eta T}}^{*})\!-\!L(\bW^{*})\big] .
\end{align*}
This completes the proof. 
\end{proof}

Now, we can give the proof of Theorem~\ref{thm:opt}.
\begin{proof}[Proof of Theorem~\ref{thm:opt}]
Recall that $\tilde{b}= \frac{  c_\bx^2 B_{\sigma''}}{2}\big( \frac{2B_{\sigma'} c_\bx}{m^{c-1/2}}   + \sqrt{2c_0} \big)$. 
Eq.\eqref{eq:opt-2} with $\bW=\bW^{*}_{\frac{1}{\eta T}}$ implies
\begin{align} \label{eq:opt-3}
     &\frac{1}{T}\sum_{s=0}^{T-1 }\E[L_S(\bW_{s}) ] \nonumber\\
      &\le   \E[L_S(\bW^{*}_{\frac{1}{\eta T}})] \!+\!  \frac{ \tilde{b}  ( \sqrt{2 \eta T c_0}\!+\!\|\bW_{\frac{1}{\eta T}}^{*} -\bW_0  \|_2 ) }{m^c T} \sum_{s=0}^{T-1 }  1\vee\E[\|\bW^{*}_{\frac{1}{\eta T}} -\bW_s  \|_2^2]  \!+\! \frac{\|\bW^{*}_{\frac{1}{\eta T}}\! -\!\bW_0   \|_2^2}{2\eta T}, 
\end{align}
where in the last inequality we used \eqref{eq:opt-5}.

Further, by monotonically decreasing of $\{ L_S(\bW_t)\}$, we know
\begin{align*} 
    &\E[ L_S(\bW_{T})] \le    \E[L_S(\bW^{*}_{\frac{1}{\eta T}})] \!+\!  \frac{  \tilde{b}  ( \sqrt{2 \eta T c_0}\!+\!\|\bW_{\frac{1}{\eta T}}^{*}\! -\!\bW_0   \|_2 )  }{m^c T} \sum_{s=0}^{T-1 }    1\vee\E[\|\bW^{*}_{\frac{1}{\eta T}} \!-\!\bW_s  \|_2^2]  + \frac{\|\bW^{*}_{\frac{1}{\eta T}} \!-\!\bW_0    \|_2^2}{2\eta T}.  
\end{align*}
Note that Lemma~\ref{lem:distance} shows 
\begin{align*}
    &1\vee\E[ \|\bW_{\frac{1}{\eta T}}^{*} -\bW_t  \|_2^2 ]  \le\! \Big( \frac{16  e^2 \eta^3 \rho^2 t^2 }{ n^2} \!+\!    \frac{16 e  \eta^2 t  \rho       }{n } \Big) \!\sum_{s=0}^{t-1}   \!  \E\big[ L_S(\bW_s)\big]\!+\!  2{\E\big[  \|\bW_{\frac{1}{\eta T}}^{*}\!-\!\bW_0    \|_2^2  \big]} \!+\!2\eta T \big[  L(\bW_{\frac{1}{\eta T}}^{*})\!-\!L(\bW^{*})\big].
\end{align*}
Combining the above two inequalities together, we get
\begin{align*}
      \E[L_S(\bW_{T})]  
     &\le \E[L_S(\bW^{*}_{\frac{1}{\eta T}})] + \frac{   2 \tilde{b}  ( \sqrt{2 \eta T c_0}+\|\bW_{\frac{1}{\eta T}}^{*} \! -\!\bW_0   \|_2 ) }{m^c }\Big( \Big( \frac{8  e^2 \eta^3 \rho^2 T^2 }{ n^2} +    \frac{8 e  \eta^2 T  \rho       }{n } \Big) \sum_{s=0}^{T-1}     \E\big[ L_S(\bW_s)\big] \\
     & \quad +   {  \|\bW_{\frac{1}{\eta T}}^{*} -\bW_0   \|_2^2   }  + \eta T \big[  L(\bW_{\frac{1}{\eta T}}^{*})-L(\bW^{*})\big]  \Big) +\frac{   \|\bW_{\frac{1}{\eta T}}^{*}  -\bW_0   \|_2^2   } {2\eta T}. 
\end{align*}
The theorem is proved.
 \end{proof}
 
 \begin{lemma}\label{lem:sum-risk}
 Suppose Assumptions~\ref{ass:activation} and  \ref{ass:loss}   hold.  
Let $\{\bW_t\}$ be produced by \eqref{eq:GD} with  $\eta\le 1/(2\rho)$. Assume \eqref{m-1} and \eqref{m-5} hold. Then
\begin{align*}
 &\sum_{s=0}^{T-1 }\E[L_S(\bW_{s})]  \le  4 T  L(\bW^{*}_{\frac{1}{\eta T}}) - 2\eta T L(\bW^{*})+ \Big( \frac{ 2\tilde{b} T      ( \sqrt{2 \eta T c_0}+\|\bW_{\frac{1}{\eta T}}^{*} -\bW_0   \|_2 ) }{m^c  } +\frac{1}{2\eta}\Big)      \|\bW_{\frac{1}{\eta T}}^{*}    -\bW_0  \|_2^2 . 
\end{align*}
 \end{lemma}
 \begin{proof}
 Multiplying $T$ over both sides of \eqref{eq:opt-3} and using Lemma~\ref{lem:distance} we get
\begin{align*}
  \sum_{s=0}^{T-1 }\E[L_S(\bW_{s})] & \le  T  L(\bW^{*}_{\frac{1}{\eta T}}) +  \frac{ \tilde{b}    ( \sqrt{2 \eta T c_0}+\|\bW_{\frac{1}{\eta T}}^{*} -\bW_0     \|_2 )}{m^c  } \sum_{s=0}^{T-1 } 1\vee    \E[ \|\bW^{*}_{\frac{1}{\eta T }} -\bW_s  \|_2^2]  + \frac{\|\bW^{*}_{\frac{1}{\eta T}} -\bW_0     \|_2^2}{2\eta  }\\
     &\le  T  L(\bW^{*}_{\frac{1}{\eta T}}) \!+\!  \frac{\tilde{b} T    ( \sqrt{2 \eta T c_0}+\|\bW_{\frac{1}{\eta T}}^{*} -\bW_0     \|_2 )}{m^c  }   \Big( \big( \frac{16  e^2 \eta^3 \rho^2 T^2 }{ n^2}\! +\!    \frac{16 e  \eta^2 T  \rho       }{n } \big)  \sum_{s=0}^{T-1}     \E\big[ L_S(\bW_s)\big]\nonumber\\
     &\quad +  2{  \|\bW_{\frac{1}{\eta T}}^{*} -\bW_0     \|_2^2   }  +2\eta T \big[  L(\bW_{\frac{1}{\eta T}}^{*})-L(\bW^{*})\big]  \Big) + \frac{\|\bW^{*}_{\frac{1}{\eta T}}  -\bW_0    \|_2^2}{2\eta  }.
\end{align*}
Condition~\eqref{m-5}  implies $m\ge \big(2 \tilde{b} T   ( \sqrt{2 \eta T c_0}+\|\bW_{\frac{1}{\eta T}}^{*} -\bW_0    \|_2 ) \big( \frac{16  e^2 \eta^3 \rho^2 T^2 }{ n^2} +    \frac{16 e  \eta^2 T  \rho       }{n } \big)\big)^{1/c}$ and $m\ge \big(2 \tilde{b} T   ( \sqrt{2 \eta T c_0}+\|\bW_{\frac{1}{\eta T}}^{*} -\bW_0   \|_2 )\big)^{1/c}$, there holds
\begin{align*}
 &\sum_{s=0}^{T-1 }\E[L_S(\bW_{s})]  \le  4 T  L(\bW^{*}_{\frac{1}{\eta T}}) - 2\eta T L(\bW^{*})+ \Big( \frac{ 2\tilde{b} T      ( \sqrt{2 \eta T c_0}+\|\bW_{\frac{1}{\eta T}}^{*} -\bW_0   \|_2 ) }{m^c  } +\frac{1}{2\eta}\Big)      \|\bW_{\frac{1}{\eta T}}^{*}  -\bW_0   \|_2^2   ,
\end{align*}
which completes the proof. 
 \end{proof}

\subsection{Proofs of Excess Risk Bounds} \label{appendix:snn-excess}
\begin{proof}[Proof of Theorem~\ref{thm:excess}]  
According to Lemma~\ref{lem:sum-risk} and noting that $\tilde{b}=\O(1)$, we know
\begin{align}\label{eq:excess-1}
     \!\sum_{s=0}^{T-1} \E[ L_S(\bW_s) ]\! =\!\O\Big(  TL(\bW^{*}_{\frac{1}{\eta T}}) \! +\! \Big(\frac{  T   ( \sqrt{  \eta T  }\!+\!\|\bW_{\frac{1}{\eta T}}^{*} \!-\!\bW_0    \|_2 )   }{m^c} \!+\! \frac{1}{\eta}\Big) \|\bW^{*}_{\frac{1}{\eta T}}\!-\!\bW_0  \|_2^2  \Big).  
\end{align}
The upper bound of the generalization error  can be controlled by plugging   \eqref{eq:excess-1} back into  Theorem~\ref{thm:generalization} 
\begin{align}\label{eq:excess-6}
     &\E[ L(\bW_T) - L_S( \bW_T )] \nonumber\\
    &= \O\Big( \big( \frac{\eta^2 \rho^2 T}{n^2} +\frac{\eta \rho}{n} \big) \sum_{s=0}^{T-1} \E[ L_S(\bW_s) ] \Big)\nonumber\\
    &=\O\Big( \Big( \frac{\eta^2 \rho^2 T^2}{n^2} +\frac{\eta T \rho}{n} \Big) \Big(   L(\bW^{*}_{\frac{1}{\eta T}}) + \Big(\frac{     \sqrt{  \eta T  }+\|\bW_{\frac{1}{\eta T}}^{*} -\bW_0  \|_2  }{m^c} + \frac{1 }{\eta T}\Big)\|\bW^{*}_{\frac{1}{\eta T}}-\bW_0  \|_2^2  \Big). 
\end{align}
The estimation of the optimization error is given by  plugging  \eqref{eq:excess-1} back into  Theorem~\ref{thm:opt}
\begin{align}\label{eq:excess-2}
    &\E[ L_S(\bW_T) -L_S(\bW^{*}_{\frac{1}{\eta T}}) -\frac{1}{2\eta T}\|\bW^{*}_{\frac{1}{2\eta T}} -\bW_0  \|_2^2 ]\nonumber\\
    &\!=\O\Big(  \frac{    ( \sqrt{  \eta T  }+\|\bW_{\frac{1}{\eta T}}^{*}  - \bW_0   \|_2 )  }{m^c }\Big[ \Big( \frac{  \eta^3 \rho^2 T^2 }{ n^2}  +     \frac{  \eta^2 T  \rho       }{n } \Big) \sum_{s=0}^{T-1}     \E\big[ L_S(\bW_s)\big] +   {  \|\bW_{\frac{1}{\eta T}}^{*}  - \bW_0   \|_2^2   } \nonumber\\
    &\quad + \eta T \big[ L(\bW^{*}_{\frac{1}{\eta T}})  -  L(\bW^{*}) \big] \Big] \Big)\nonumber\\
    &\!=\O\Big( \!  \frac{   ( \sqrt{  \eta T  }\!+\!\|\bW_{\frac{1}{\eta T}}^{*}\!-\!\bW_0   \|_2 )   }{m^c } \Big( \frac{  \eta^3 \rho^2 T^2 }{ n^2} \!+\!    \frac{  \eta^2  T  \rho       }{n } \Big) \Big(\frac{  T     ( \sqrt{  \eta T  }\!+\!\|\bW_{\frac{1}{\eta T}}^{*} \!-\!\bW_0   \|_2 )   }{m^c }\!+\! \frac{1}{\eta}\Big) \|\bW^{*}_{\frac{1}{\eta T}}\!-\!\bW_0  \|_2^2 \nonumber\\
    &\qquad +  \frac{\eta T     ( \sqrt{  \eta T  }+\|\bW_{\frac{1}{\eta T}}^{*} -\bW_0   \|_2 )   }{m^c } \Big( \frac{  \eta^2 \rho^2 T^2 }{ n^2} +    \frac{  \eta  T  \rho       }{n } \Big) L(\bW^{*}_{\frac{1}{\eta T}}) \nonumber\\
    &\qquad + \frac{      ( \sqrt{  \eta T  }+\|\bW_{\frac{1}{\eta T}}^{*} -\bW_0  \|_2 )  }{m^c }\big(\|\bW^{*}_{\frac{1}{\eta T}}-\bW_0  \|_2^2+ \eta T \Lambda_{\frac{1}{\eta T}} \big)\Big), \end{align}
where we used  the fact that $L(\bW^{*}_{\frac{1}{\eta T}} )-L(\bW^{*}) \le \Lambda_{\frac{1}{\eta T}}$. 
    
Combining  \eqref{eq:excess-6} and  \eqref{eq:excess-2}  together and noting that the approximation error $\Lambda_{\frac{1}{\eta T}} = L(\bW^{*}_{\frac{1}{\eta T}} ) + \frac{1}{2\eta T} \| \bW^{*}_{\frac{1}{\eta T}}  -\bW_0  \|_2^2  -  L(\bW^{*})$  we get
\begin{align*} 
      &\E[L(\bW_T) - L(\bW^{*})]\nonumber\\ 
     &\!=\! \Big[   \E[ L(\bW_T)  - L_S(\bW_T) \Big] + \E\Big[ L_S(\bW_T) - \big( L_S(\bW^{*}_{\frac{1}{\eta T}} ) + \frac{1}{2\eta T} \| \bW^{*}_{\frac{1}{\eta T}} -\bW_0   \|_2^2 \big) \Big]\nonumber\\
     &\quad + \Big[ L(\bW^{*}_{\frac{1}{\eta T}} ) + \frac{1}{2\eta T} \| \bW^{*}_{\frac{1}{\eta T}} -\bW_0   \|_2^2  -  L(\bW^{*}) \Big]\nonumber\\
    &\!=\! \O\Big(    \frac{  \eta  T  \rho       }{n }  \Big(  \frac{ \! \eta  \rho  T  }{ n } \!  +\!  1 \Big)\Big( 1\!+\! \frac{ \!\eta T     ( \sqrt{  \eta T  }\!+\!\|\bW_{\frac{1}{\eta T}}^{*}\!\!-\!\bW_0    \|_2 )    }{m^c } \Big)\Big[ L(\bW^{*}_{\frac{1}{\eta T}}\!) \!+\! \Big(\frac{1}{2\eta T} \!+\!\frac{       \!\! \sqrt{  \eta T  }\!+\!\|\bW_{\frac{1}{\eta T}}^{*}\!\!-\!\bW_0    \|_2 }{m^c } \Big)\nonumber\\
    &\quad \times \|\bW^{*}_{\frac{1}{\eta T}}\!\!-\!\bW_0  \|_2^2 \Big]   + \frac{    ( \sqrt{  \eta T  }+\|\bW_{\frac{1}{\eta T}}^{*} -\bW_0   \|_2 )}{m^c} \big(\|\bW^{*}_{\frac{1}{\eta T}}-\bW_0  \|_2^2+ \eta T \Lambda_{\frac{1}{\eta T}} \big) +\Lambda_{\frac{1}{\eta T}} \Big) . 
\end{align*}
Recalling that $\rho=\O(m^{1-2c})$.  If $\eta T m^{1-2c}=\O(n)$ and $ \eta T( \sqrt{\eta T} + \|\bW^{*}-\bW_0  \|_2 ) =\O(m^c)$,   there holds $\eta T \rho=\O(n)$ and $ \eta T( \sqrt{\eta T} + \|\bW^{*}_{\frac{1}{\eta T}}  -\bW_0  \|_2 )/m^c =\O(1)$.  Then from the above bound we can get
\begin{align*} 
    &\E[L(\bW_T) - L(\bW^{*})]= \O\Big(    \frac{  \eta  T  \rho       }{n }  \Big[ L(\bW^{*}_{\frac{1}{\eta T}}) +\frac{1}{2\eta T} \|\bW^{*}_{\frac{1}{\eta T}}\!\!-\!\bW_0  \|_2^2   \Big]   +  \frac{1}{\eta T}  \|\bW^{*}_{\frac{1}{\eta T}}\!\!-\!\bW_0  \|_2^2 +\Lambda_{\frac{1}{\eta T}} \Big) . 
\end{align*}
Combining the above  bound  with the facts $L(\bW^{*}_{\frac{1}{\eta T}}) +\frac{1}{2\eta T} \|\bW^{*}_{\frac{1}{\eta T}}-\bW_0  \|_2^2 =L(\bW^{*})+ \Lambda_{\frac{1}{\eta T}} $ and $\|\bW^{*}_{\frac{1}{\eta T}}-\bW_0  \|_2\le \sqrt{\eta T\Lambda_{\frac{1}{\eta T}}}$ together we get
 \begin{align*} 
    \E[L(\bW_T) - L(\bW^{*})] 
   =  \O\Big(    \frac{  \eta  T  \rho       }{n }    L(\bW^{*} )    + \Lambda_{\frac{1}{\eta T}} \Big) . 
\end{align*}
The proof is completed. 
\end{proof}

\begin{proof}[Proof of Corollary~\ref{cor:excess-rate}]
\noindent\textbf{Part (a).} 
\noindent\textbf{Case 1.} 
From the definition of the approximation error $\Lambda_{\frac{1}{\eta T}}$, we know that $\Lambda_{\frac{1}{\eta T}}\le \frac{1}{2\eta T}\|\bW^{*}-\bW_0\|_2^2$. 
Combining this with  Theorem~\ref{thm:excess}, we have
 \begin{align*} 
    \E[L(\bW_T) - L(\bW^{*})] 
   =  \O\Big(    \frac{  \eta  T  \rho       }{n }    L(\bW^{*} )    + \frac{1}{\eta T}\|\bW^{*}-\bW_0\|_2^2\Big) . 
\end{align*}
Without loss of generality, we consider $\|\bW_0\|_2$ as a constant.
To obtain the excess risk rate, we discuss the following two cases: $2c+6\mu-3> 0$ and $2c+6\mu-3\le 0$. 

For the case $2c+6\mu-3> 0$,  
to ensure  conditions \eqref{m-1}, \eqref{m-5} and $ \eta T( \sqrt{\eta T} + \|\bW^{*}_{\frac{1}{\eta T}}-\bW_0\|_2 ) =\O(m^c)$ hold, we set $m\asymp(\eta T)^{\frac{3}{2c}}$ for this case. Then according to Theorem~\ref{thm:excess} and Assumption~\ref{ass:minimizer}  we know
\begin{align*}
      \E[L(\bW_T) - L(\bW^{*})] 
    &=\O\Big( (\eta T)^{ \frac{3-4c}{2c} } n^{-1}  + (\eta T)^{\frac{3-6\mu-2c}{2c}} \Big). 
\end{align*}
If $c<3/4$, under the condition $ \eta T \lesssim n^{\frac{c}{3-4c} } $ and $n^{\frac{c}{6\mu +2c -3}}\lesssim \eta T $, there holds $(\eta T)^{ \frac{3-4c}{2c} } n^{-1}  + (\eta T)^{\frac{3-6\mu-2c}{2c}}  =\O(1/\sqrt{n})$. To ensure the above-mentioned conditions hold simultaneously, we further require $c+\mu\ge 1$ such that $n^\frac{c}{3-4c}\gtrsim n^{\frac{c}{6\mu +2c -3}}$. Therefore, if $c\in[1-\mu, 3/4)$ and $c+3\mu>3/2$, we can obtain \begin{align*}
      \E[L(\bW_T) - L(\bW^{*})] 
    &=\O\Big( \frac{1}{\sqrt{n}} \Big)
\end{align*}   with $m\asymp  
(\eta T)^{\frac{3}{2c}}$ and $n^{\frac{c}{6\mu+2c-3}} \lesssim \eta T   \lesssim n^{\frac{c}{3-4c} } $. 

If $c\ge 3/4$, for any $\eta T\ge 1$ and $n\ge 1$, there holds $(\eta T)^{ \frac{3-4c}{2c} } n^{-1} =\O(1/\sqrt{n})$. Similar to before, if $\eta T\gtrsim  n^{\frac{c}{6\mu+2c-3}}$,   there holds $(\eta T)^{\frac{3-6\mu-2c}{2c}}=\O(1/\sqrt{n})$. 
Then we can obtain the excess population bound $\O(1/\sqrt{n})$ with $m\asymp (\eta T)^{\frac{3}{2c}}$ and $\eta T\gtrsim n^{\frac{c}{6\mu+2c-3}}$ . 

\noindent\textbf{Case 2.} 
For the case $2c+6\mu-3\le 0$, we can choose $m\asymp (\eta T)^{\frac{1}{c+\mu-1/2}}$ to ensure conditions \eqref{m-1},  \eqref{m-5} and $ \eta T( \sqrt{\eta T} + \|\bW^{*}_{\frac{1}{\eta T}}-\bW_0\|_2 ) =\O(m^c)$ hold. 
From Theorem~\ref{thm:excess}  we know
\begin{align*}
      \E[L(\bW_T) - L(\bW^{*})] 
    &=\O\big( (\eta T)^{ \frac{1-2c+2\mu}{2c+2\mu-1} } n^{-1}  + (\eta T)^{\frac{3-6\mu-2c}{2c+2\mu-1}} \big). 
\end{align*}
Note $3-6\mu-2c\ge 0$ and $2c+2\mu-1>0$. Then the term $(\eta T)^{\frac{3-6\mu-2c}{2c+2\mu-1}} $ will not converge for any choice of $\eta T$ in this case. 
The proof of Part (a) is completed.

\textbf{Part (b).} Now, we consider the low noise case, i.e.,  $L(\bW^{*})=0$.  Combining the fact $\Lambda_{\frac{1}{\eta T}}\le \frac{1}{2\eta T}\|\bW^{*}-\bW_0\|_2^2$, Assumption~\ref{ass:minimizer}  and Theorem~\ref{thm:excess}   with  $L(\bW^{*})=0$, we can get 
\begin{align*}
       \E[L(\bW_T) - L(\bW^{*})]
      &=\O\Big(  \frac{m^{1-2\mu}}{\eta T} \Big). 
\end{align*}
Similar to part (a), we can set $m\asymp (\eta T)^{\frac{3}{2c}}$, $ \eta T \gtrsim n^{\frac{2c}{6\mu+2c-3}}$  and obtain
\[
\E[L(\bW_T) - L(\bW^{*})] 
     = \O\big(  \frac{1}{n} \big).
\]
We can check that this choice of $m$ and $\eta T$ satisfies  conditions \eqref{m-1} and  \eqref{m-5}. 
The proof  is completed. \end{proof}
\begin{remark}
Several works \cite{charles2018stability,hardt2016train,lei2020sharper,zhou2022understanding,mou2018generalization} studied the stability behavior of stochastic gradient methods for non-convex losses, which can be applied to  two-layer networks.   Specifically, to obtain  meaningful stability bounds,   \cite{hardt2016train} required a time-dependent step size $\eta_t=1/t$, which is insufficient to get a good convergence rate for optimization error.  \cite{charles2018stability,lei2020sharper,zhou2022understanding} established generalization bounds by introducing the Polyak-\L ojasiewicz condition,  which depends on a problem-dependent number. This number might be large in practice and results in a worse generalization bound.  
 It is hard to provide a direct comparison with their results since the learning settings are different.   
\end{remark}

\section{Proofs of Three-layer Neural Networks}\label{appendix:three}
\subsection{Proofs of Generalization Bounds}\label{appendix:three-generalization}
For a matrix $\bW$, let $\bW_{s;\cdot}$ and $\bW_{is}$ denote the  $s$-th row  and the $(i,s)$-th entry of $\bW$, respectively.
 \begin{lemma}[Smoothness and Curvature]\label{lem:3-unbounded-smoothness}
Suppose Assumptions~\ref{ass:activation} and \ref{ass:loss} hold. 
For any fixed $\bW=(\bW^{(1)},\bW^{(2)}) \in \R^{m\times d}\times \R^{m\times m}$ and any $z\in\Z$,  there holds
\[ \lambda_{\max}\big(\nabla^2 \ell(\bW;z)\big) \le \rho_{\bW} \text{ with } \]
\[\rho_\bW=\frac{B^4_{\sigma'}c^2_{\bx}}{m^{4c-1}}  \big\|   \bW^{(2)}   \big\|_2^2   +  \frac{B^2_{\sigma'}B^2_{\sigma}}{m^{4c-2 }}  +  C_{\bW } \Big( \frac{ B_{\sigma'}B_\sigma}{m^{2c-1}} \big\|  \bW^{(2)}  \big\|_2  + \sqrt{2c_0} \Big),\]
\[ \lambda_{\min}\big(\nabla^2 \ell( \bW;z) \big)\!\ge\! -  C_{\bW }\Big(  \frac{ B_{\sigma'}B_\sigma}{m^{2c-1}} \big\|  \bW^{(2)} \big\|_2 +  \sqrt{2c_0}\Big) , \]
where $  C_{\bW }=\frac{B^2_{\sigma'}B_{\sigma''}c_{\bx}^2}{m^{3c}} \big\|   \bW^{(2)} \big\|_2^2 + \Big(\frac{B_{
    \sigma'} B_{\sigma''} c_{\bx}^2}{m^{2c-\frac{1}{2}}}+ \frac{2B_{\sigma''}B_{\sigma'}B_\sigma c_\bx }{m^{3c-\frac{1}{2}}} \Big)\| \bW^{(2)} \|_2  + \frac{ B_{\sigma''}B_{
    \sigma}^2}{m^{3c-1}} 
    + \frac{ 2 B_{\sigma'}^2 c_\bx}{m^{2c-\frac{1}{2}}}$.
\end{lemma}
\begin{proof}
Let $A_{\bW^{(1)}}=[ \sigma(\bW_{1;\cdot}^{(1)}\bx),\ldots, \sigma(\bW_{m;\cdot}^{(1)}\bx) ]^\top 
\in \R^m$. 
Let $\tilde{\bw}=\bw^\top$. 
We first estimate the upper bound of $\|\nabla f_\bW(\bx)\|_2$.  Note that  for any $ k=1,\ldots,m$
\[ \nabla_{\tilde{\bw}^{(1)}_k}f_\bW(\bx)= \frac{\partial f_\bW(\bx)}{\partial \tilde{\bw}^{(1)}_k} = \frac{1}{m^{2c}} \sum_{i=1}^m a_i \sigma'
\Big( \frac{1}{m^c} \sum_{s=1}^m \bW_{is}^{(2)} \sigma( \bW_{s;\cdot}^{(1)} \bx ) \Big)  \bW^{(2)}_{ik}  \sigma'( \bW_{k;\cdot}^{(1)} \bx ) \bx  \]
 and   
\[ \nabla_{\tilde{\bw}^{(2)}_{k }} f_\bW(
\bx ) = \frac{\partial f_\bW(\bx )}{\partial \tilde{\bw}^{(2)}_{k}}=\frac{  a_k }{m^{2c}} \sigma'\Big( \frac{1}{m^c}   \sum_{s=1}^m \bW_{ks}^{(2)} \sigma( \bW^{(1)}_{s;\cdot} \bx )  \Big) A_{\bW^{(1)}}. \]
According to Assumptions~\ref{ass:activation} and \ref{ass:loss},  the upper bound of the gradient can be controlled as follows 
 \begin{align}\label{eq:3-nabla}
     \| \nabla f_\bW(\bx) \|_2^2&=\sum_{k=1}^m \Big( \big\|\nabla_{\tilde{\bw}^{(1)}_{k}} f_\bW(
\bx )\big\|_2^2  + \big\|\nabla_{\tilde{\bw}^{(2)}_{k}} f_\bW(
\bx )\big\|_2^2 \Big)\nonumber\\
& \le  \frac{B^4_{\sigma'}c^2_{\bx}}{m^{4c-1}} \sum_{k=1}^m \sum_{i=1}^m   \big|\bW_{ik}^{(2)}\big|^2 + \frac{B^2_{\sigma'}B^2_{\sigma}}{m^{4c-2}} \nonumber\\
&\le \frac{B^4_{\sigma'}c^2_{\bx}}{m^{4c-1}}  \big\|   \bW^{(2)}   \big\|_2^2  + \frac{B^2_{\sigma'}B^2_{\sigma}}{m^{4c-2 }}.  \end{align}
For any $k,j\in[m]$, we know
\begin{align*}
     \frac{\partial^2 f_\bW(\bx )}{\big(\partial\tilde{\bw}^{(1)}_{k}\big)^2} =&  \frac{1}{m^{3c}}\sum_{i=1}^m a_i \sigma''\Big( \frac{1}{m^c} \sum_{s=1}^m \bw^{2}_{i s} \sigma(\bW^{(1)}_{s;\cdot} \bx)\Big) \big(\bW^{(2)}_{i k}\big)^2 \big(\sigma'( \bW^{(1)}_{k;\cdot} \bx )\big)^2 \bx \bx^\top \\
     &+  \frac{1}{m^{2c}}\sum_{i=1}^m a_i \sigma'\Big( \frac{1}{m^c} \sum_{s=1}^m \bW^{2}_{i s} \sigma(\bW^{(1)}_{s;\cdot} \bx)\Big)  \bW^{(2)}_{i k}   \sigma''( \bW^{(1)}_{k;\cdot} \bx )  \bx \bx^\top,
\end{align*}  
\[   \frac{\partial^2 f_\bW(\bx )}{\big(\partial\tilde{\bw}^{(2)}_{k}\big)^2}= \frac{a_k}{m^{3c}} \sigma''\Big( \frac{1}{m^c} \sum_{s=1}^m \bW^{(2)}_{ks} \sigma(\bW^{(1)}_{s;\cdot}\bx)\Big) A_{\bW^{(1)}}A_{\bW^{(1)}}^\top \]
and
\begin{align*}
     \frac{\partial^2 f_\bW(\bx )}{ \partial\tilde{\bw}^{(1)}_{k} \partial\tilde{\bw}^{(2)}_{j} }= & \frac{1}{m^{3c}}a_{j} \sigma''\Big( \frac{1}{m^c} \sum_{s=1}^m \bW^{(2)}_{j s} \sigma( \bW_{s;\cdot}^{(1)}\bx )  \Big) \bW^{(2)}_{j k} \sigma'( \bW^{(1)}_{k;\cdot} \bx ) \bx A^\top_{\bW^{(1)}}   \\
     &+ \frac{1}{m^{2c}}a_{j} \sigma'\Big( \frac{1}{m^c} \sum_{s=1}^m \bW^{(2)}_{j s} \sigma( \bW_{s;\cdot}^{(1)}\bx )  \Big)  \sigma'( \bW^{(1)}_{k;\cdot} \bx ) \bx B_{k}  
\end{align*} 
where $B_{k}\in\R^{1\times m}$ with $k$-th element is 1 and others are 0. Let  the vector $\bu\in\R^{md+m^2}$ have unit norm $\|\bu\|_2=1 $ and be composed in a manner matching the parameter $\bW=( \bW^{(1)}, \bW^{(2)} )$  so that $\bu= ( \bu^{(1)},  \bu^{(2)}) $  where $\bu^{(1)}\in\R^{m\times d}$ and $\bu^{(2)}\in\R^{m\times m}$ have been vectorised in a row-major
manner with $\bu_k^{(1)}\in\R^{1\times d}$ and $\bu_k^{(2)}\in\R^{1\times m}$.
Then
\begin{align}\label{eq:3-hessian}
    &\bu^{\top} \nabla^2 f_\bW(\bx) \bu\nonumber\\
    &\!=\! \sum_{k=1}^m \Big( \bu^{(1)}_k \frac{\partial^2 f_\bW(\bx)}{\big(\partial \tilde{\bw}_k^{(1)}\big)^2}  \big(\bu^{(1)}_k\big)^{\top}\!+\!  \bu^{(2)}_k \frac{\partial^2 f_\bW(\bx)}{\big(\partial \tilde{\bw}_k^{(2)}\big)^2}\big(\bu^{(2)}_k\big)^{\top}  \!+\! 2 \sum_{j=1 }^m\bu^{(1)}_k \frac{\partial^2 f_\bW(\bx)}{ \partial \tilde{\bw}_k^{(1)}\partial \tilde{\bw}_j^{(2)} } \big(\bu^{(2)}_j\big)^{\top} \Big). 
\end{align}
We estimate the above three terms separately. Let $\bW_{\cdot k}^{(2)}$ denote the $k$-th column of $\bW^{(2)}$.   
\begin{align}\label{eq:3-hessian-1}
     &\sum_{k=1}^m  \bu^{(1)}_k \frac{\partial^2 f_\bW(\bx)}{\big(\partial \tilde{\bw}_k^{(1)}\big)^2} \big(\bu^{(1)}_k\big)^{\top}\nonumber \\
     &\le   \frac{B_{\sigma''} B^2_{\sigma'}}{m^{3c}} \sum_{k=1}^m \bu^{(1)}_k  \sum_{i=1}^m     \big(\bW^{(2)}_{i k}\big)^2  \bx \bx^\top \big(\bu^{(1)}_k\big)^{\top} + \frac{B_{\sigma'}B_{\sigma''}}{m^{2c}} \sum_{k=1}^m  \bu^{(1)}_k  \sum_{i=1}^m    \bW^{(2)}_{i k}    \bx \bx^\top\big(\bu^{(1)}_k\big)^{\top}\nonumber\\
     &\le  \frac{B_{\sigma''} B^2_{\sigma'}}{m^{3c}} \sum_{i=1}^m \sum_{k=1}^m \big(\bW^{(2)}_{ik}\big)^2 \big(   \bu^{(1)}_k    \bx  \big)^2+ \frac{B_{\sigma'}B_{\sigma''}}{m^{2c}} \sqrt{ \sum_{k=1}^m  \big(\sum_{i=1}^m    \bW^{(2)}_{i k}  \big)^2 } \sqrt{\sum_{k=1}^m    \big(\bu^{(1)}_k   \bx \bx^\top\big(\bu^{(1)}_k\big)^{\top}\big)^2 } \nonumber\\
     &\le  \frac{B_{\sigma''} B^2_{\sigma'}c_{\bx}^2}{m^{3c}} \sum_{k=1}^m \sum_{i=1}^m \big(\bW^{(2)}_{i k}\big)^2 \big\| \bu_k^{(1)} \big\|_2^2  + \frac{B_{\sigma'}B_{\sigma''}}{m^{2c}} \sqrt{m \sum_{k=1}^m  \sum_{i=1}^m   \big(\bW_{i k}^{(2)}\big)^2 }  \sum_{k=1}^m    \big(\bu^{(1)}_k   \bx \bx^\top\big(\bu^{(1)}_k\big)^{\top}\big)   \nonumber\\
     &\le \frac{B_{\sigma''} B^2_{\sigma'}c_{\bx}^2}{m^{3c}}  \big\|  \bW^{(2)} \big\|_2^2 + \frac{B_{\sigma'}B_{\sigma''}c_\bx^2}{m^{2c-\frac{1}{2}}}  \big\|\bW^{(2)} \big\|_2,
\end{align}
where  we used $\big(\sum_{i=1}^m \bW_{ik}^{(2)}\big)^2 \le   m\sum_{i=1}^m \big(  \bW_{i k}^{(2)}\big)^2 $ and $ \sum_{k=1}^m    \big(\bu^{(1)}_k   \bx \bx^\top\big(\bu^{(1)}_k\big)^{\top}\big)^2  \le  \big(\sum_{k=1}^m  \bu^{(1)}_k   \bx \bx^\top\big(\bu^{(1)}_k\big)^{\top}\big)^2   $ in the third inequality, and the last inequality follows from $  \bu^{(1)}_k   \bx \bx^\top\big(\bu^{(1)}_k\big)^{\top}= \big( \bu_k^{(1)}\bx \big)^2\le c_\bx^2 \| \bu_k^{(1)}\|_2^2  $ and $\| \bu_k^{(1)}\|_2^2 \le 1$.  

For the second term in \eqref{eq:3-hessian}, we control it by 
\begin{align}\label{eq:3-hessian-2}
    &\sum_{k=1}^m \bu^{(2)}_k  \frac{\partial^2 f_\bW(\bx)}{\big(\partial \tilde{\bw}_k^{(2)}\big)^2}  \big(\bu_k^{(2)}\big)^\top \le \frac{B_{\sigma''}}{m^{3c}} \sum_{k=1}^m \big( \bu^{(2)}_k  A_{\bW^{(1)}} \big)^2 \le \frac{B_{\sigma''}B_{\sigma}}{m^{3c-1}}  \big\| \bu^{(2)} \big\|_2^2 \le \frac{B_{\sigma''}B_{\sigma}}{m^{3c-1}}.
\end{align}
Further, according to Cauchy-Schwarz inequality, we can get
\begin{align}\label{eq:3-hessian-3}
    &\sum_{k=1}^m \sum_{ j=1 }^m  \bu^{(1)}_k \frac{\partial^2 f_\bW(\bx)}{ \partial \tilde{\bw}_k^{(1)}\partial \tilde{\bw}_j^{(2)}} \big(\bu^{(2)}_j \big)^{\top}  \nonumber\\  
     &\le \frac{B_{\sigma''}B_{\sigma'}B_{\sigma}}{m^{3c}}\sum_{j=1}^m     \bW_{j;\cdot}^{(2)} \bu^{(1)} \bx  \big\| \bu^{(2)}_j\big\|_1 + \frac{B_{\sigma'}^2}{m^{2c}}  \sum_{k=1}^m \bu_k^{(1)} \bx \Big(\sum_{j=1}^m \bu_{j k}^{(2)} \Big) \nonumber\\
     &\le \frac{B_{\sigma''}B_{\sigma'}B_{\sigma}}{m^{3c}}\sqrt{\sum_{j=1}^m   \big\| \bu^{(2)}_\ell\big\|_1^2} \sqrt{ \sum_{j=1}^m \Big(\bW_{j;\cdot}^{(2)} \bu^{(1)}  \bx\Big)^2 } + \frac{B_{\sigma'}^2}{m^{2c }} \sqrt{\sum_{k=1}^m \big(\bu_k^{(1)} \bx\big)^2} \sqrt{\sum_{k=1}^m \big\| \bu_{\cdot k}^{(2)} \big\|_1^2} \nonumber\\
     &\le \frac{B_{\sigma''}B_{\sigma'}B_{\sigma}}{m^{3c}}\sqrt{m
     \sum_{j=1}^m   \big\| \bu^{(2)}_j\big\|_2^2} \sqrt{ \sum_{j=1}^m \big\|\bW_{j;\cdot}^{(2)}\big\|_2^2 \big\|   \bu^{(1)}  \bx\big\|_2^2 } + \frac{B_{\sigma'}^2c_\bx}{m^{2c  }} \sqrt{\sum_{k=1}^m  \big\|\bu_k^{(1)} \big\|_2^2} \sqrt{m\sum_{k=1}^m \big\| \bu_{\cdot k}^{(2)} \big\|_2^2} \nonumber\\
     &\le  \frac{B_{\sigma''}B_{\sigma'}B_{\sigma}c_\bx }{m^{3c-\frac{1}{2}}} \big\| \bu^{(1)} \big\|_2 \big\| \bu^{(2)} \big\|_2   \big\|    \bW^{(2)} \big\|_2+ \frac{B_{\sigma'}^2c_\bx}{m^{2c -\frac{1}{2} }}   \big\|\bu^{(1)} \big\|_2  \big\|\bu^{(2)}\big\|_2\nonumber\\
     &\le \frac{B_{\sigma''}B_{\sigma'}B_{\sigma}c_\bx }{m^{3c-\frac{1}{2}}}     \big\|     \bW^{(2)} \big\|_2   + \frac{B_{\sigma'}^2c_\bx}{m^{2c -\frac{1}{2} }}    ,
\end{align}
 where in the first equality we used $\sum_{k=1}^m   \bu^{(1)}_k \bW_{j k}^{(2)} \bx = \bW_{j;\cdot}^{(2)}  \bu^{(1)} \bx$,   the second inequality follows from $\Big(\sum_{j=1}^m \bu_{j k}^{(2)} \Big) \le  \big\| \bu_{\cdot k}^{(2)} \big\|_1     $, here $\bu_{\cdot k}^{(2)}$ denotes the $k$-th column of $\bu^{(2)}$.
 
Plugging \eqref{eq:3-hessian-1}, \eqref{eq:3-hessian-2} and \eqref{eq:3-hessian-3} back into \eqref{eq:3-hessian} we can get
\begin{align}\label{eq:3-hessian-4}
    &\big\| \nabla^2f_\bW(\bx) \big\|_{op}\nonumber\\
    &\le  \frac{B^2_{\sigma'}B_{\sigma''}c_{\bx}^2}{m^{3c}} \big\|   \bW^{(2)} \big\|_2^2 \!+\! \Big(\frac{B_{
    \sigma'} B_{\sigma''} c_{\bx}^2}{m^{2c-\frac{1}{2}}}\!+\! \frac{2B_{\sigma''}B_{\sigma'}B_\sigma c_\bx }{m^{3c-\frac{1}{2}}} \Big)\| \bW^{(2)} \|_2 \! +\! \frac{ B_{\sigma''}B_{
    \sigma}^2}{m^{3c-1}} 
    \!+\! \frac{ 2 B_{\sigma'}^2 c_\bx}{m^{2c-\frac{1}{2}}}:= C_{\bW }.
\end{align}
For any $\bW, \widetilde{\bW}$,  according to Assumptions~\ref{ass:activation} ans \ref{ass:loss} we can get
\begin{align}\label{eq:3-hessian-6}
  &  \big| f_\bW(\bx) -   f_{\widetilde{\bW}}(\bx) \big| \nonumber\\
  & \le \frac{1}{m^c}\sum_{k=1}^m \big| \sigma
\Big( \frac{1}{m^c} \sum_{s=1}^m \bW_{ks}^{(2)} \sigma( \bW_s^{(1)} \bx ) \Big)  - \sigma
\Big( \frac{1}{m^c} \sum_{s=1}^m  \widetilde{\bW}_{ks}^{(2)} \sigma(  \widetilde{\bW}_{s;\cdot}^{(1)} \bx ) \Big)  \big|\nonumber\\
&\le \frac{ B_{\sigma'}}{m^{2c}}\sum_{k=1}^m \sum_{s=1}^m\big|  \bW_{ks}^{(2)} \sigma( \bW_{s;\cdot}^{(1)} \bx ) -  \widetilde{\bW}_{ks}^{(2)} \sigma( \bW_{s;\cdot}^{(1)} \bx ) + \widetilde{\bW}_{ks}^{(2)} \sigma( \bW_{s;\cdot}^{(1)} \bx ) -    \widetilde{\bW}_{ks}^{(2)} \sigma( \widetilde{\bW}_{s;\cdot}^{(1)} \bx )\big|\nonumber\\
&\le \frac{ B_{\sigma'}B_\sigma}{m^{2c}}\sum_{k=1}^m \sum_{s=1}^m\big|  \bW_{ks}^{(2)}  -     \widetilde{\bW}_{ks}^{(2)}  \big| + \frac{ B^2_{\sigma'}}{m^{2c}}\sum_{k=1}^m \sum_{s=1}^m \big|\widetilde{\bW}_{ks}^{(2)} \big|\cdot \big|    \big(\bW_{s;\cdot}^{(1)}  -  \widetilde{\bW}_{s;\cdot}^{(1)}\big) \bx  \big|\nonumber\\
&\le  \frac{ B_{\sigma'}B_\sigma}{m^{2c-1}} \big\|  \bW^{(2)}  -     \widetilde{\bW}^{(2)}  \big\|_2 + \frac{ B^2_{\sigma'}c_\bx \|\widetilde{\bW}^{(2)}\|_{\infty}}{m^{2c-\frac{3}{2}}}   \big\|     \bW^{(1)}  -  \widetilde{\bW}^{(1)}  \big\|_2.
\end{align}
Since
\begin{equation}\label{eq:hessian-5}
    \nabla^2 \ell(\bW;z)= \nabla f_\bW(\bx) \nabla f_\bW(\bx)^{\top} + \nabla^2 f_\bW(\bx)\big( f_\bW(\bx) -y \big)  .
\end{equation}
  Then for any $\bW \in\R^{md+m^2}$, we can upper bound the maximum eigenvalue of the Hessian by combining \eqref{eq:3-nabla},    \eqref{eq:3-hessian-4} and \eqref{eq:3-hessian-6} with $\widetilde{\bW}=\mathbf{0}$ together
\begin{align*}
   \lambda_{\max}(\nabla^2 \ell( \bW;z )) 
   &\le \| \nabla f_{\bW}(\bx) \|_2^2 + \|\nabla^2 f_\bW(\bx)\|_{op} |f_\bW(\bx) -y| \\
   &\le  \frac{B^4_{\sigma'}c^2_{\bx}}{m^{4c-1}}  \big\|   \bW^{(2)}   \big\|_2^2  + \frac{B^2_{\sigma'}B^2_{\sigma}}{m^{4c-2 }} + C_{\bW } \big( \big|f_\bW(\bx) - f_{\mathbf{0}}(\bx)\big| + \big| f_{\mathbf{0}}(\bx) - y \big| \big)\\
   &\le \frac{B^4_{\sigma'}c^2_{\bx}}{m^{4c-1}}  \big\|   \bW^{(2)}   \big\|_2^2 \! +\! \frac{B^2_{\sigma'}B^2_{\sigma}}{m^{4c-2 }} \!+\! C_{\bW } \Big( \frac{ B_{\sigma'}B_\sigma}{m^{2c-1}} \big\|  \bW^{(2)}   \big\|_2   + \sqrt{2\ell(\mathbf{0};z)} \Big)
   \\
     &\le \frac{B^4_{\sigma'}c^2_{\bx}}{m^{4c-1}}  \big\|   \bW^{(2)}   \big\|_2^2   +  \frac{B^2_{\sigma'}B^2_{\sigma}}{m^{4c-2 }}  +  C_{\bW } \Big( \frac{ B_{\sigma'}B_\sigma}{m^{2c-1}} \big\|  \bW^{(2)}  \big\|_2  + \sqrt{2c_0} \Big). 
\end{align*}
Note that  $\nabla f_\bW(\bx) \nabla f_\bW(\bx)^{\top}$ is positive semi-definite, then from \eqref{eq:3-hessian-4}, \eqref{eq:hessian-5} and  \eqref{eq:3-hessian-6} with $\widetilde{\bW}=\mathbf{0}$  we can get
\begin{align*}
     \lambda_{\min}( \nabla^2 \ell(\bW;z) ) 
    &\ge  -  \| \nabla^2 f_\bW(\bx)  \|_{op}\big|  f_\bW(\bx) -y \big| \nonumber\\
    & \ge -  C_{\bW }\Big( \big|  f_\bW(\bx) -  f_{\mathbf{0}}(\bx) \big| + \big|   f_{\mathbf{0}}(\bx)-y  \big|  \Big)\nonumber\\
    &\ge -   C_{\bW }\Big(   \frac{ B_{\sigma'}B_\sigma}{m^{2c-1}} \big\|  \bW^{(2)}    \big\|_2   +  \sqrt{2\ell(\bW_0;z)}\Big)\nonumber\\
    &\ge -  C_{\bW }\Big(  \frac{ B_{\sigma'}B_\sigma}{m^{2c-1}} \big\|  \bW^{(2)} \big\|_2 +  \sqrt{2c_0}\Big).
\end{align*}
The proof is completed. 
\end{proof}
Let $B_1=\max\big\{ B_{\sigma'}^2 B_{\sigma''} c_{\bx}^2, B_{\sigma'} B_{\sigma''} c_{\bx}^2 , 2 B_{ \sigma''} B_{\sigma} c_{\bx},   B_{ \sigma''} B_{\sigma}^2, 2  B_{ \sigma''}^2 c_\bx\big\}$
and 
$ B_2= \max\big\{ B_{\sigma'}^4  c_{\bx}^2, B_{\sigma'}^2 B_{\sigma''}^2,  B_{ \sigma''} B_{\sigma}, \sqrt{2c_0}\big\}.$
\begin{lemma}\label{lem:3-unbounded-wbound}
Suppose Assumptions~\ref{ass:activation} and \ref{ass:loss} hold.
 Let $\{\bW_t\}$ and $\{\bW'_t\}$ be produced by GD iterates with   $T$ iterations based on $S$ and $S'$, respectively.  Let $\hat{\rho}=4B_2(1+2B_1)$ and $C>0$ be a constant. Assume $\eta \le 1/(2\hat{\rho})$ and \eqref{eq:3-m-condition-1} hold.  
     Then  for any $c\in(1/2,1]$ and   any $t=0,\ldots,T$ there holds \[\|\bW_t-\bW_0 \|_2  \le \sqrt{2c_0\eta t }\]  and
     \[  \| \nabla \ell(\bW_t;z) - \nabla \ell(\bW'_t;z)  \|_2\le\hat{\rho} \|\bW_t-\bW_t'\|_2.\]
\end{lemma}
\begin{proof}
We will prove by induction to show $\|\bW_t-\bW_0 \|_2\le \eta t m^{2c-1}$. Further, we can show that $\rho_{\bW}=\O(1)$ for any $\bW$ produced by GD iterates if $m$ satisfies \eqref{eq:3-m-condition-1}. Then by assuming $\eta \le 1/(2\hat{\rho})$ we can prove that $\|\bW_t-\bW_0 \|_2\le \sqrt{2c_0\eta t}$. 

It's obvious that $\|\bW_t-\bW_0 \|_2\le \eta t m^{2c-1 }$ with $t=0$ holds trivially. Assume $\|\bW_t-\bW_0\|_2\le \eta t m^{2c-1}$, 
according to the update rule \eqref{eq:GD} we know
\begin{align}\label{eq:smooth-11}
    &\|\bW_{t+1}-\bW_0 \|_2\nonumber\\ &\le \|\bW_t-\bW_0\|_2 + \eta \|\nabla L_S(\bW_t)\|_2\nonumber\\
    &\le \|\bW_t-\bW_0\|_2 + \eta \max_{i\in[n]} \|\nabla f_{\bW_t}(\bx_i)\|_2 \big| f_{\bW_t}(\bx_i) - y_i \big|\nonumber\\
        &\le \|\bW_t-\bW_0\|_2 + \eta\Big(\frac{B_{\sigma'}^2 c_{\bx}}{m^{2c-\frac{1}{2}}}(\|\bW_t-\bW_0\|_2+\|\bW_0\|_2) + \frac{ B_{\sigma'} B_{\sigma}}{m^{2c-1}}  \Big)\Big( \big| f_{\bW_t}(\bx_i)-f_{\mathbf{0}}(\bW_t) \big|\nonumber\\
        &\quad + \big| f_{\mathbf{0}}(\bW_t)  - y_i \big| \Big) \nonumber\\
    &\le \|\bW_t-\bW_0\|_2 + \eta \Big(\frac{B_{\sigma'}^2 c_{\bx}}{m^{2c-\frac{1}{2}}}(\|\bW_t-\bW_0\|_2+\|\bW_0\|_2) + \frac{B_{\sigma'}B_{\sigma}}{m^{2c-1}}\Big)\Big( \frac{ B_{\sigma'} B_{\sigma}}{m^{2c-1}} \big(\|\bW_t-\bW_0\|_2\nonumber\\
    &\quad +\|\bW_0\|_2\big) + \sqrt{2c_0}\Big) ,
\end{align}
where in the third inequality we used \eqref{eq:3-nabla}, the last inequality used \eqref{eq:3-hessian-6} with $\widetilde{\bW}=\mathbf{0} $. 
If $m$ is large enough such that  
\begin{align}\label{eq:w-bound-1}
   &\Big(\frac{B_{\sigma'}^2 c_{\bx}}{m^{2c-\frac{1}{2}}}(\|\bW_t\!-\!\bW_0\|_2\!+\!\|\bW_0\|_2) + \frac{B_{\sigma'}B_{\sigma}}{m^{2c-1}}\Big)\Big( \frac{ B_{\sigma'} B_{\sigma} }{m^{2c-1}} \big(\|\bW_t-\bW_0\|_2+\|\bW_0\|_2\big) \!+\! \sqrt{2c_0}\Big)\nonumber\\
   &\le m^{2c-1},
\end{align}
then from \eqref{eq:smooth-11} we know that $\|\bW_{t+1}-\bW_0\|_2\le \eta (t+1) m^{2c-1}$. The first part of the lemma can be proved. 
Now, we discuss the conditions on $m$ such that \eqref{eq:w-bound-1} holds. Let $x=\|\bW_t-\bW_0\|_2+\|\bW_0\|_2$. To guarantee \eqref{eq:w-bound-1}, it suffices that the following three  inequalities hold
\begin{equation} \label{condition}   \frac{B_{\sigma'}^3B_{\sigma} c_{\bx}}{m^{4c-\frac{3}{2}}}x^2 \le  \frac{m^{2c-1}}{3}, 
    \big(\frac{B_{\sigma'}^2B_{\sigma}^2  }{m^{4c-2}}+  \frac{\sqrt{2c_0} B_{\sigma'}^2c_{\bx}}{m^{2c-\frac{1}{2}}}\big) x \le\!\frac{m^{2c-1}}{3}, 
\frac{\sqrt{2c_0}B_{\sigma'}B_{\sigma} }{m^{2c-1}} \le  \frac{m^{2c-1}}{3}. 
\end{equation} 
It's easy to verify that \eqref{condition} holds if $m\gtrsim \max \{(\eta T)^{\frac{4}{4c-1}}, (\eta T)^{\frac{1}{4c-2}}, \|\bW_0\|_2^{\frac{1}{3c-\frac{5}{4}}},\|\bW_0\|_2^{\frac{1}{6c-3}} \}$ , which can be ensured by  \eqref{eq:3-m-condition-1}. 
Hence, if $c\in(1/2,1]$ and  \eqref{eq:3-m-condition-1} holds, we have $\|\bW_t-\bW_0\|_2\le \eta t m^{2c-1}$ for all $t=0,1,\ldots,T$. 

Recall that \[B_1=\max\big\{ B_{\sigma'}^2 B_{\sigma''} c_{\bx}^2, B_{\sigma'} B_{\sigma''} c_{\bx}^2 , 2 B_{ \sigma''} B_{\sigma} c_{\bx},   B_{s\sigma''} B_{\sigma}^2, 2  B_{s\sigma''}^2 c_\bx\big\}\]
and 
\[ B_2= \max\big\{ B_{\sigma'}^4  c_{\bx}^2, B_{\sigma'}^2 B_{\sigma''}^2,  B_{s\sigma''} B_{\sigma}, \sqrt{2c_0}\big\}. \]
Then from Lemma~\ref{lem:3-unbounded-smoothness} we know 
\[ C_\bW\le B_1\Big(  {m^{-3c}} \big\|   \bW^{(2)} \big\|_2^2 + \big( {m^{\frac{1}{2}-2c}}+  {m^{\frac{1}{2}-3c}} \big)\| \bW^{(2)} \|_2  +  {m^{1-3c}} 
    +  {m^{\frac{1}{2}-2c}} \Big)\]
    and
\[\rho_{\bW} \le B_2 \Big(  m^{1-4c}  \big\|   \bW   \big\|_2^2  +  {m^{2-4c }} + C_{\bW } \big(  {m^{1-2c}} \big\|  \bW  \big\|_2 +  1\big) \Big).\]
Note that \eqref{eq:3-m-condition-1} implies $m \gtrsim (\eta T)^4 + \|\bW_0\|_2^{\frac{4}{8c-3}}$. By using $\|\bW_t \|_2\le \eta t m^{2c-1} +\|\bW_0\|_2$ we can verify that 
\[ \rho_{\bW_t}\le 4B_2(1+2B_1):=\hat{\rho} \quad  \text{ for any  } \quad  t=0,\ldots,T. \]
Hence, we know that $\ell$ is $\hat{\rho}$-smooth when the parameter space is the trajectory of GD. Then for any $t=0,\ldots,T$ and any $\bW_t,\bW_t' $ produced by GD iterates, there holds
\begin{equation}
    \| \nabla \ell(\bW_t;z) - \nabla \ell(\bW'_t;z)  \|_2\le \hat{\rho} \|\bW_t-\bW_t'\|_2. 
\end{equation}
In addition, by the smoothness of $\ell$ we can get for any $j=0,\ldots,T-1$
\[ L_S( \bW_{j+1} )\le L_S(\bW_j) - \eta \Big(1 - \frac{\eta \hat{\rho}}{2}\Big)\| \nabla L_S( \bW_j ) \|_2^2. \]
Rearranging and summing over $j$ yields
\[ \eta \Big(1 - \frac{\eta \hat{\rho}}{2}\Big)\sum_{j=0}^t \| \nabla L_S( \bW_j ) \|_2^2 \le \sum_{j=0}^t L_S(\bW_j) -L_S( \bW_{j+1} )\le  L_S(\bW_0)   . \]
Note that the update rule of GD \eqref{eq:GD} implies
\[ \bW_{t+1}  =\bW_0 -\eta \sum_{j=0}^t \nabla L_S( \bW_j ). \]
Combining the above two equations  together and noting that $\eta \hat{\rho}\le 1/2$, we obtain
\[ \| \bW_{t+1}-\bW_0   \|_2^2 = \eta^2  \big\|\sum_{j=0}^t \nabla L_S( \bW_j )\big\|^2_2 \le  \eta^2  t \sum_{j=0}^t \big\| \nabla L_S( \bW_j )\big\|^2_2\le 2\eta t L_S(\bW_0)\le 2 c_0 \eta t.\]
The  proof is completed. 
\end{proof}

The almost co-coercivity of the gradient operator for three-layer neural networks  is given as follows. Recall that $S^{(i)}=\{z_1,\ldots,z_{i-1},z'_i,z_{i+1},z_n\}$ is the set formed from $S$ by replacing the $i$-th element with $z_i'$ and for any $\bW$, 
\[ L_{S^{\setminus i}}(\bW)=L_S(\bW) - \frac{1}{ n}\ell(\bW;z_i) = L_{S^{(i)}}(\bW)-\frac{1}{ n}\ell(\bW;z_i') .\]
Let $\{\bW_t \}$ and $\{\bW_t^{(i)}\}$ be the sequence produced by GD based on $S$ and $S^{(i)}$, respectively. Let $C_{3,T}= 4B_1\big( 2c_0 \eta T {m^{-3c}}   + \big( {m^{\frac{1}{2}-2c}}+  {m^{\frac{1}{2}-3c}} \big)\sqrt{2c_0 \eta T} +  {m^{1-3c}} 
    +  {m^{\frac{1}{2}-2c}} \big)$. 
\begin{lemma}\label{lem:3-unbounded-coercivity}
    Suppose $\eta \le 1/(8 \hat{\rho})$ where $ \hat{\rho}=4B_2(1+2B_1)$. Assume   \eqref{eq:3-m-condition-1} holds.  Then
\begin{align*}
    \langle \bW_t - \bW_t^{(i)}, \nabla L_{S^{\setminus i}}(\bW_t) - \nabla L_{S^{\setminus i}}&(\bW^{(i)}_t)  \rangle \ge    2\eta \Big(1-4\eta \hat{\rho} \Big)\big\| \nabla L_{S^{\setminus i}}(\bW_t) - \nabla L_{S^{\setminus i}}(\bW^{(i)}_t)   \big\|_2^2\\
    & - \tilde{\epsilon}_t \big\| \bW_t -\bW_t^{(i)} -\eta \big( \nabla L_{S^{\setminus i}}(\bW_t) - \nabla L_{S^{\setminus i}}(\bW_t^{(i)}) \big) \big\|_2^2,
\end{align*}
where $\tilde{\epsilon}_t =C_{3,T}\big( B_3\big( \frac{\sqrt{\eta T}+\|\bW_0\|_2}{m^{2c-\frac{1}{2}}}+ m^{1-2c} \big)\big( (1+\eta \hat{\rho})\|\bW_t-\bW_t^{(i)}\|_2 + 2 \big(\sqrt{2c_0 \eta T}+\|\bW_0\|_2 \big)+ \sqrt{2c_0}\big)$. 
\end{lemma}
\begin{proof}
For any $\bW\in \mathcal{B}\big(0,2(\sqrt{2c_0\eta T}+\|\bW_0\|_2 )\big)$,  defining the following two functions
\[ G_1(\bW)=L_{S^{\setminus  i}}(\bW)-\langle \nabla L_{S^{\setminus  i}}(\bW_t^{(i)}), \bW \rangle, \qquad  G_2(\bW)=L_{S^{\setminus  i}}(\bW)-\langle \nabla L_{S^{\setminus  i}}(\bW_t ), \bW \rangle.\]
Note that 
\begin{align}\label{eq:14}
     \!\langle \bW_t \!-\! \bW_t^{(i)}, \nabla L_{S^{\setminus i}}(\bW_t) \!-\! \nabla L_{S^{\setminus i}}(\bW^{(i)}_t)  \rangle \!=\! \big(G_1\big(\bW_t\big) \!-\! G_1\big(\bW^{(i)}_t\big) \big)\!+\!  \big(G_2\big(\bW^{(i)}_t\big) \!-\! G_2\big(\bW_t\big)  \big).
\end{align}
Hence, it is enough to lower bound $G_1\big(\bW_t\big) - G_1\big(\bW^{(i)}_t\big)$ and $G_2\big(\bW^{(i)}_t\big) - G_2\big(\bW_t\big) $.

Note that  for any $\bW_t$, Lemma~\ref{lem:3-unbounded-smoothness} implies that $\|\bW_t\|_2\le \sqrt{2c_0\eta t} + \|\bW_0\|_2$. Then there holds
 \begin{align*}
     \big\|\bW_t - \eta \nabla G_1(\bW_t)\big\|_2 &\le   \big\|\bW_t\big\|_2 +\eta \big\| \nabla L_{S^{\setminus  i}}(\bW_t)-  \nabla L_{S^{\setminus  i}}(\bW_t^{(i)})\big\|_2\\& \le  \big\|\bW_t\big\|_2 + \eta \hat{\rho}\big\| \bW_t- \bW_t^{(i)}\big\|_2\\
     &\le \sqrt{2c_0\eta t} + \|\bW_0\|_2 + 2 \eta \hat{\rho} \big( \sqrt{2c_0\eta t} + \|\bW_0\|_2 \big) \\
     &\le 2\big(\sqrt{2c_0\eta t} + \|\bW_0\|_2 \big),
 \end{align*}
 where  in the last inequality we used $\eta \hat{\rho}  \le 1/8$. 
 Similarly, we can show that  $\big\|\bW_t^{(i)} - \eta \nabla G_2(\bW_t^{(i)})\big\|_2 \le 2(\sqrt{2c_0\eta T}+\|\bW_0\|_2 )$. Hence, we know $\bW_t - \eta \nabla G_1(\bW_t) \in \mathcal{B}\big(0,2(\sqrt{2c_0\eta T}+\|\bW_0\|_2 )\big)$ and $ \bW_t^{(i)} - \eta \nabla G_2(\bW_t^{(i)}) \in \mathcal{B}\big(0,2(\sqrt{2c_0\eta T}+\|\bW_0\|_2 )\big)$. 
On the other hand, similar to     Lemma~\ref{lem:3-unbounded-wbound}, we can show that  $G_1(\bW)$ and $G_2(\bW)$ is $8\hat{\rho}$-smooth for any $\bW \in \mathcal{B}\big(0,2(\sqrt{2c_0\eta T}+\|\bW_0\|_2 )\big)$.  
Combining the above results, 
we can get
 \begin{equation}\label{eq:3-coercivity-2} G_1(\bW_t-\eta \nabla G_1(\bW_t) ) \le G_1(\bW_t) - \eta \big(1- {4\eta \hat{\rho} } \big) \| \nabla G_1(\bW_t) \|_2^2, 
 \end{equation} \begin{equation}\label{eq:3-coercivity-3}
     G_2(\bW_t^{(i)}-\eta \nabla G_2(\bW_t^{(i)}) ) \le G_2(\bW_t^{(i)}) - \eta \big(1- {4\eta \hat{\rho}} \big) \| \nabla G_2(\bW_t^{(i)}) \|_2^2.
 \end{equation}  
If we can further show that
 \begin{equation}\label{eq:3-coercivity-4}
     G_1(\bW_t-\eta \nabla G_1(\bW_t) ) \ge G_1(\bW_t^{(i)}) -\frac{\tilde{\epsilon}_t}{2}\| \bW_t -\bW_t^{(i)} -\eta \nabla G_1(\bW_t) \|_2^2,
 \end{equation}
  \begin{equation}\label{eq:3-coercivity-5}
    G_2(\bW_t^{(i)}-\eta \nabla G_2(\bW_t^{(i)}) )\ge G_2(\bW_t ) -\frac{\tilde{\epsilon}_t}{2}\|    \bW_t^{(i)}-\bW_t -\eta \nabla G_2(\bW_t) \|_2^2,
 \end{equation}
 with $\tilde{\epsilon}_t =C_{3,T}\big( B_3\big( \frac{\sqrt{\eta T}+\|\bW_0\|_2}{m^{2c-\frac{1}{2}}}+ m^{1-2c} \big)\big( (1+\eta \hat{\rho})\|\bW_t-\bW_t^{(i)}\|_2 + 2\big(\sqrt{2c_0 \eta T} +\|\bW_0\|_2\big)+ \sqrt{2c_0}\big)$, where $C_{3,T}= 4B_1\big( 2c_0 \eta T {m^{-3c}}   + \big( {m^{\frac{1}{2}-2c}}+  {m^{\frac{1}{2}-3c}} \big)\sqrt{2c_0 \eta T} +  {m^{1-3c}} 
    +  {m^{\frac{1}{2}-2c}} \big)$. 
Then combining \eqref{eq:3-coercivity-2}, \eqref{eq:3-coercivity-3}, \eqref{eq:3-coercivity-4} and \eqref{eq:3-coercivity-5} together yields 
  \begin{equation*}
     G_1(\bW_t  ) - G_1(\bW_t^{(i)}) \ge \eta \big(1- {4\eta \hat{\rho}} \big)\|\nabla G_1(\bW_t)\|_2^2 -\frac{\tilde{\epsilon}_t}{2}\| \bW_t -\bW_t^{(i)} -\eta \nabla G_1(\bW_t) \|_2^2,
 \end{equation*}
  \begin{equation*}
    G_2(\bW_t^{(i)}  ) -    G_2(\bW_t  ) \ge \eta \big(1- {4\eta \hat{\rho}} \big)\|\nabla G_2(\bW_t)\|_2^2 -\frac{\tilde{\epsilon}_t}{2}\|    \bW_t^{(i)}-\bW_t -\eta \nabla G_2(\bW_t) \|_2^2.
 \end{equation*}
Plugging the above two inequalities back into \eqref{eq:14} yields 
 \begin{align*}
    & \langle \bW_t - \bW_t^{(i)}, \nabla L_{S^{\setminus i}}(\bW_t) - \nabla L_{S^{\setminus i}}(\bW^{(i)}_t)  \rangle  =  G_1(\bW_t) -G_1(\bW_t^{(i)}) +  G_2(\bW_t^{(i)}  ) -    G_2(\bW_t  )\\
     &\!\ge\! 2\eta \big(1\!-\! {4\eta \hat{\rho}} \big)\big\| \nabla L_{S^{\setminus i}}\!(\bW_t) \!-\! \nabla L_{S^{\setminus i}}\!(\bW^{(i)}_t)   \!\big\|_2^2 \!\! -\! \tilde{\epsilon}_t \big\| \bW_t \!-\!\bW_t^{(i)} \!\!-\!\eta \big( \nabla L_{S^{\setminus i}}\!(\bW_t) \!-\! \nabla L_{S^{\setminus i}}\!(\bW_t^{(i)}) \big) \!\big\|_2^2.
 \end{align*} 
The desired result has been proved. 
   
Now,  we give the proof of \eqref{eq:3-coercivity-4} and \eqref{eq:3-coercivity-5}. For $\alpha\in[0,1]$, let $ \bW(\alpha) = \alpha \bW_t  + (1-\alpha)\bW_t^{(i)} -\alpha \eta \big( \nabla L_{S^{\setminus i}}(\bW_t)-\nabla L_{S^{\setminus i}}(\bW_t^{(i)}) \big)$.     For any $\alpha\in[0,1]$, it's obvious that  
$\| \bW(\alpha) \|_2\le 2(\sqrt{2c_0\eta t}+\|\bW_0\|_2)$ by using Lemma~\ref{lem:3-unbounded-wbound}.  Combining this observation  with \eqref{eq:hessian-5} we can obtain
\begin{align}\label{eq:3-coercivity-6}
    &\lambda_{\min}\big(\nabla^2 L_{S^{\setminus  i}}(\bW(\alpha) ) \big)\nonumber
    \\
    &\ge  -  \max_{j\in[n]} \| \nabla^2 f_{\bW(\alpha)}(\bx_j)  \|_{op}  \big|  f_{\bW(\alpha) }(\bx_j) -y_j \big|  \nonumber\\
    &\ge - C_{\bW(\alpha)} \max_{j\in[n]}\Big(  \big|  f_{\bW(\alpha) }(\bx_j) - f_{\bW_t^{(i)} }(\bx_j)\big| + \big| f_{\bW_t^{(i)} }(\bx_j)- f_{\mathbf{0}}(\bx_j) \big| + \big|   f_{\mathbf{0}}(\bx_j)-y_j  \big|  \Big)\nonumber\\
    &\ge - C_{\bW(\alpha)} \Big(  \Big( \frac{B_{\sigma'}^2 c_{\bx}}{m^{2c-\frac{1}{2}}}\|\bW(\alpha)\|_2 + \frac{B_{\sigma'}B_\sigma}{m^{2c-1}} \Big)\big( \| \bW(\alpha) - \bW_t^{(i)}\|_2 +  \|  \bW_t^{(i)} \|_2 \big)  +\sqrt{2c_0}  \Big)\nonumber\\
    &\ge - C_{\bW(\alpha)}  \Big( B_3\big( \frac{\sqrt{\eta T}+\|\bW_0\|_2}{m^{2c-\frac{1}{2}}}+ m^{1-2c} \big)\big( \|\bW_t-\bW_t^{(i)}\|_2 +  \eta \| \nabla \ell(\bW_t) - \ell(\bW_t^{(i)}) \|_2 \nonumber\\
    &\quad +\|\bW_t^{(i)}\|_2 \big) + \sqrt{2c_0 } \Big)\nonumber\\
    &\ge - \tilde{\epsilon}_t
    , 
\end{align}
 where  the second inequality is due to   \eqref{eq:3-hessian-4}, the third inequality is according to \eqref{eq:3-nabla} and in the last inequality  we used Lemma~\ref{lem:3-unbounded-wbound}. Similarly, 
let $ \widetilde{\bW}(\alpha) = \alpha \bW_t^{(i)}  + (1-\alpha)\bW_t  -\alpha \eta \big(\nabla L_{S^{\setminus i}}(\bW_t^{(i)})- \nabla L_{S^{\setminus i}}(\bW_t)  \big)$,   we can also control $\lambda_{\min}\big(\nabla^2 L_{S^{\setminus  i}}(\widetilde{\bW}(\alpha) ) \big) $ by $- \tilde{\epsilon}_t$.

Let $\Delta=\big\| \bW_t -\bW_t^{(i)} -\eta \big( \nabla L_{S^{\setminus i}}(\bW_t) - \nabla L_{S^{\setminus i}}(\bW_t^{(i)}) \big) \big\|_2^2 $. We define \[g_1(\alpha)= G_1(\bW(\alpha)) +\frac{\tilde{\epsilon}_t\alpha^2 }{2} \Delta,  \qquad g_2(\alpha)= G_2(\widetilde{\bW}(\alpha)) +\frac{\tilde{\epsilon}_t \alpha^2}{2} \Delta .\]
 From \eqref{eq:3-coercivity-6}   we know that $g_1''(\alpha)\ge 0$  for any $\alpha\in[0,1]$. Hence, $g_1$ is convex on $[0,1]$.  Then  from the convexity of $g_1$ we can get that
 \begin{align*}
     0=g_1'(0)&\le g_1(1)-g_1(0) \le  G_1(\bW_t-\eta \nabla G_1(\bW_t) )  +\frac{\tilde{\epsilon}_t}{2}\Delta - G_1(\bW_t^{(i)}),
 \end{align*}
 which completes the proof of \eqref{eq:3-coercivity-4}. We can also show $g_2(\alpha)$ is convex on $[0,1]$ and prove \eqref{eq:3-coercivity-5} in a similar way. 
 The proof is completed. 
\end{proof}
Based on Lemma~\ref{lem:3-unbounded-smoothness}, Lemma~\ref{lem:3-unbounded-wbound} and  Lemma~\ref{lem:3-unbounded-coercivity} we can establish the following uniform stability bounds for three-layer neural networks. 

\begin{theorem}[Uniform Stability]\label{thm:3-unbounded-stability}
Suppose Assumptions~\ref{ass:activation} and \ref{ass:loss} hold. Let $S$ and $ S^{(i)}$ be constructed in Definition~\ref{def:stability}. Let $\{\bW_t\}$ and $\{\bW_t^{(i)}\}$ be produced by \eqref{eq:GD} with $\eta\le 1/(8\hat{\rho})$ based on $S$ and $S^{(i)}$, respectively.  
Assume \eqref{eq:3-m-condition-1}  holds. 
Then, for any $t\in[T]$, there holds
\begin{align*}
     \big\| \bW_{t } -\bW_{t }^{(i)} \big\|_2 \le  \frac{2 e   \eta T  \sqrt{2c_0\hat{\rho}   (  \hat{\rho}\eta T   +   2 )} }{ n   }  .
\end{align*}
\end{theorem}
\begin{proof}
Similar to   \eqref{eq:stab-1}, by the update rule $\bW_{t+1}=  \bW_{t} - \eta \nabla L_S(\bW_t)  $  we know
\begin{align}\label{eq:3-stab-1}
     \big\| \bW_{t+1} -\bW_{t+1}^{(i)} \big\|_2^2 
    &\le (1+p)\big\| \bW_{t} -\bW_{t}^{(i)} -\eta\big(\nabla  L_{S^{\setminus i}}(\bW_t) -\nabla  L_{S^{\setminus i}}(\bW_{t}^{(i)} ) \big) \big\|_2^2 \nonumber\\
    &\quad + \frac{ 2 \eta^2 (1+1/p)}{ n^2}\big(\big\| \nabla \ell(\bW_t;z_i)\big\|_2^2+ \big\| \nabla \ell(\bW_t^{(i)};z_i')   \big\|_2^2\big).
\end{align}
From Lemma~\ref{lem:3-unbounded-coercivity} we know
\begin{align*}
     &\big\| \bW_{t} -\bW_{t}^{(i)} -\eta\big( \nabla L_{S^{\setminus i}}(\bW_t) -\nabla  L_{S^{\setminus i}}(\bW_{t}^{(i)} ) \big) \big\|_2^2 \\
    & =  \big\| \bW_{t} - \bW_{t}^{(i)}   \big\|_2^2  + \eta^2 \big\| \nabla  L_{S^{\setminus i}}(\bW_t)  -  \nabla L_{S^{\setminus i}}(\bW_{t}^{(i)} )   \big\|_2^2    -  2\eta \Big\langle \bW_{t}  -
     \bW_{t}^{(i)},\nabla L_{S^{\setminus i}}(\bW_t)  -  \nabla L_{S^{\setminus i}}(\bW_{t}^{(i)} )   \Big\rangle\\
    &\le \big\| \bW_{t} -\bW_{t}^{(i)}   \big\|_2^2 +\eta^2 \big\|  \nabla L_{S^{\setminus i}}(\bW_t) - \nabla L_{S^{\setminus i}}(\bW_{t}^{(i)} )   \big\|_2^2  - 4\eta^2 \big(1- {4\eta \hat{\rho}} \big)\big\| \nabla L_{S^{\setminus i}}(\bW_t) - \nabla L_{S^{\setminus i}}(\bW^{(i)}_t)   \big\|_2^2  \\&\quad + 2\eta \tilde{\epsilon}_t  \big\| \bW_t -\bW_t^{(i)} -\eta \big( \nabla L_{S^{\setminus i}}(\bW_t) - \nabla L_{S^{\setminus i}}(\bW_t^{(i)}) \big) \big\|_2^2.
\end{align*} 
Note $4\eta \hat{\rho} \le 1/2$ implies $1-4(1- {4\eta  \hat{\rho}} ) < 0$ and condition~\eqref{eq:3-m-condition-1} ensures that $2\eta \tilde{\epsilon}_t   <  1$ for any $t\in[T]$. Then from the above inequality we can get
\begin{align}\label{eq:3-stability-5}
    &(1-2\eta \tilde{\epsilon}_t )\big\| \bW_t -\bW_t^{(i)} -\eta \big( \nabla L_{S^{\setminus i}}(\bW_t) - \nabla L_{S^{\setminus i}}(\bW_t^{(i)}) \big) \big\|_2^2 \le  \big\| \bW_{t} -\bW_{t}^{(i)}   \big\|_2^2 . 
\end{align}
 
Now, plugging \eqref{eq:3-stability-5} back into \eqref{eq:3-stab-1} we have
\begin{align*} 
     \big\| \bW_{t+1} \!-\!\bW_{t+1}^{(i)} \big\|_2^2\!\le \! \frac{1+p}{ 1\!-\!2\eta \tilde{\epsilon}_t   }\big\| \bW_{t} \!-\!\bW_{t}^{(i)}   \big\|_2^2 \! +\! \frac{ 2 \eta^2 (1\!+\!1/p)}{ n^2}\Big(\big\| \nabla \ell(\bW_t;z_i)\big\|_2^2\!+\! \big\| \nabla \ell(\bW_t^{(i)};z_i')   \big\|_2^2\Big).
\end{align*}
Applying the above inequality recursively and note that $\bW_0=\bW_0^{(i)}$ we get
\begin{align*}
     \big\| \bW_{t+1} -\bW_{t+1}^{(i)} \big\|_2^2&\le   \frac{ 2 \eta^2 (1+1/p)}{ n^2} \sum_{j=0}^t  \big(\big\| \nabla \ell(\bW_j;z_i)\big\|_2^2+ \big\| \nabla \ell(\bW_j^{(i)};z_i')   \big\|_2^2\big)\prod_{\tilde{j}=j+1}^{t} \frac{1+p}{ 1-2\eta \tilde{\epsilon}_{\tilde{j}} }\nonumber\\
     &\le  \frac{  2\eta^2 (1+1/p) (1+p)^t}{ n^2 (1-2\eta \tilde{\epsilon}_j )^t} \sum_{j=0}^t  \big(\big\| \nabla \ell(\bW_j;z_i)\big\|_2^2+ \big\| \nabla \ell(\bW_j^{(i)};z_i')   \big\|_2^2\big) . 
\end{align*}
Let $p=1/t$ and note that $(1+1/t)^t\le e$, we have 
\begin{align}\label{eq:3-stab-3}
     \big\| \bW_{t+1} -\bW_{t+1}^{(i)} \big\|_2^2 \le  \frac{2 e  \eta^2 (1+t)  }{ n^2  } \sum_{j=0}^t  \big(\big\| \nabla \ell(\bW_j;z_i)\big\|_2^2+ \big\| \nabla \ell(\bW_j^{(i)};z_i')   \big\|_2^2\big)\prod_{\tilde{j}=j+1}^{t} \frac{1 }{ 1-2\eta \tilde{\epsilon}_{\tilde{j}} } . 
\end{align}
According to   Lemma~\ref{lem:3-unbounded-smoothness} and Lemma~\ref{lem:self-bounding},  Assumption~\ref{ass:loss} and noting that $\|\bW_j-\bW_0\|_2\le \sqrt{2c_0\eta j}$ for any $j\le t$, we know
\begin{align*}
     \| \nabla \ell(\bW_j ;z) \|_2^2 &\le 2\| \nabla \ell(\bW_j;z) - \nabla \ell(\bW_0;z) \|_2^2 + 2 \|  \nabla \ell(\bW_0;z) \|_2^2\\
     & \le 2 \hat{\rho}^2 \|\bW_j -\bW_0\|_2^2 + 4\hat{\rho}  \ell(\bW_0;z) \le 4c_0\hat{\rho} ( \hat{\rho} \eta j  + 1) . 
\end{align*}
Similarly, we have
\begin{align*}
     \| \nabla \ell(\bW_j^{(i)} ;z) \|_2^2  \le 4c_0\hat{\rho} ( \hat{\rho} \eta j  + 1) .
\end{align*}
Combining the above three inequalities together, we get 
\begin{align*}
     &\big\| \bW_{t+1} -\bW_{t+1}^{(i)} \big\|_2^2\le  \frac{8c_0 e  \eta^2 (1+t)^2 \hat{\rho}\big(  \hat{\rho}   \eta t  +  2) \big) }{ n^2  } \prod_{\tilde{j}=j+1}^{t} \frac{1 }{ 1-2\eta \tilde{\epsilon}_{\tilde{j}} }   . 
\end{align*}
Similar to the proof of Theorem~\ref{thm:stability}, we can derive the following stability result by induction
 \begin{align*}
     &\big\| \bW_{t+1} -\bW_{t+1}^{(i)} \big\|_2\le  \frac{2 e   \eta T   \sqrt{ 2c_0 \hat{\rho}( \hat{\rho}\eta T+2 )}}{ n  }  . 
\end{align*}
Here, the condition $   \frac{1}{(1-2\eta \tilde{\epsilon}_j) ^t} \le \Big( \frac{1}{1-1/(t+1)} \Big)^t \le e$ is ensured by condition~\eqref{eq:3-m-condition-1}, i.e.,
\begin{align*}
     m\!\gtrsim & \big( (\eta T\mathcal{B}_{T})^2  + \frac{(\eta T)^{\frac{7}{2}}\mathcal{B}_{T}  }{n}\big)^{\frac{1}{5c-\frac{1}{2}}} + \big( (\eta T)^{\frac{3}{2}} \mathcal{B}_{T}^2 \!+\! \frac{(\eta T)^3\mathcal{B}_{T} }{n} \big)^{\frac{1}{4c-1}}\!\!+\! \big( (\eta T)^2 \mathcal{B}_{T}  \!+\!\frac{(\eta T)^{\frac{7}{2}}}{n}\big)^{\frac{1}{5c-1}}\\
    &\!+\!\big( (\eta T)^{\frac{3}{2}} \mathcal{B}_{T} \!+\! \frac{(\eta T)^3}{n}\big)^{\frac{1}{4c-\frac{3}{2}}}
\end{align*} 
with  $\mathcal{B}_{T}=\sqrt{\eta T}+\|\bW_0\|_2$. 
The proof is the same as Theorem~\ref{thm:stability}, we omit it for simplicity. 
\end{proof}

We can combine Theorem~\ref{thm:3-unbounded-stability} and Lemma~\ref{lem:connection} together to get the upper bound of the generalization error.

\begin{proof}[Proof of Theorem~\ref{thm:3-unbounded-generalization}]
The proof is similar to that of Theorem~\ref{thm:generalization}. 
From \eqref{eq:3-stab-3} we know that
\begin{align*}
     \big\| \bW_{t+1} -\bW_{t+1}^{(i)} \big\|_2^2 \le  \frac{4  e^2 \eta^2 \hat{\rho} (1+t) }{n^2} \sum_{j=0}^t  \Big(  \ell(\bW_j;z_i) +   \ell(\bW_j^{(i)};z_i')     \Big),
\end{align*}
where  we used  the self-bounding property of the smooth loss (Lemma~\ref{lem:self-bounding}).  

Then, taking an average over $i\in[n]$ and noting that $\E\big[\ell(\bW_j;z_i)\big]=\E\big[\ell(\bW_j^{(i)};z_i') \big]$, we have
\begin{align*}
    \frac{1}{n}\sum_{i=1}^n \E \big\| \bW_{t+1} -\bW_{t+1}^{(i)} \big\|_2^2 &\le   \frac{8  e^2 \eta^2 \hat{\rho} (1+t) }{n^2} \sum_{j=0}^t     \E\big[ L_S(\bW_j )\big],  
\end{align*}

Combining the above stability bounds with  Lemma~\ref{lem:connection} together and noting that $ L_S(\bW_t)  \le \frac{1}{t}\sum_{j=1}^{t-1} L_S( \bW_j ) $ \cite{richards2021stability}, the desired result is obtained.
\end{proof}

 \subsection{Proofs of Optimization Bounds}\label{appendix:three-optimization}
To show optimization error bounds, we first introduce the following lemma on the bound of GD iterates. 
\begin{lemma}\label{lem:3-unbounded-distance}
Suppose Assumptions~\ref{ass:activation} and \ref{ass:loss} hold, and $\eta \le 1/(8\hat{\rho})$. Assume \eqref{eq:3-m-condition-1} and \eqref{m-condition-unbounded2} hold.    
 Then for any $t\in[T]$, there holds
\begin{align*}
   &1\vee \E[ \|\bW_{\frac{1}{\eta T}}^{*} - \bW_t  \|_2^2 ] \le \Big( \frac{16  e^2 \eta^3   t^2 \hat{\rho}^2}{ n^2}\! +\! \frac{16 e  \eta^2 t  \hat{\rho}       }{n } \Big) \sum_{s=0}^{t-1} \E \big[  L_S(\bW_s) \big]       \! +\!  2 \|\bW_{\frac{1}{\eta T}}^{*} \!\!-\!\bW_0 \|_2^2   \!+\! 2\eta T \big[  L(\bW_{\frac{1}{\eta T}}^{*})-L(\bW^{*})\big] .
\end{align*}
\end{lemma}
\begin{proof}
  For any $\bW\in \R^{md}$ and $\alpha\in[0,1]$, define $\bW(\alpha):=\bW_t + \alpha(\bW- \bW_t)$. Similar to \eqref{eq:3-coercivity-6}, according to Lemma~\ref{lem:3-unbounded-smoothness} we can show that
\begin{align*}
    \lambda_{\min}\big( \nabla^2 L_S(\bW(\alpha))\big) 
    \!\ge\!  -C_{\bW(\alpha)} \Big( \big(\frac{B_{\sigma'}^2 c_\bx}{m^{2c-1/2}} \|\bW(\alpha)\|_2\!+\!\frac{B_{\sigma'}B_{\sigma}}{m^{2c-1}}\big) \big(\|\bW \!-\! {\bW }_t  \|_2\!+\!  \| {\bW }_t     \|_2\big) \!+\! \sqrt{2c_0} \Big), 
\end{align*}
where $C_{\bW(\alpha)}=\frac{B^2_{\sigma'}B_{\sigma''}c_{\bx}^2}{m^{3c}} \big\|   \bW(\alpha)^{(2)} \big\|_2^2 + \Big(\frac{B_{
    \sigma'} B_{\sigma''} c_{\bx}^2}{m^{2c-\frac{1}{2}}}+ \frac{2B_{\sigma''}B_{\sigma'}B_\sigma c_\bx }{m^{3c-\frac{1}{2}}} \Big)\| \bW(\alpha)^{(2)} \|_2  + \frac{ B_{\sigma''}B_{
    \sigma}^2}{m^{3c-1}} 
    + \frac{ 2 B_{\sigma'}^2 c_\bx}{m^{2c-\frac{1}{2}}}$. Let $\hat{C}_\bW=4B_1\big(  m^{-3c}(\|\bW\|_2^2 + 4c_0\eta T+2\|\bW_0\|_2^2) +  m^{\frac{1}{2}-2c}(\|\bW\|_2 +  \sqrt{2c_0\eta T} + \|\bW_0\|_2) \big)$. According to Lemma~\ref{lem:3-unbounded-wbound}, we can verify that $C_{\bW(\alpha)}\le \hat{C}_\bW$ for any $\alpha\in[0,1]$.  

Now, let
\begin{align*}
    g(\alpha)\!:=& L_S( \bW(\alpha) ) \!+\! \frac{\alpha^2 \hat{C}_\bW}{2}  \Big( \!\big(\frac{B_{\sigma'}^2 c_\bx}{m^{2c-1/2}} \|\bW(\alpha)\|_2+\frac{B_{\sigma'}B_{\sigma}}{m^{2c-1}}\big)   \big( \|\bW \!-\! \bW_t  \|_2 \!+\!  \|  \bW_t    \|_2\big) \!+\! \sqrt{2c_0} \Big)\\
    &\quad \times (1 \vee \E[\|\bW  -\bW_t  \|_2^2]). 
\end{align*} 
It is obvious that $g(\alpha)$ is convex in $\alpha\in[0,1]$. 
Similar to the proof of Lemma~\ref{lem:distance}, 
by convexity of $g$ and smoothness of the loss we can show that
\begin{align}\label{eq:3-unbounded-opt-2}
   & \frac{1}{t}\sum_{s=0}^{t-1} \E[  L_S(\bW_{s})] +   \frac{\E\big[   \|\bW_{t} - \bW   \|_2^2\big] }{2\eta t}\nonumber\\ 
  &\le    \E[L_S(\bW)] + \frac{\E\big[  \|\bW \!-\! \bW_0   \|_2^2  \big]}{ 2\eta t}  + \frac{\hat{C}_\bW}{2t} \sum_{s=0}^{t-1} \Big( \big(\frac{B_{\sigma'}^2 c_\bx}{m^{2c-1/2}} \|\bW(\alpha)\|_2 + \frac{B_{\sigma'}B_{\sigma}}{m^{2c-1}}\big) \big(\E[\|\bW \!-\!
   {\bW }_s  \|_2]\nonumber\\
  &\quad +   \sqrt{2c_0\eta T} + \|\bW_0\|_2\big) + \sqrt{2c_0} \Big)(1 \vee \E[\|\bW  -\bW_s  \|_2^2]).
\end{align}

Combining the above inequality with Theorem~\ref{thm:3-unbounded-generalization} and let  $\bW=\bW^{*}_{\frac{1}{\eta T}}$, we have
\begin{align}\label{eq:3-unbounded-opt-4}
      &\frac{\E\big[   \|\bW_{t} -\bW_{\frac{1}{\eta T}}^{*}   \|_2^2 \big]}{2\eta t} \nonumber\\
      &\!\!\le   \frac{1}{t}\sum_{s=1}^{t-1} \big[  L(\bW_{\frac{1}{\eta T}}^{*})-\E[  L(\bW_s)] \big]  + \frac{   \|\bW_{\frac{1}{\eta T}}^{*} \!-\!\bW_0  \|_2^2  }{2\eta t}  \!+\!  \Big( \frac{4 e^2 \eta^2  t \hat{\rho}^2}{ n^2} \!+\!    \frac{4 e  \eta   \hat{\rho}      }{n } \Big) \sum_{s=0}^{t-1}     \E\big[ L_S(\bW_s)\big] \nonumber\\
      &  \!\! + \! \frac{1}{2t} \!\sum_{s=0}^{t-1}  \!\hat{C}_{\bW_{\frac{1}{\eta T}}^{*}} \!\!\Big( \big(\frac{B_{\sigma'}^2 c_\bx}{m^{2c - 1/2}} \|\bW( \alpha )\|_2\!+\!\frac{B_{\sigma'}B_{\sigma}}{m^{2c\!-\!1}}\big) \big(\E[\|\bW_{\frac{1}{\eta T}}^{*} \!\!- 
  \!{\bW }_{ s}  \|_2]  \!+\!   \sqrt{2c_0\eta T}\!+\!\|\bW_0\|_2 \big)  \!+\! \sqrt{2c_0} \Big)   (1 \vee \E[\|\bW_{\frac{1}{\eta T}}^{*}  - \bW_{ s}  \|_2^2])\nonumber\\
      &\!\le  L(\bW_{\frac{1}{\eta T}}^{*})-L(\bW^{*}) +  \frac{   \|\bW_{\frac{1}{\eta T}}^{*} - \bW_0  \|_2^2   }{2\eta t}  +  \Big( \frac{4 e^2 \eta^2 t \hat{\rho}^2}{ n^2} +    \frac{4 e  \eta \hat{\rho}     }{n } \Big) \sum_{s=0}^{t-1}     \E\big[ L_S(\bW_s)\big]\nonumber\\
      &   +\!  \frac{1}{2t}  \!\sum_{s=0}^{t-1} \!  \hat{C}_{\bW_{\frac{1}{\eta T}}^{*}}  \!\Big( \!\big(\frac{B_{\sigma'}^2 c_\bx}{m^{2c -1/2}} \|\bW( \alpha )\|_2 \!+\! \frac{B_{\sigma'}B_{\sigma}}{m^{2c - 1}}\big) \big(\|\bW_{\frac{1}{\eta T}}^{*}\!\!-\! 
  {\bW }_{ s}  \|_2   \!+\!   \sqrt{2c_0\eta T} \!+\! \|\bW_0\|_2\big) \!+\!  \sqrt{2c_0} \Big) (1  \vee  \E[\|\bW_{\frac{1}{\eta T}}^{*}  - \bW_{ s}    \|_2^2]), 
\end{align}
where in the second inequality we used $  L(\bW_{\frac{1}{\eta T}}^{*})- L(\bW_s)\le L(\bW_{\frac{1}{\eta T}}^{*}) - L(\bW^{*}) $ since $L(\bW_s)\ge  L(\bW^{*})$ for any $s\in[t-1]$. 
 

According to Lemma~\ref{lem:3-unbounded-wbound}  we can get 
\begin{align}\label{eq:opt-unbounded-5}
     \|  \bW_{\frac{1}{\eta T}}^{*} -\bW_s  \|_2 \le  \|\bW_{\frac{1}{\eta T}}^{*} -\bW_0 \|_2 + \|\bW_s -\bW_0\|_2 \le  \|\bW_{\frac{1}{\eta T}}^{*} -\bW_0\|_2 + \sqrt{2 c_0\eta T  }.
\end{align}
Then there holds
\begin{align}\label{eq:3-unbounded-opt-6}
 &\big(\frac{B_{\sigma'}^2 c_\bx}{m^{2c-1/2}} \|\bW(\alpha)\|_2\!+\!\frac{B_{\sigma'}B_{\sigma}}{m^{2c-1}}\big) \big( \|\bW_{\frac{1}{\eta T}}^{*}\!-
  \!{\bW }_s  \|_2   +   \sqrt{2c_0\eta T} + \|\bW_0\|_2 \big)  + \sqrt{2c_0}\nonumber\\
  & \le  \Big(\frac{B_{\sigma'}^2 c_\bx}{m^{2c -1/2}} (2\sqrt{2 c_0\eta T }+2\|\bW_0\|_2  +  \|\bW_{\frac{1}{\eta T}}^{*} -\bW_0\|_2 )+\frac{B_{\sigma'}B_{\sigma}}{m^{2c-1}}\Big)  ( 2\sqrt{2 c_0 \eta T} +\|\bW_0\|_2  +  \|\bW_{\frac{1}{\eta T}}^{*} -\bW_0\|_2  )  + \sqrt{2c_0}. 
\end{align}
Plugging the above inequality    back into \eqref{eq:3-unbounded-opt-4} and multiplying both sides by $2\eta t$ yields
\begin{align*}
       \E\big[   \|\bW_{t} -\bW_{\frac{1}{\eta T}}^{*}\|_2^2\big]  
      & \le     \|\bW_{\frac{1}{\eta T}}^{*} \!\!   -\! \bW_0 \|_2^2    \!+\!  \eta \hat{C}_{\bW_{\frac{1}{\eta T}}^{*}} \! \Big( \Big(\frac{B_{\sigma'}^2 c_\bx}{m^{2c - 1/2}} (2\sqrt{2 c_0\eta T } + \|\bW_0\|_2 \!+\!   \|\bW_{\frac{1}{\eta T}}^{*} \!\!-\!\bW_0\|_2) \!+\!\frac{B_{\sigma'}B_{\sigma}}{m^{2c - 1}}\Big) \\
      &\quad \times ( 2\sqrt{2\eta T c_0} +\|\bW_0\|_2+    \|\bW_{\frac{1}{\eta T}}^{*} -\bW_0 \|_2  )   +  \sqrt{2c_0}\Big) \sum_{s=0}^{t-1}     (1 \vee \E[\|\bW_{\frac{1}{\eta T}}^{*}  - \bW_s  \|_2^2])  \\& \quad  +  \Big( \frac{8 e^2 \eta^3 t^2 \hat{\rho}^2}{ n^2}  +     \frac{8e  \eta^2   t   \hat{\rho}  }{n } \Big) \sum_{s=0}^{t-1}     \E\big[ L_S(\bW_s)\big]   + 2\eta T \big[  L(\bW_{\frac{1}{\eta T}}^{*})\!- L(\bW^{*})\big]   .
\end{align*}
Let $x=\max_{s\in[T]}\E[ \|\bW_{\frac{1}{\eta T}}^{*} -\bW_s  \|_2^2 ]\vee 1 $. Then the above inequality implies
\begin{align*}
      x  
      & \le     \|\bW_{\frac{1}{\eta T}}^{*}\! \!-\!\bW_0  \|_2^2   \! +\!  \eta T\hat{C}_{\bW_{\frac{1}{\eta T}}^{*}}\!\! \Big( \Big(\frac{B_{\sigma'}^2 c_\bx}{m^{2c\!-\!1/2}} (2\sqrt{2 c_0\eta T } + \|\bW_0\|_2\!+\! \|\bW_{\frac{1}{\eta T}}^{*}\!\!-\bW_0 \|_2)\!+\!\frac{B_{\sigma'}B_{\sigma}}{m^{2c-1}}\Big)\\
      &\times ( 2\sqrt{2\eta T c_0}\!+\!\|\bW_0\|_2\!+ \! \|\bW_{\frac{1}{\eta T}}^{*} \!\!-\bW_0\|_2 )  \!+
      \!\sqrt{2c_0}\Big)x     \!+\!  \Big( \frac{8 e^2 \eta^3   t^2 \hat{\rho}^2}{ n^2} \!+\!    \frac{8 e  \eta^2   t   \hat{\rho}  }{n } \Big) \sum_{s=0}^{t-1}     \E\big[ L_S(\bW_s)\big]\\
      &+2\eta T \big[  L(\bW_{\frac{1}{\eta T}}^{*})-L(\bW^{*})\big]  .
\end{align*}
Note that condition \eqref{m-condition-unbounded2}   implies that  $ \eta T\hat{C}_{\bW_{\frac{1}{\eta T}}^{*}} \big(  (\frac{B_{\sigma'}^2 c_\bx}{m^{2c - 1/2}} (2\sqrt{2 c_0\eta T } + \|\bW_0\|_2 +  \|\bW_{\frac{1}{\eta T}}^{*} -\bW_0 \|_2) + \frac{B_{\sigma'}B_{\sigma}}{m^{2c-1}} ) ( 2\sqrt{2\eta T c_0} + \|\bW_0\|_2 +  \|\bW_{\frac{1}{\eta T}}^{*}  -\bW_0\|_2 )  +
       \sqrt{2c_0}\big)\le \frac{1}{2}$,  then there holds
\[ x \le  \Big( \frac{16  e^2 \eta^3  t^2 \hat{\rho}^2}{ n^2} +    \frac{16 e  \eta^2 t \hat{\rho}       }{n } \Big) \sum_{s=0}^{t-1}     \E\big[ L_S(\bW_s)\big] +  2   \|\bW_{\frac{1}{\eta T}}^{*}\!\!-\!\bW_0    \|_2^2     +2\eta T \big[  L(\bW_{\frac{1}{\eta T}}^{*})-L(\bW^{*})\big] .  \]
It then follows that
\begin{align*}
    &1\vee \E[ \|\bW_{\frac{1}{\eta T}}^{*} \!-\!\bW_t  \|_2^2 ]   \le \Big( \frac{16  e^2 \eta^3  t^2 \hat{\rho}^2}{ n^2}\! +\!    \frac{16 e  \eta^2 t \hat{\rho}       }{n } \Big) \sum_{s=0}^{t-1}     \E\big[ L_S(\bW_s)\big] \!+\!  2 \|\bW_{\frac{1}{\eta T}}^{*} \!\! -\!\bW_0 \|_2^2   \! +\!2\eta T \big[  L(\bW_{\frac{1}{\eta T}}^{*})\!-\! L(\bW^{*})\big] .
\end{align*}
This completes the proof. 
\end{proof}

\begin{proof}[Proof of Theorem~\ref{thm:opt-unbounded}]
Combining  \eqref{eq:3-unbounded-opt-6} and  \eqref{eq:3-unbounded-opt-2} with $\bW=\bW^{*}_{\frac{1}{\eta T}}$ together yields 
\begin{align}\label{eq:3-pot-unbounded-5}
    \frac{1}{T}\sum_{s=0}^{T-1} \E[  L_S(\bW_{s})]  
  &\!\le\!   \E[L_S(\bW^{*}_{\frac{1}{\eta T}})]\! +\! \frac{   \|\bW^{*}_{\frac{1}{\eta T}} \!\!-\!\bW_0\|_2^2 }{ 2\eta T}\! +\! \frac{\hat{C}_{\bW^{*}_{\frac{1}{\eta T}}}}{2T} \!\!\sum_{s=0}^{T-1} \! \Big( \!\big(\frac{B_{\sigma'}^2 c_\bx}{m^{2c-1/2}} (2\sqrt{2 c_0\eta T }\!+\!\|\bW_0\|_2  \!+\!  \|\bW_{\frac{1}{\eta T}}^{*} \!\!-\! \bW_0\|_2 ) \nonumber\\
  & \quad + \frac{B_{\sigma'}B_{\sigma}}{m^{2c-1}}\big)  \big(\E[\|\bW^{*}_{\frac{1}{\eta T}} -
   {\bW }_s  \|_2] +    2\sqrt{2c_0\eta T}  + \|\bW_0\|_2\big) + \sqrt{2c_0} \Big) (1 \vee \E[\|\bW^{*}_{\frac{1}{\eta T}}  - \bW_s  \|_2^2]) \nonumber\\
  & \le     \E[L_S(\bW^{*}_{\frac{1}{\eta T}} )] +  \frac{ 1}{2T}\hat{C}_{\bW^{*}_{\frac{1}{\eta T}}} \hat{B}_{\bW^{*}_{\frac{1}{\eta T}}}   \sum_{s=0}^{T-1 } 1  \vee  \E[\|\bW^{*}_{\frac{1}{\eta T}} - \bW_s  \|_2^2] + \frac{ \|\bW^{*}_{\frac{1}{\eta T}}-\bW_0 \|_2^2}{2\eta T}, 
\end{align}
where $\hat{B}_{\bW^{*}_{\frac{1}{\eta T}}}= \big(\frac{B_{\sigma'}^2 c_\bx}{m^{2c-1/2}} (2 \sqrt{2 c_0\eta T } +  \| \bW_{\frac{1}{\eta T}}^{*}-\bW_0  \|_2) + \frac{B_{\sigma'}B_{\sigma}}{m^{2c-1}}\big)( 2\sqrt{2\eta T c_0}+\|\bW_0\|_2 + \|\bW_{\frac{1}{\eta T}}^{*} -\bW_0\|_2 )  + \sqrt{2c_0}$ and in the last inequality we used \eqref{eq:opt-unbounded-5}.  
 
By monotonically decreasing of $\{ L_S(\bW_t)\}$ and Lemma~\ref{lem:3-unbounded-distance}, we further know
\begin{align*} 
     \E[ L_S(\bW_{T})] &\le     \E[L_S(\bW^{*}_{\frac{1}{\eta T}} )] +   \frac{\hat{C}_{\bW^{*}_{\frac{1}{\eta T}}}\hat{B}_{\bW^{*}_{\frac{1}{\eta T}}}}{2T}   \sum_{s=0}^{T-1 } 1 \vee  \E[\|\bW^{*}_{\frac{1}{\eta T}} \!\!-\! \bW_s  \|_2^2]  + \frac{ \|\bW^{*}_{\frac{1}{\eta T}} \!\!-\!\bW_0\|_2^2 }{2\eta T}\\
    & \le  \E[L_S(\bW^{*}_{\frac{1}{\eta T}})] +      \hat{C}_{\bW^{*}_{\frac{1}{\eta T}}}  \hat{B}_{\bW^{*}_{\frac{1}{\eta T}}} \Big( \Big( \frac{8  e^2 \eta^3  T^2 \hat{\rho}^2}{ n^2}\!+\!\frac{8 e  \eta^2 T  \hat{\rho}     }{n } \!\Big)\! \sum_{s=0}^{T-1 }\!\E \big[  L_S(\bW_s) \big] \\
    &\quad +   \|\bW_{\frac{1}{\eta T}}^{*} \!-\!\bW_0 \|_2^2   +  \eta T \big[  L(\bW_{\frac{1}{\eta T}}^{*})-L(\bW^{*})\big]\Big)  + \frac{ \|\bW^{*}_{\frac{1}{\eta T}} -\bW_0 \|_2^2 }{2\eta T},  
\end{align*}
which completes the proof.
 \end{proof}
 
 \begin{lemma}\label{lem:sum-risk-unbounded}
 Suppose Assumptions~\ref{ass:activation} and  \ref{ass:loss}   hold. 
Let $\{\bW_t\}$ be produced by \eqref{eq:GD} with $\eta\le 1/(8\hat{\rho})$. Assume \eqref{eq:3-m-condition-1} and \eqref{m-condition-unbounded2}  hold. Then
\begin{align*}
 \sum_{s=0}^{T-1 }\E[L_S(\bW_{s})]  \le 2 T  L(\bW^{*}_{\frac{1}{\eta T}}) -\eta T L(\bW^{*})  + 
    \big(1+\frac{1}{2\eta}\big) \|\bW^{*}_{\frac{1}{\eta T}} -\bW_0 \|_2^2 . 
\end{align*}
 \end{lemma}
 \begin{proof}
 Multiplying $T$ over both sides of \eqref{eq:3-pot-unbounded-5} and using Lemma~\ref{lem:3-unbounded-distance} we get
\begin{align} 
    \sum_{s=0}^{T-1} \E[   L_S(\bW_{s})] 
    &\le    T L(\bW^{*}_{\frac{1}{\eta T}}) \! +\!  \hat{C}_{\bW^{*}_{\frac{1}{\eta T}}} \!\!\hat{B}_{\bW^{*}_{\frac{1}{\eta T}}} \! \!\Big(  \Big( \frac{8  e^2 \eta^3  T^2  \hat{\rho}^2}{ n^2}\!+\!\frac{8 e  \eta^2 T  \hat{\rho}       }{n } \!\Big)\! \sum_{s=0}^{T-1 }\!\E \big[  L_S(\bW_s) \big]\nonumber\\
   &  +  \|\bW_{\frac{1}{\eta T}}^{*} -\bW_0 \|_2^2    +  \eta T \big[  L(\bW_{\frac{1}{\eta T}}^{*})-L(\bW^{*})\big]\Big)  + \frac{ \|\bW^{*}_{\frac{1}{\eta T}} -\bW_0 \|_2^2 }{2\eta  }. 
\end{align}
Condition~\eqref{m-condition-unbounded2}  with $\|\bW_{\frac{1}{\eta T}}^{*}-\bW_0\|_2\le \|\bW^{*}-\bW_0\|_2$ implies $   \hat{C}_{\bW^{*}_{\frac{1}{\eta T}}}\hat{B}_{\bW^{*}_{\frac{1}{\eta T}}} \big( \frac{4  e^2 \eta^3  T^2 \hat{\rho}^2}{ n^2}\!+\!\frac{4 e  \eta^2 T  \hat{\rho}       }{n } \!\big)\! \le 1/2$ and $ \hat{C}_{\bW^{*}_{\frac{1}{\eta T}}}\hat{B}_{\bW^{*}_{\frac{1}{\eta T}}} \le 1 $, then there holds
\begin{align*}
 \sum_{s=0}^{T-1 }\E[L_S(\bW_{s})]  \le 2 T  L(\bW^{*}_{\frac{1}{\eta T}}) -\eta T L(\bW^{*})  + 
    \big(1+\frac{1}{2\eta}\big) \|\bW^{*}_{\frac{1}{\eta T}} -\bW_0 \|_2^2 ,
\end{align*}
which completes the proof. 
 \end{proof}

\subsection{Proofs of Excess Population  Bounds}\label{appendix:three-excess}

 \begin{proof}[Proof of Theorem~\ref{thm:excess-unbounded}]  
 Note $\hat{\rho}=\O(1)$. 
According to Theorem~\ref{thm:3-unbounded-generalization} and Lemma~\ref{lem:sum-risk-unbounded} we know
\begin{align}\label{eq:excess-unbounded-1}
     \E[ L(\bW_T) - L_S( \bW_T )] 
 =\O\Big( \big( \frac{\eta^2  T^2  }{n^2} +\frac{\eta T  }{n} \big) \big(   L(\bW^{*}_{\frac{1}{\eta T}}) +   \frac{1 }{\eta T} \|\bW^{*}_{\frac{1}{\eta T}}-\bW_0\|_2^2  \Big). 
\end{align}
The estimation of the optimization error is given by combining Lemma~\ref{lem:sum-risk-unbounded} and  Theorem~\ref{thm:opt-unbounded} together
\begin{align*}
    &\E[ L_S(\bW_T) -L_S(\bW^{*}_{\frac{1}{\eta T}}) -\frac{1}{2\eta T}\|\bW^{*}_{\frac{1}{\eta T}}-\bW_0 \|_2^2 ]\nonumber\\
     &=\O\Big(  \big( \frac{\eta^2T^2 }{n^2}\!+\!\frac{\eta T   }{n} \big)\big(L(\bW^{*}_{\frac{1}{\eta T}})\!+\! \frac{1}{2\eta T}\|\bW^{*}_{\frac{1}{\eta T}}\!\!-\!\bW_0\|_2^2\big) \!+\! \frac{1}{ \eta T} \|\bW^{*}_{\frac{1}{\eta T}}\!\!-\!\bW_0\|_2^2 + \Lambda_{\frac{1}{\eta T}}\big)  \Big),
    \end{align*}
    where we used the fact that $ \hat{C}_{\bW^{*}_{\frac{1}{\eta T}}}\hat{B}_{\bW^{*}_{\frac{1}{\eta T}}} \le 1/\eta T   $ implied by condition \eqref{m-condition-unbounded2}  and $L(\bW^{*}_{\frac{1}{\eta T}})-L(\bW^{*})\le \Lambda_{\frac{1}{\eta T}}.$

Combining the above two inequalities together we get
\begin{align*} 
      &\E[L(\bW_T) - L(\bW^{*})]\\ &= \Big[   \E[ L(\bW_T)  - L_S(\bW_T) \Big] + \E\Big[ L_S(\bW_T) - \big( L_S(\bW^{*}_{\frac{1}{\eta T}} ) + \frac{1}{2\eta T} \| \bW^{*}_{\frac{1}{\eta T}} -\bW_0 \|_2^2 \big) \Big]\nonumber\\
     &\quad + \Big[ L(\bW^{*}_{\frac{1}{\eta T}} ) + \frac{1}{2\eta T} \| \bW^{*}_{\frac{1}{\eta T}} -\bW_0 \|_2^2  -  L(\bW^{*}) \Big]\nonumber\\
    &= \O\Big(    \frac{  \eta  T      }{n }  \Big(  \frac{  \eta   T   }{ n }   \!+\!  1 \Big) \Big[ L(\bW^{*}_{\frac{1}{\eta T}}) \!+\!  \frac{1}{2\eta T}  \|\bW^{*}_{\frac{1}{\eta T}}\!\!-\!\bW_0\|_2^2 \Big] \!  + \!\frac{  1}{\eta T} \|\bW^{*}_{\frac{1}{\eta T}}\!\!-\!\bW_0\|_2^2 + \Lambda_{\frac{1}{\eta T}} \Big) . 
\end{align*}
Since $n\gtrsim \eta T  $, we further have
\begin{align*} 
     &\E[L(\bW_T) \!-\! L(\bW^{*})]  
   = \O\Big(    \frac{  \eta  T       }{n }  \Big[ L(\bW^{*}_{\frac{1}{\eta T}}) \!+\!\frac{1}{2\eta T} \|\bW^{*}_{\frac{1}{\eta T}}\!\!-\!\bW_0\|_2^2   \Big] \!  +\! \frac{  1}{\eta T}  \|\bW^{*}_{\frac{1}{\eta T}}\!\!-\!\bW_0\|_2^2 +\Lambda_{\frac{1}{\eta T}} \Big) . 
\end{align*}
 Finally, note that  $L(\bW^{*}_{\frac{1}{\eta T}}) +\frac{1}{2\eta T} \|\bW^{*}_{\frac{1}{\eta T}}-\bW_0\|_2^2 =L(\bW^{*})+ \Lambda_{\frac{1}{\eta T}} $ and $ \|\bW^{*}_{\frac{1}{\eta T}}-\bW_0\|_2^2 \le \eta T \Lambda_{\frac{1}{\eta T}}$, we have
 \begin{align*} 
    \E[L(\bW_T) - L(\bW^{*})] =
    \O\Big(    \frac{  \eta  T      }{n }    L(\bW^{*} )    +\Lambda_{\frac{1}{\eta T}} \Big) . 
\end{align*}
The proof is completed. 
\end{proof}
\begin{proof}[Proof of Corollary~\ref{cor:3-excess-rate-2}]
 \textbf{Part (a)}. 
 From the definition of the approximation error $\Lambda_{\frac{1}{\eta T}}$ and Theorem~\ref{thm:excess-unbounded} we can get 
  \begin{align*} 
    \E[L(\bW_T) - L(\bW^{*})] =
    \O\Big(    \frac{  \eta  T      }{n }    L(\bW^{*} )    +\frac{1}{\eta T}\| \bW^{*}-\bW_0\|_2^2 \Big) . 
\end{align*}
For the case $c\in [9/16,1]$,  
 to ensure conditions \eqref{eq:3-m-condition-1}  and \eqref{m-condition-unbounded2} hold, 
 we choose $m\asymp (\eta T)^{4}$ for this case. Then according to Theorem~\ref{thm:excess-unbounded} and Assumption~\ref{ass:minimizer}, there holds
  \begin{align*} 
    \E[L(\bW_T) - L(\bW^{*})] =
    \O\Big(    \frac{  \eta  T         }{n }     + (\eta T)^{3-8\mu}  \Big). 
\end{align*}
Here, the condition $\mu\ge 1/2$ ensures that $3-8\mu<0$. Hence, the bound will vanish as $\eta T$ tends to $0$.  
Further, if $n^{\frac{ 1}{2(8\mu-3)}}\lesssim \eta T\lesssim \sqrt{n}$ (the existence of $\eta T$ is ensured by $\mu\ge 1/2$), then there holds $\eta T/n=\O(n^{-1/2})$ and $(\eta T)^{3-8\mu} =\O(n^{-1/2})$. That is
  \begin{align*} 
    \E[L(\bW_T) - L(\bW^{*})] =
    \O\big( \frac{1}{\sqrt{n}}  \big). 
\end{align*}

For the case $c\in(1/2,9/16)$, we choose $m\asymp (\eta T)^{\frac{1}{4c-2}}$, and $n^{\frac{2c-1}{2\mu+4c-3}}\lesssim \eta T\lesssim \sqrt{n}$. From Theorem~\ref{thm:excess-unbounded} and Assumption~\ref{ass:minimizer} we have
 \begin{align*} 
    \E[L(\bW_T) - L(\bW^{*})] =
    \O\big( \frac{1}{\sqrt{n}}  \big). 
\end{align*}
The first part of the theorem is proved.

\textbf{Part (b).} For the case $c\in[9/16,1]$, by choosing $m\asymp (\eta T)^{4}$ and $\eta T\gtrsim
n^{\frac{ 1}{ 8\mu-3}}$, 
from Theorem~\ref{thm:excess-unbounded} and Assumption~\ref{ass:minimizer} we have
 \begin{align*} 
    \E[L(\bW_T) - L(\bW^{*})] =
    \O\big( \frac{1}{ {n}}  \big). 
\end{align*}

For the case $c\in(1/2,9/16)$, 
by choosing $m\asymp (\eta T)^{\frac{1}{4c-2}}$ and $\eta T\gtrsim n^{\frac{4c-2}{4c+2\mu-3}}$, 
from Theorem~\ref{thm:excess-unbounded} and Assumption~\ref{ass:minimizer} we have
 \begin{align*} 
    \E[L(\bW_T) - L(\bW^{*})] =
    \O\big( \frac{1}{ {n}}  \big). 
\end{align*}
The proof is completed. 
\end{proof}

\subsection{More Discussion on Related Works}\label{appendix:disc}

\cite{richards2021learning} derived a lower bound on the minimum eigenvalue of the Hessian for the  three-layer NN, where the first layer activation is linear, and the second activation is smooth. They proved the weak convexity of the empirical risk scales with $m^{\frac{1}{2}-2c}$ when optimizing the first and third layers of weights with  Lipschitz  and convex losses. We train the first and the second layers of the network with general smooth activation functions for both layers. 
Our result shows that the weak convexity of the least square loss scales with $m^{\frac{1}{2}-2c}$.

\bibliography{learning.bib}
\bibliographystyle{abbrv}

\end{document}